\documentclass[journal, twoside]{IEEEtran}
\usepackage{times}
\usepackage[pdftex]{graphicx}
\usepackage{subfigure}
\usepackage{amsmath,amssymb,amsopn,amstext,amsfonts,amsthm}
\let\emptyset\varnothing
\usepackage{cancel}
\usepackage[space]{cite}
\usepackage{pdfsync}
\usepackage{balance}
\usepackage{color}
\usepackage{mathtools}
\usepackage[noend]{algpseudocode}
\usepackage{algorithm}
\usepackage{bm}
\usepackage{diagbox}
\usepackage{float}
\usepackage[outdir=./]{epstopdf}
\usepackage{url}
\usepackage{multirow}
\usepackage{makecell}
\usepackage{adjustbox}
\usepackage{tabularx, booktabs}
\usepackage{lipsum}
\usepackage{enumitem}
\usepackage{tikz}
\usepackage{textcomp}

\usepackage[font=small,skip=2pt]{caption}
\usepackage[linkcolor=black,citecolor=black,urlcolor=black,colorlinks=true]{hyperref}

\newtheoremstyle{remarkIndent}
  {0.5\topsep}   
  {0.5\topsep}   
  {\normalfont}  
  {1.0em}        
  {\itshape}     
  {.}            
  {5pt plus 1pt minus 1pt} 
  {}              

\theoremstyle{remarkIndent}
\newtheorem{theorem}{Theorem}
\newtheorem{prop}{Proposition}
\newtheorem{problem}{Problem}

\newtheorem{lemma}{Lemma}
\newlength{\textfloatsepsave}
\newcommand{\norm}[1]{\lVert#1\rVert_2}

\newcommand{\point}{\mathbf}
\newcommand{\set}{\mathcal}

\DeclareGraphicsExtensions{.pdf}

\newcommand\copyrighttext{%
  \footnotesize \textcopyright 2019 IEEE. IEEE Transactions on Robotics. Personal use of this material is permitted. Permission from IEEE must be obtained for all other uses. }
\newcommand\copyrightnotice{%
\begin{tikzpicture}[remember picture,overlay]
\node[anchor=south,yshift=10pt] at (current page.south) {\fbox{\parbox{\dimexpr\textwidth-\fboxsep-\fboxrule\relax}{\copyrighttext}}};
\end{tikzpicture}%
}

\begin{document}

\title{An Efficient B-spline-Based Kinodynamic \\Replanning Framework for Quadrotors}
\author{Wenchao Ding, Wenliang Gao, Kaixuan Wang, and Shaojie Shen%
\thanks{Accepted final version. To Appear in IEEE Transactions on Robotics. \textcopyright 2019 IEEE. Personal use of this material is permitted. Permission from IEEE must be obtained for all other uses.
This work was supported by Hong Kong PhD Fellowship Scheme, HKUST-DJI Joint Innovation Laboratory. (\textit{Corresponding author: Wenchao Ding.}) }
\thanks{The authors are with the Department of Electronic and Computer Engineering, Hong Kong University of Science and Technology, Hong Kong, China (email: wdingae@ust.hk; wenliang.gao@ust.hk; kwangap@ust.hk; \mbox{eeshaojie@ust.hk}).}
\thanks{This paper has supplementary downloadable multimedia material available at \url{http://ieeexplore.ieee.org}.}
\thanks{Color versions of one or more of the figures in this paper are available online at \url{http://ieeexplore.ieee.org}.}
}

\markboth{Author's version}{DING \MakeLowercase{\textit{et al.}}: An Efficient B-spline-Based Kinodynamic Replanning Framework for Quadrotors}

\maketitle

\begin{abstract}
Trajectory replanning for quadrotors is essential to enable fully autonomous flight in unknown environments. Hierarchical motion planning frameworks, which combine path planning with path parameterization, are popular due to their time efficiency. However, the path planning cannot properly deal with non-static initial states of the quadrotor, which may result in non-smooth or even dynamically infeasible trajectories. In this paper, we present an efficient kinodynamic replanning framework by exploiting the advantageous properties of the B-spline, which facilitates dealing with the non-static state and guarantees safety and dynamical feasibility. Our framework starts with an efficient B-spline-based kinodynamic (EBK) search algorithm which finds a feasible trajectory with minimum control effort and time. To compensate for the discretization induced by the EBK search, an elastic optimization (EO) approach is proposed to refine the control point placement to the optimal location. Systematic comparisons against the state-of-the-art are conducted to validate the performance. Comprehensive onboard experiments using two different vision-based quadrotors are carried out showing the general applicability of the framework.
\end{abstract}


\copyrightnotice

\section{Introduction}
\IEEEPARstart{A}{utonomous} navigation for quadrotors in unknown environments has gained significant interest for its practical usage in various inspection and exploration tasks. To fulfill the need of fully autonomous exploration in unknown environments, trajectory replanning is of great significance. Replanning requires a real-time response to unexpected obstacles to guarantee safety while satisfying the low-level feasibility constraints induced by the non-trivial dynamics.

Many existing methods~\cite{mellinger2011minsnap, gao2016online, richter2016polyunqp, liu2017sfc, chen2016online} tackle this challenging problem using a hierarchical framework, which first finds a geometric path and then locally optimizes the path to a dynamically feasible trajectory with respect to a given time allocation. Although this framework is efficient, inadequacy exists between the path finding and the local path parameterization. Specifically, the parameterization may be restricted by the geometric path to a homotopy class that does not contain a globally optimal (or even feasible) solution, especially when faced with non-static initial states (non-zero velocities or any other higher-order derivatives), as shown in Fig.~\ref{fig:cover_example}.

\begin{figure}[t]
	\centering
	\includegraphics[width=0.43\textwidth]{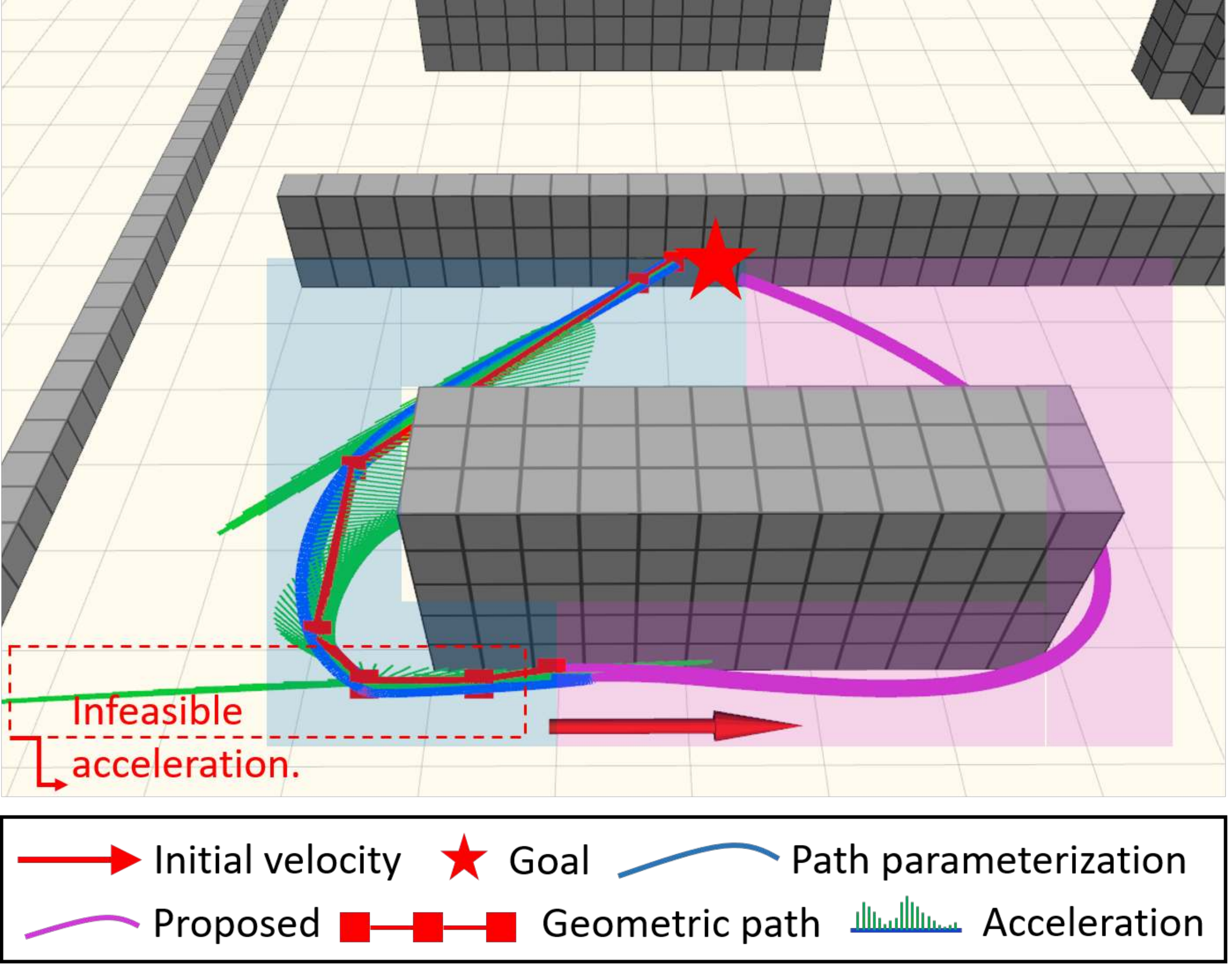}
	\caption{Illustration of the motivating example. The initial state has non-zero velocity (\textit{red arrow}). The traditional geometric planner finds the shortest path (red squares) and then parameterizes it using a piecewise polynomial (\textit{blue}). However, the local path parameterization is restricted to a homotopy class (\textit{blue} area), and the resultant trajectory is jerky (even infeasible) with respect to the given time allocation. In contrast, our framework produces a dynamically feasible trajectory (\textit{purple line}) using kinodynamic planning.}\label{fig:cover_example}
	\vspace{-0.5cm}
\end{figure}

\begin{figure}[h]
	\begin{center}
		\subfigure[Monocular Vision-Based Quadrotor Testbed\label{fig:drone}]{\includegraphics[trim={0cm 0cm 0cm 2.5cm},clip,width=0.44\textwidth]{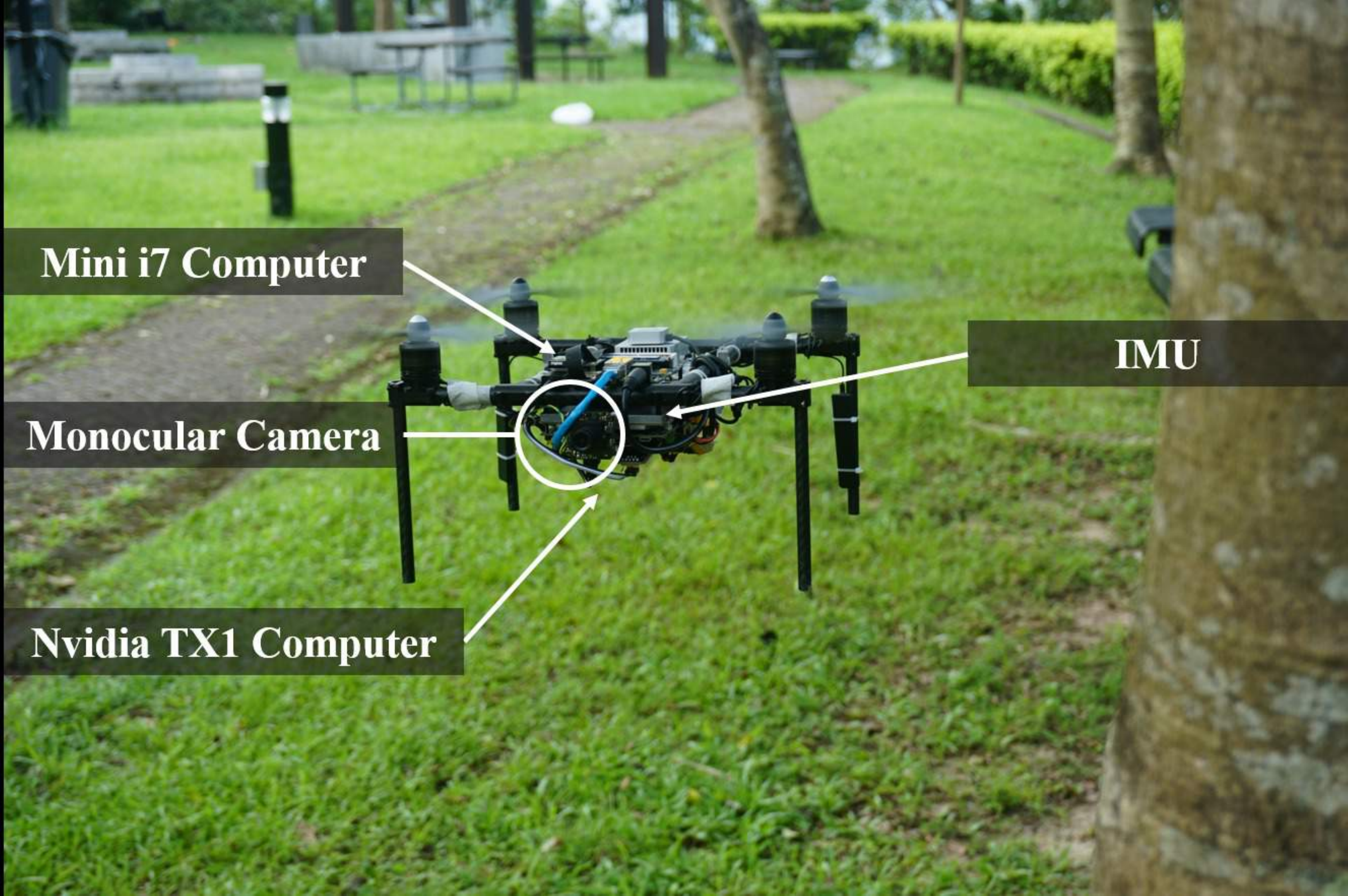}}
		\subfigure[Dual-fisheye Vision-Based Quadrotor Testbed\label{fig:fisheye_drone}]{\includegraphics[trim={0cm 0cm 0cm 2.5cm},clip,width=0.44\textwidth]{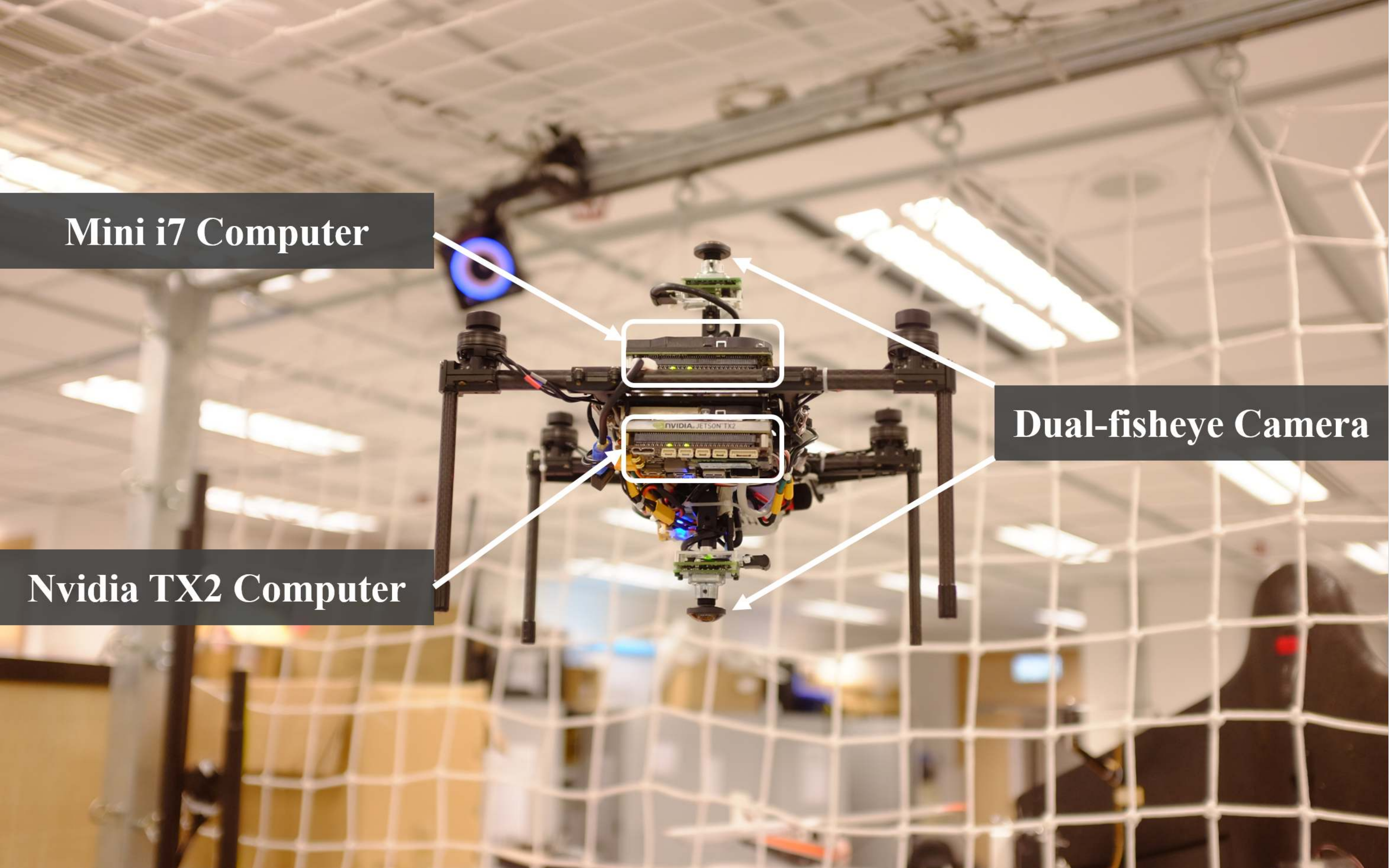}}
	\end{center}
	\caption{Illustration of our (a) monocular vision-based quadrotor testbed and (b) dual-fisheye vision-based quadrotor testbed.}
	\vspace{-0.3cm}
\end{figure}

To address the limitations of the hierarchical methods, it is essential to use kinodynamic motion planners, which directly find time-parameterized trajectories that are globally optimal with respect to control efforts and dynamical limits. The incorporation of kinodynamic planners into replanning facilitates dealing with non-static initial states and enhances replanning consistency.

Sampling-based motion planning algorithms, such as rapidly-exploring random trees (RRTs)~\cite{lavalle2001randomized} and their variants~\cite{webb2013kinodynamic, gammell2015bit,karaman2011optimal, janson2015fast, kuwata2008motion}, are popular in the kinematic/kinodynamic planning literature. Asymptotical optimality has been proved for some of them~\cite{webb2013kinodynamic, karaman2011optimal, janson2015fast, gammell2015bit}. However, when applied to complex kinodynamic systems, they typically require solving a computationally expensive non-linear two-point boundary value problem (BVP)~\cite{xie2015kinobit, li2016sst} and cannot run in real-time. Liu \textit{et al.}~\cite{liu2017smp} explored a search-based counterpart and proposed a heuristic-guided resolution-complete (optimal with respect to discretization) search method using linear quadratic minimum time control. However, for high-order kinodynamic systems or a large dynamic range, the run-time efficiency is inadequate.

In this paper, we present a kinodynamic replanning framework which addresses the efficiency bottleneck. Our framework starts with an efficient B-spline-based kinodynamic (EBK) search algorithm, which finds B-spline control points on a spatial grid. We introduce the novel concept of vertex tuple to keep the search problem simple and analyzable, which enables a thorough theoretical characterization of the problem. Built on top of the structure of the optimal solution, a graph aggregation technique is proposed to minimize the computation time through a controllable discretization of the search space. An offline-computable minimum inflation is adopted to avoid unnecessary collision checking and further accelerate the online search. Compared to state-of-the-art methods (Sect.~\ref{sec:analysis}), our kinodynamic search finds the lowest-cost dynamically feasible trajectories in real time.

To compensate for the discretization in the search, an elastic optimization (EO) approach is proposed to refine the control point placement to the optimal location, by solving a convex quadratically constrained quadratic programming (QCQP) problem. The two components are integrated into a receding horizon replanner based on the local control property of the B-spline.

Our replanning framework is not only theoretically analyzed, but also validated in simulated and onboard experiments. Superior performance is shown through comprehensive comparisons against the state-of-the-art kinodynamic planning and trajectory optimization methods. Moreover, the practical impact of our framework is verified through onboard experiments using a monocular vision-based quadrotor and a dual-fisheye vision-based quadrotor in unknown indoor and outdoor environments. We summarize our contributions as follows:
\begin{itemize}[leftmargin=*]
	\item An EBK search algorithm which provides an initial-state-aware dynamically feasible time-parameterized trajectory.
	\item An EO approach that refines the control point placement to the optimal location while preserving the safety and dynamical feasibility.
	\item Systematic comparisons against the state-of-the-art showing the superior performance of the proposed framework.
	\item Integration of our framework into a real monocular vision-based quadrotor and a dual-fisheye vision-based quadrotor as well as extensive experiments demonstrating fully autonomous navigation in unknown, complex indoor and outdoor environments.
\end{itemize}

A basic version of our framework was originally presented in~\cite{ding18replanning}, where we introduced the real-time B-spline-based kinodynamic (RBK) search. Although the RBK search achieves high computational efficiency, the absence of theoretical optimality analysis in~\cite{ding18replanning} limits confidence in its solution quality and further limits its theoretical impact. In this paper, instead of directly developing efficient methods, we tackle the kinodynamic search problem in a systematic way: 1) We first characterize the complexity and optimal solution of the search problem using a novel vertex tuple structure. 2) We then establish the quality-efficiency tradeoff using a novel graph aggregation technique, which provides a user-specified parameter to control algorithm efficiency and solution quality. The above two theoretical additions render the more flexible and theoretically reliable EBK search. It also turns out that the preliminary version in~\cite{ding18replanning} is perfectly contained in the EBK search. Apart from the theoretical additions, more comprehensive experimental analyses in a wide variety of environments are presented to support the new characteristics.

The relevant literature is discussed in Sect.~\ref{sec:related_works}. An overview of the proposed replanning framework is provided in Sect.~\ref{sec:overview}. Mathematical background and advantageous properties of B-spline are introduced in Sect.~\ref{sec:Bspline_property}. The problem formulation and algorithm detail of the EBK search are elaborated in Sect.~\ref{sec:kinodynamic_search}. The EO approach is presented in Sect.~\ref{sec:elastic_optimization}. Implementation details are given in Sect.~\ref{sec:implementation_details}. Systematic comparisons against the state-of-the-art methods are provided in Sect.~\ref{sec:analysis}, and onboard experimental results are illustrated in Sect.~\ref{sec:experimental}. Finally, a conclusion and further possible research directions are provided in Sect.~\ref{sec:conclusion}.
\begin{figure*}[t]
	\centering
	\includegraphics[width=0.97\textwidth]{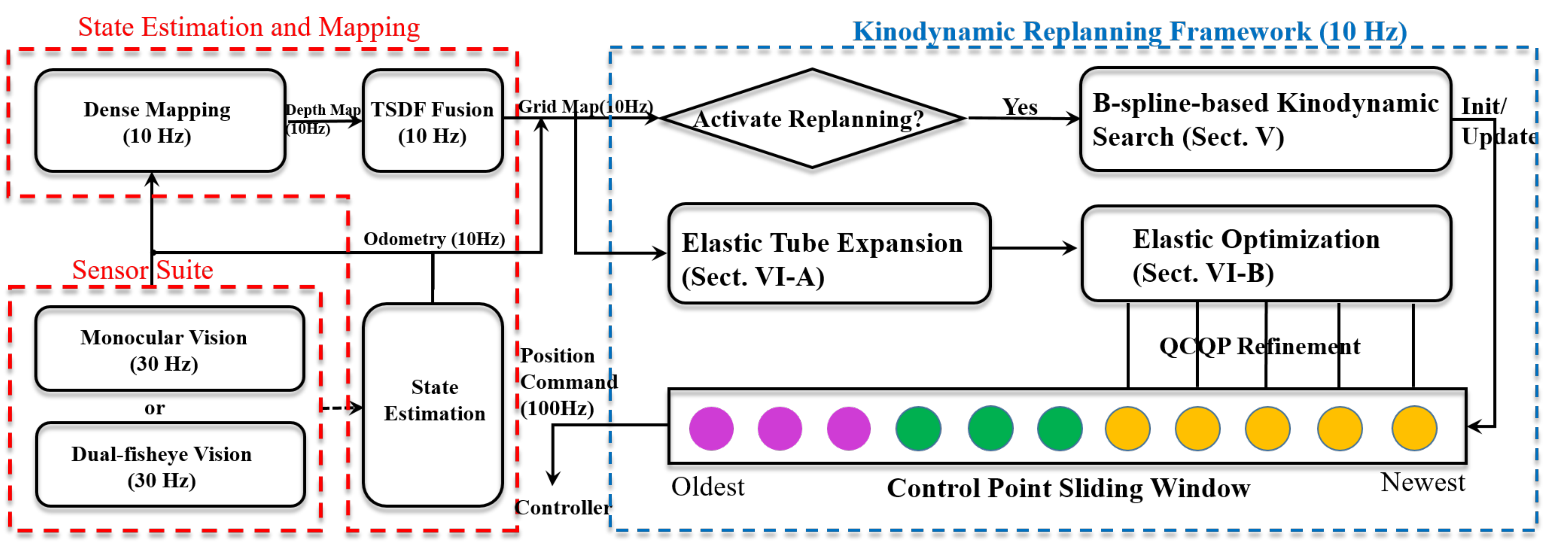}
	\caption{A diagram of our kinodynamic replanning framework together with state estimation and mapping modules.\label{fig:system_overview}}
	\vspace{-0.25cm}
\end{figure*}

\section{Related Works}
\label{sec:related_works}
There is extensive literature on motion planning techniques for quadrotors from various perspectives, such as control-based methods~\cite{van2012lqg, zhou2014vector, bareiss2015stochastic}, search-based methods~\cite{liu2017smp, liu2017search,likhachev2009dstar,aine2016multi}, sampling-based methods~\cite{karaman2011optimal, karaman2010incremental, webb2013kinodynamic, allen2016real, pivtoraiko2013sampleprimitive} and optimization-based methods~\cite{chen2016online, gao2016online, oleynikova2016ct, deits2015iris}. It is difficult to give a full literature review of all these techniques, so in this section, we choose the most relevant and organize them into two categories, namely, hierarchical motion planning techniques and kinodynamic motion planning techniques.

\subsection{Hierarchical Motion Planning}
Hierarchical motion planning refers to a high-level geometric path planner coupled with a low-level time parameterization scheme. The high-level geometric planner is concerned with finding an obstacle-free path, while the low-level parameterization scheme takes care of the vehicle dynamical constraints and generates a time-parameterized trajectory for execution. For quadrotors which have non-trivial dynamics, directly generating a trajectory in the high-dimensional state space is time consuming, while smoothing a given geometric path is computationally efficient (with a suitable relaxation)~\cite{kannan2013close}. As such, the hierarchical framework is popular for quadrotors, and it enables a number of online methods~\cite{mellinger2011minsnap, gao2016online, richter2016polyunqp, liu2017sfc, chen2016online}.

Two pioneering works~\cite{mellinger2011minsnap, richter2016polyunqp} extract waypoints from the geometric path and formulate the trajectory generation problem as quadratic programming (QP) on polynomial coefficients. These methods are based on the differential flatness of the quadrotor~\cite{mellinger2011minsnap}. Due to the deviation of the polynomial trajectory from the straight-line collision-free path, an iterative waypoint insertion scheme is adopted~\cite{richter2016polyunqp}. However, how many additional waypoints are needed is not quantified. Chen \textit{et al.}~\cite{chen2016online} propose a corridor-based geometric planner based on the octree-based map structure~\cite{chen2017improving}. The control effort can be reduced by generating the trajectory in a series of connected cubes. Apart from that, they propose an iterative process of adding constraints on polynomial extremas to cope with the deviation from the corridor, and prove that a finite number of iterations is needed to guarantee safety. Liu \textit{et al.}~\cite{liu2017sfc} further generalize the corridor representation to a series of connected convex polygons.

Although the hierarchical methods have made significant achievements, they suffer from the common problem that the geometric planner is unaware of the vehicle dynamics, resulting in inadequacy between the path planning and path parameterization, especially when faced with non-static initial states. The example in Fig.~\ref{fig:cover_example} motivates us to explore the problem from the kinodynamic planning perspective, which is discussed in Sect.~\ref{sec:review_kinodynamic}.

\subsection{Kinodynamic Motion Planning}\label{sec:review_kinodynamic}
The kinodynamic motion planner directly explores the high-dimensional state space, and outputs a time-parameterized trajectory, which fundamentally avoids the inadequacy between path planning and parameterization. RRTs~\cite{lavalle2001randomized} and their variants~\cite{webb2013kinodynamic, gammell2015bit, karaman2010incremental,karaman2011optimal, janson2015fast, kuwata2008motion} were originally designed for kinematic systems and can be easily extended to kinodynamic systems. These methods provide an efficient way of exploring the high-dimensional state space, and some of them possess asymptotical optimality~\cite{webb2013kinodynamic, karaman2011optimal, janson2015fast, gammell2015bit}. However, for robots with non-trivial dynamics, the tree expansion typically involves solving the BVP, which is non-linear and challenging. Webb~\textit{et al.}~\cite{webb2013kinodynamic} propose a fixed-final-state-free-final-time optimal controller which solves the BVP for linear (or linearized) controllable systems in closed form. Xie~\textit{et al.}~\cite{xie2015kinobit} propose an efficient BVP solver for general kinodynamic systems using sequential quadratic programming (SQP). Li~\textit{et al.}~\cite{li2016sst} work in another direction, namely, expanding the tree using random control propagation, for cases where system models are complex and BVP solvers are not available.

Despite the fact that the efficiency of the kinodynamic planning techniques keeps improving~\cite{xie2015kinobit, li2016sst}, it is still prohibitively expensive for replanning. Allen~\textit{et al.}~\cite{allen2016real} work towards a real-time kinodynamic planning framework by combining FMT*~\cite{janson2015fast} with a support vector machine (SVM) for the classification of the reachable set. This framework~\cite{allen2016real} reduces the calling of the BVP solver to gain efficiency. However, the solution quality largely depends on the number of states pre-sampled. On the other hand, Liu~\textit{et al.}~\cite{liu2017smp} explore the search-based kinodynamic planning counterpart and develop efficient heuristics by solving a linear quadratic minimum time problem. Their solution is resolution-complete with respect to the discretization on the control input, and achieves near real-time performance. Note that both~\cite{liu2017smp} and~\cite{allen2016real} use a simplified system model, i.e., a double or triple integrator, to reduce the computation complexity. However, the resultant trajectory only has limited continuity. To improve the smoothness, both~\cite{liu2017smp} and~\cite{allen2016real} adopt trajectory reparameterization using the unconstrained QP formulation~\cite{richter2016polyunqp}, which may break the dynamical feasibility and safety.

In contrast, the proposed EBK search adopts a high-order B-spline parameterization with continuity up to snap, which can be directly used to control the quadrotor. Moreover, the advantageous properties of the B-spline facilitate the kinodynamic replanning as follows:
\begin{itemize}
	\item Local control property for incrementally constructing the B-spline trajectory in the kinodynamic search and local refinement of the trajectory during replanning.
	\item Convex hull property for enforcing collision-free constraints and providing a dynamical feasibility guarantee for the entire trajectory.
\end{itemize}

Given the same run-time budget according to the real-time requirement, the EBK search finds a lower-cost trajectory, as validated by the comprehensive comparisons against the state-of-the-art in Sect.~\ref{sec:analysis}. Apart from this, the proposed refinement module, namely, the EO approach, can preserve the safety and dynamical feasibility by taking advantage of the convex hull property of the B-spline.

\section{Overview}\label{sec:overview}
The structure of the proposed framework is shown in Fig.~\ref{fig:system_overview}. The replanning framework is built on top of the state estimation and dense mapping module, which are discussed in Sec.~\ref{sec:implementation_details}.
The frequency of the grid map update is $10$ Hz. As such, the real-time requirement in this paper refers to a run-time of less than $100$ ms. The updated map and the initial state of the quadrotor are fed to the EBK search module (Sect.~\ref{sec:review_kinodynamic}), and the replanning strategy is elaborated in Sect.~\ref{sec:implementation_details}. The control points are constantly refined by the proposed EO approach (Sect.~\ref{sec:elastic_optimization}), which consists of an elastic tube expansion module (Sec.~\ref{sec:elastic_expansion}) for free space characterization and a convex optimization formulation (Sec.~\ref{sec:elastic_opt_formulation}) for trajectory refinement. The confirmed control points are evaluated, and position commands are generated accordingly.

\section{B-spline Curve and Replanning}\label{sec:Bspline_property}
For the proposed kinodynamic planning framework, we adopt a B-spline parameterization for its advantageous properties, namely, local control and convex hull property, and we further adopt the uniform B-spline for its convenient closed-form evaluations. In this section, we elaborate these properties and explain how they can be applied to the replanning system.

Given $n+1$ control points $\mathbf{p}_0, \mathbf{p}_1, \ldots, \mathbf{p}_n$ and knot vector $\{t_0, t_1, \ldots, t_m \}$, the B-spline curve $\mathbf{s}(t)$ of degree $k$ is defined as follows:
\begin{equation}
	\mathbf{s}(t) = \sum_{i=0}^{n} \point{p}_i N_{i,k}(t),
\end{equation}
where $N_{i,k}(t)$ is the B-spline blending function of degree $k$, which can be evaluated recursively as follows:
\begin{equation}
	\begin{aligned}
		N_{i,0}(t) &=
				\begin{cases}
					1 &\text{if $t_i \leq t < t_{i+1}$}\\
				    0 &\text{otherwise}
				\end{cases}\\
		N_{i,k}(t) &= \frac{t-t_i}{t_{i+k}-t_i}N_{i,k-1}(t) + \frac{t_{i+k+1}-t}{t_{i+k+1}-t_{i+1}}N_{i+1,k-1}(t).
	\end{aligned}
\end{equation}

The total number of knots should satisfy $m+1 = n+k+2$. The uniform B-spline is a special type of B-spline whose knot vector is uniformly distributed. Suppose the knot vector is separated with equidistance $\Delta_t$. The half-open interval $[t_i, t_{i+1})$ is called the $i-$th knot span. We normalize each knot span using $u = (t-t_i)/\Delta_t$, and for the $i$-th knot span, only $k+1$ blending functions are non-zero, corresponding to $k+1$ control points $\mathbf{p}_{i-k}, \ldots, \mathbf{p}_i$. We stack the $k+1$ control points and call the stacked coordinate matrix a control point span $\mathbf{P}_{i-k} \coloneqq {\left[ \mathbf{p}_{i-k} \, \mathbf{p}_{i-k+1} \cdots \mathbf{p}_{i} \right]}^ {\intercal} \in \mathbb{R}^{ (k+1)\times3}$. Since the blending functions $N_{i,k}(t)$ are shifted versions of each other for the uniform B-spline, we have closed-form matrix representations~\cite{qin2000bsplineMatrix} for parametric evaluation. Let $j=i-k$, and the position and the derivatives of the B-spline curve corresponding to the $j$-th control point span can be evaluated as follows:
\begin{equation}
	\label{eq:bspline_derivative}
	\frac{d \mathbf{s}_j(u)}{d^{l}u} = \frac{1}{ {(\Delta t)}^l} \frac{d \mathbf{b}^{\intercal}}{d^l u} \mathbf{M}_k \mathbf{P}_j,
\end{equation}
where $l$ denotes the order of the derivative ($l=0$ means the position), $\mathbf{b} = {\left[ 1 \, u \, u^2\, \cdots\, u^{k} \right]}^{\intercal} \in \mathbb{R}^{k+1}$ denotes the basis vector, and $\mathbf{M}_k = (m_{i,j})\in \mathbb{R}^{ (k+1)\times(k+1)}$ denotes the blending matrix, where $m_{i,j} = \frac{1}{k!}{k \choose k-i}\sum_{s=j}^{k}{(-1)}^{s-j}{k+1 \choose s-j}{(k-s)}^{k-i}$. According to Eq.~\ref{eq:bspline_derivative}, the evaluation of the derivatives of the B-spline curve can be expressed by a linear matrix multiplication in terms of the control point span $\mathbf{P}_j$. The paper uses a quintic uniform B-spline ($k=5$) to ensure the continuity up to snap for controlling quadrotors.

As described by Mellinger\cite{mellinger2011minsnap}, the control cost of a quadrotor is closely related to the integral over squared derivatives of the planned trajectory, which can also be evaluated in closed form in the case of the uniform B-spline. The total control cost $E_j^l$ of the $j$-th control point span can be expressed by the integral over the squared derivatives of degree $l$ (e.g., for the min-snap trajectory, $l=4$) as follows:
\begin{equation}
	\label{eq:control_cost}
	E_j^l = \int_{0}^{1} {\left( \frac{d \mathbf{s}_j(u)}{d^{l}u} \right)^2 du}
	= \mathbf{P}_{j}^{\intercal} \mathbf{M}_k^{\intercal} \mathbf{Q}_l \mathbf{M}_k \mathbf{P}_j,
\end{equation}
where $\mathbf{Q}_l = \int_{0}^{1} \left( \frac{d \mathbf{b}}{d^l u} \right) {\left( \frac{d \mathbf{b}}{d^l u} \right)}^{\intercal} du /{(\Delta_t)}^{2l-1}$ is the Hessian matrix of the $l$-th squared derivative, which is constant for the uniform B-spline. The control cost $E_j^l$ is quadratic with respect to the control point span $\mathbf{P}_j$. Note that the cost evaluation of a span only depends on the stacked control point coordinates of this span.

\subsection{Local Control Property and Replanning}
The \textit{local control} is one of the important properties of B-spline, making it suitable for replanning. Specifically, the evaluation of any point of the B-spline curve is controlled by a single control point span containing $k+1$ control points, and any control point only affects $k+1$ control point spans. We incorporate the local control property into a receding horizon (re-)planner, and we divide the planned trajectory into three types, namely, executed trajectory, executing trajectory and optimizing trajectory. The executed trajectory means the part of the trajectory which has already been executed, the executing trajectory means the part of the trajectory corresponding to the control point span being executed, and the optimizing trajectory means the part of the trajectory whose supporting control points are potentially under optimization.

Thanks to the local control property, modification of the supporting control points of the optimizing trajectory will not affect the evaluation of the executing trajectory, as shown in Fig.~\ref{fig:bspline_replan}. Unlike~\cite{yang2010analytical} and~\cite{yang2015generation} where local reshaping may cause the violation of dynamical constraints, the dynamical feasibility of the executing trajectory can be preserved by leveraging the local control property. Moreover, the locality also makes it possible to optimize any subset of control points without re-generation of the whole trajectory, which is computationally efficient. The locality also helps to preserve a smooth trajectory since the next executing control point span always shares $k$ control points with the current executing control point span, yielding $k$-th-order continuity and a consistent trajectory.
\begin{figure}[t]
	\centering
	\includegraphics[width=0.35\textwidth]{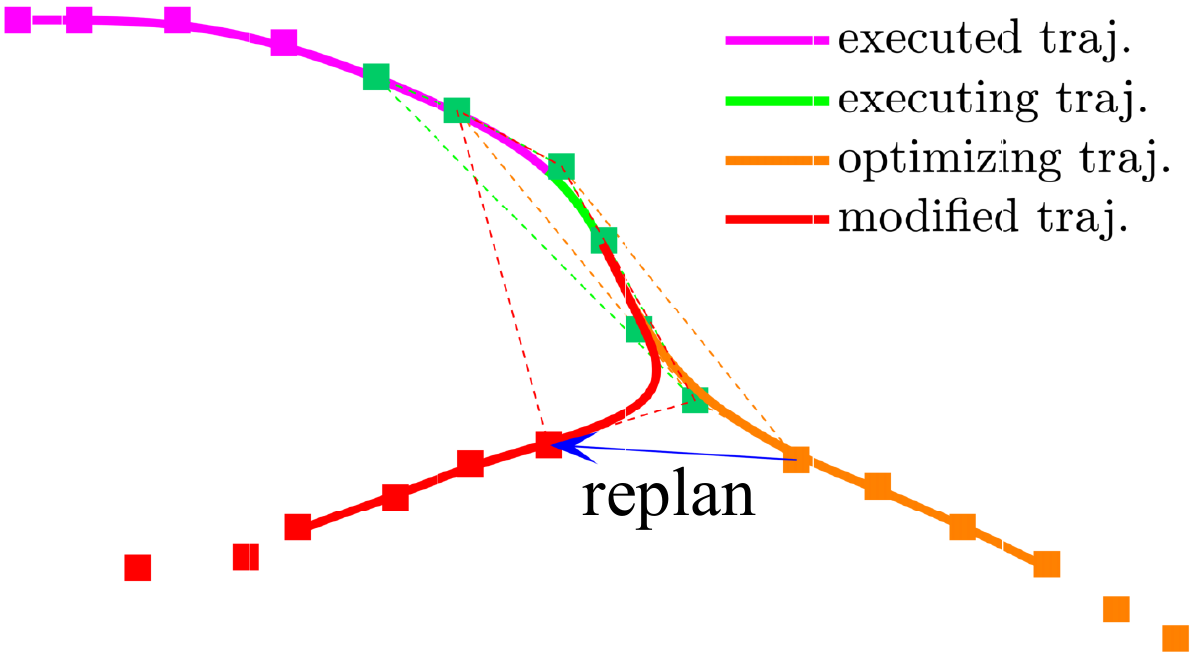}
	\caption{Illustration of the B-spline local control property and its application to the replanning system. The control points are shown by squares. The control points corresponding to the executing trajectory are shown in \textit{green}. The original locations of the control points ahead of the executing control point are in \textit{orange}, and their subsequently modified positions are marked in \textit{red}. Due to the locality of the B-spline, the replan causes no perturbation to the executing trajectory.\label{fig:bspline_replan}}
	\vspace{-0.5cm}
\end{figure}
\subsection{Convex Hull Property and Dynamical Feasibility}
Another important property of the B-spline is the \textit{convex hull} property. A B-spline (or B\`{e}zier~\cite{ding2019safe}) trajectory is strictly bounded inside the convex hull supported by the corresponding control point span. Strictly speaking, dynamical feasibility should be induced by the robot's kinematic and dynamical constraints. Given polynomial/spline parameterization, a common practice to enforce dynamical feasibility is to use maximum velocity and maximum acceleration bounds~\cite{liu2017sfc, feigao2017hg, usenko2017bspgradient}. We follow this practice in this paper. Note that for piecewise polynomial parameterization methods~\cite{liu2017sfc, gao2016online}, the dynamical feasibility constraints are enforced on a finite number of checkpoints. Denser checkpoints will enhance the robustness but yield higher computation complexity. In contrast, by using the convex hull property, the \textit{entire} velocity and acceleration profile can be strictly bounded.

We utilize the fact that the derivative of the B-spline of degree $k$ is a B-spline of degree $k-1$, which also enjoys the convex hull property. Therefore, if the supporting control points are bounded inside the convex hull expanded by the allowed maximum derivative, the derivative spline is subsequently bounded, as elaborated in Prop.~\ref{prop:linear_derivative_bound}. Note that Prop.~\ref{prop:linear_derivative_bound} is a sufficient but not necessary condition. A toy example illustrating the relation between the convex hull property and dynamical feasibility is shown in Fig.~\ref{fig:bspline_chull}.
\begin{prop}\label{prop:linear_derivative_bound}
	Given a uniform B-spline of degree $k$ and knot separation $\Delta_t$, there exists a constant linear combination $\mathbf{S}_l = \mathbf{M}_k^{-1} \mathbf{C}_l \mathbf{M}_k /{(\Delta_t)}^{l} \in \mathbb{R}^{(k+1) \times (k+1)}$ such that  $ u_{l,D}^{\text{min}}\bm{1}_{(k+1)\times 1} \leq \mathbf{S}_l\mathbf{P}^{D} \leq u_{l,D}^{\text{max}}\bm{1}_{(k+1)\times 1} $ is a sufficient condition for the derivative along coordinate $D$ to be thoroughly bounded; i.e., $u_{l,D}^{\text{min}} \leq \frac{d \mathbf{s}^D(u)}{d^l u} \leq u_{l,D}^{\text{max}}$, $\forall u\in [0,1]$, where $\mathbf{P}^{D} \in \mathbb{R}^{k+1}$ is coordinate $D \in \left\lbrace x, y, z\right\rbrace $ of the control point span $\mathbf{P}$, and $\mathbf{C}_l \in \mathbb{R}^{ (k+1)\times (k+1)}$ is a constant mapping matrix of the $l$-th derivative satisfying $\frac{d\mathbf{b}}{d^l u} = \mathbf{C}_l \mathbf{b}$.
	\footnote{ $\mathbf{S}_l$ and $\mathbf{C}_l$ are fixed given the degree $k$ and the derivative order $l$.}
\end{prop}
\begin{proof}
	Please refer to Appendix~\ref{sec:appendix_prop_bound} for the detailed proof.
\end{proof}

\begin{figure}[t]
	\centering
	\subfigure{\includegraphics[trim={0cm 0cm 0cm 0cm},clip,width=0.42\textwidth]{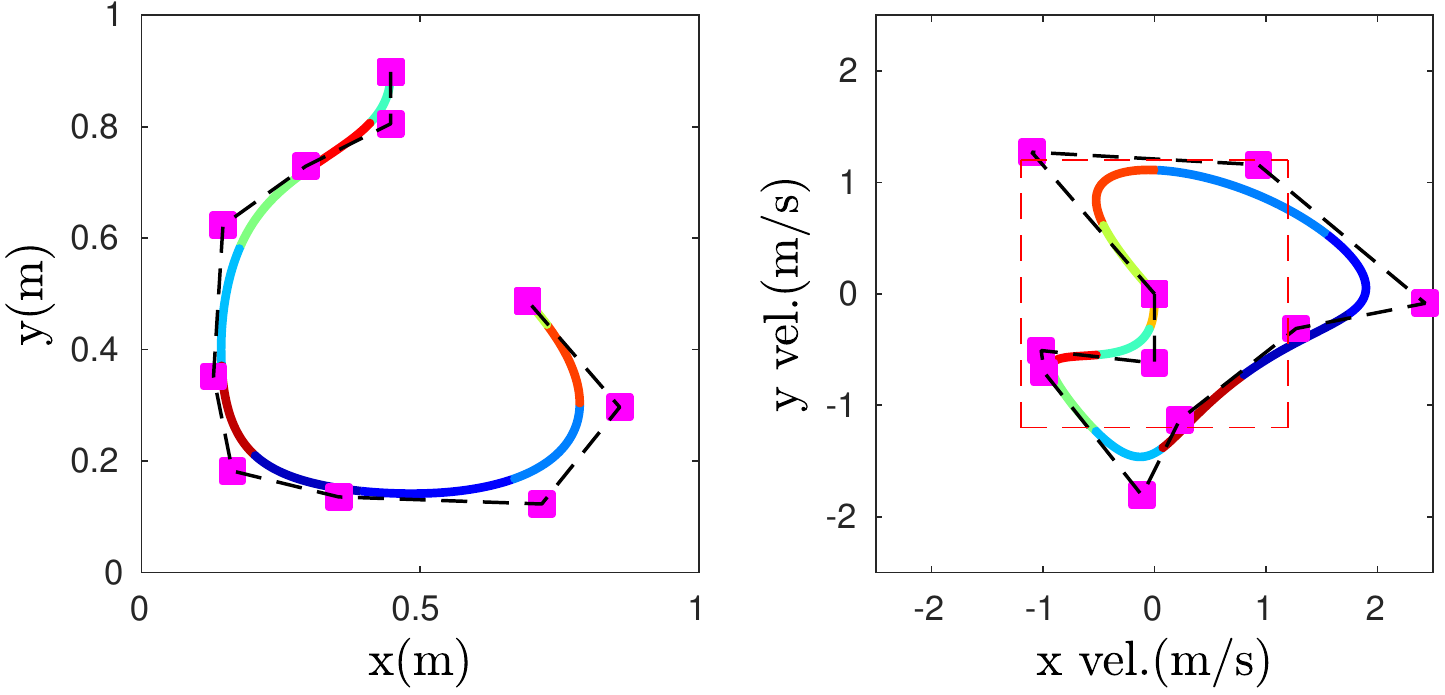}}
	\subfigure{\includegraphics[trim={0cm 0cm 0cm 0cm},clip,width=0.42\textwidth]{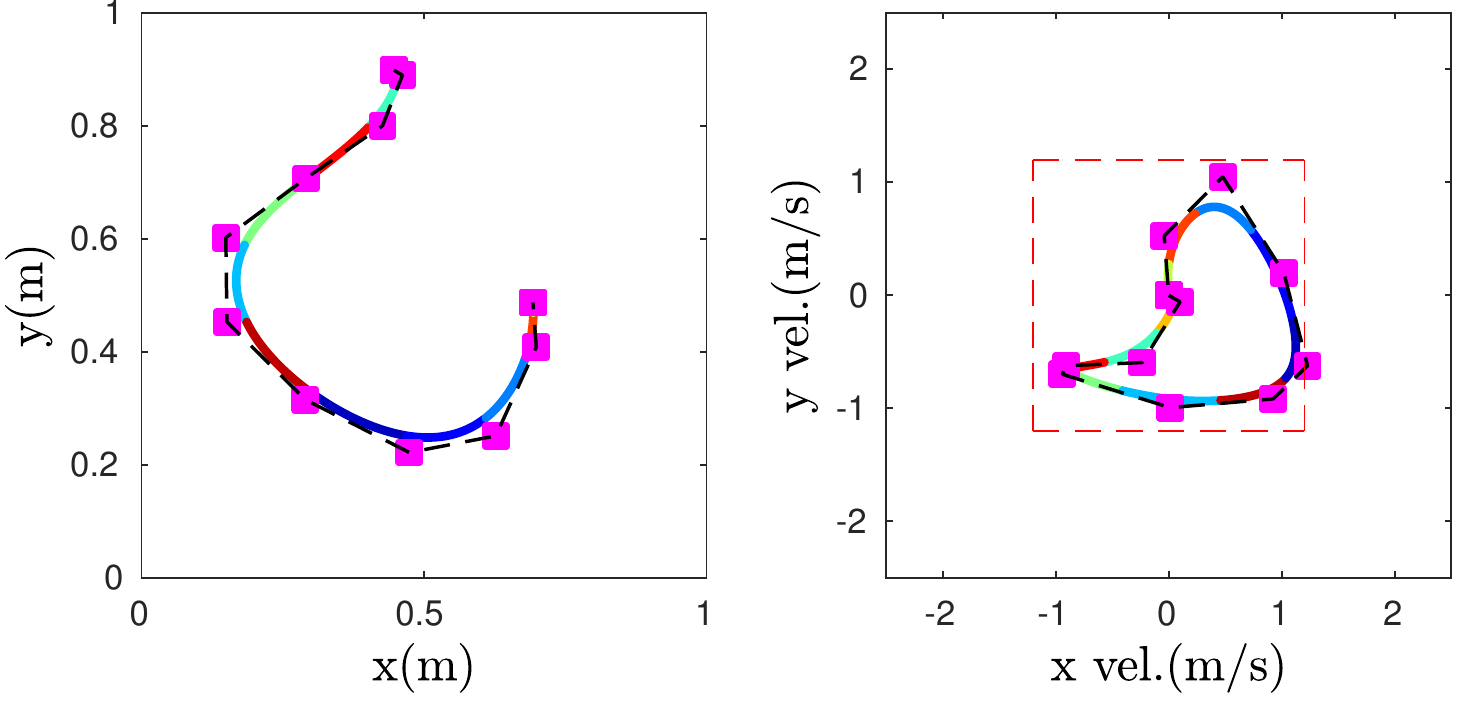}}
	\caption{Illustration of the convex hull property. The dashed $red$ box shows the feasible velocity hull (1.2 $m/s$ for each axis). Applying Prop.~\ref{prop:linear_derivative_bound} under the objective of minimum change to control point positions (bottom left figure), the resulting velocity profile is shown at the bottom right, where the velocity profile is strictly bounded.\label{fig:bspline_chull}}
	\vspace{-0.4cm}
\end{figure}

\section{B-spline-Based Kinodynamic Search}\label{sec:kinodynamic_search}
\subsection{Motivating Example}\label{sec:motivaing_example}
As shown in the motivating example in Fig.~\ref{fig:cover_example}, hierarchical motion planners may produce sub-optimal or even dynamically infeasible trajectories given a non-static initial state of the quadrotor. The reason is that the geometric planner has no knowledge about the vehicle dynamics and restricts the solution space of path parameterization to a \textit{homotopy class} of the geometric shortest path. The inadequacy motivates us to propose an efficient kinodynamic planning algorithm which can work in real-time. However, kinodynamic planning is typically time consuming~\cite{allen2016real, liu2017smp}. The major computation of the traditional kinodynamic planning lies in three tasks, namely, covering the large state space, solving the BVP and collision checking.

Given the advantageous properties of the B-spline introduced in Sect.~\ref{sec:Bspline_property}, we propose using uniform B-spline parameterization in kinodynamic planning, which facilitates reducing the computation time for the above three tasks. Specifically, we propose a spatial-grid-based deterministic graph search to place B-spline control points. The proposed search algorithm has three major features as follows:
\begin{itemize}
	\item Controllable discretization of the state space: Due to the locality of the B-spline, it is possible to incrementally sample the B-spline control points during the search. A vertex tuple structure is proposed to recover the Markovian assumption and make the problem analyzable. A novel graph aggregation technique is proposed to control the discretization of the state space, which achieves a speed-quality tradeoff.
	\item Closed-form evaluations of control cost and dynamical feasibility: The control cost and feasibility of the B-spline can be evaluated in closed forms efficiently.
	\item Offline-computable inflation to avoid collision checking: The maximum deviation of uniform B-spline from the free-cells can be characterized offline and compensated for by workspace inflation.
\end{itemize}
The kinodynamic search accounts for the total control efforts and dynamical limits, which is a systematic way to deal with the non-static initial states in replanning.

\subsection{Problem Formulation}\label{sec:problem_formulation}
In this section, we formally present the problem of the B-spline-based kinodynamic search on a spatial grid. The problem is formalized as a deterministic graph search where the action is the placement of the control points. Suppose the topological graph associated with the grid map is denoted as $\mathcal{G} \coloneqq (V,E)$, where $V$ is the set of vertices denoting the collection of free cells and $E$ denotes the set of edges $(i,j)\subset V\times V$ between all adjacent vertices $i$ and $j$. The adjacency of the vertices depends on the grid connectivity adopted. In this paper, every cell in the 3-D grid has 26 neighbors which are cells connected to this cell at a Chebyshev distance of 1.\footnote{Note that the connectivity can also be defined based on a Chebyshev distance larger than one, which will result in a non-uniform control point placement and a higher computational complexity.}

Given uniform B-spline parameterization of degree $k$ and knot separation $\Delta_t$, the proposed search method finds a finite sequence $\pi = (v_0, v_1,\ldots, v_T)$ of vertices representing an \textit{admissible} control point placement which satisfies $v_i \in V, (v_{i-1}, v_i)\in E$ for each $i=1,\ldots,T$ and connects the given initial state $\pi_s = (v_s^0, v_s^1, \ldots, v_s^k)$ and goal state $\pi_g=(v_g^0, v_g^1, \ldots, v_g^k)$, i.e., $(v_s^k, v_0) \in E$ and $(v_T, v_g^0)\in E$. The sequence $\pi$ possibly contains repetition since we allow placing the control points at the same cell.
$\pi_s$ and $\pi_g$ are two tuples, both containing $k+1$ vertices which form the control point span according to the definition of the B-spline in Sect.~\ref{sec:Bspline_property}. Therefore, $\pi_s$ and $\pi_g$ actually represent two short trajectories, different from the initial and goal positions used in geometric planners and the position-velocity-acceleration state vector used in these kinodynamic planners~\cite{webb2013kinodynamic, allen2016real, liu2017smp}.

Since the B-spline is evaluated in terms of the control point span, we re-organize $\pi$ by combining neighboring $k+1$ vertices as one vertex tuple. Combining the sequence $\pi$ with $\pi_s$ and $\pi_g$, the overall sequence can be formalized as $\tilde{\pi}=(v_s^0,\ldots,v_s^k, v_0, \ldots, v_T, v_g^0, \ldots, v_g^k)$. We define the \textit{ordered} sub-sequence containing consecutive $k+1$ vertices of $\tilde{\pi}$ as a $k$-\textit{degree vertex tuple}, which is denoted by ${[\tilde{\pi}]}^k$.
We provide a toy example in Fig.~\ref{fig:vertex_tuple} showing how $3$-degree vertex tuples can be constructed.

We denote by ${[\tilde{\pi}]}^k_j$ the $j$-th $k$-degree tuple in $\tilde{\pi}$, and two neighboring tuples, ${[\tilde{\pi}]}^k_{j-1}$ and ${[\tilde{\pi}]}^k_{j}$, overlap for $k$ vertices. According to this convention, ${[\tilde{\pi}]}^k_0 = \pi_s$ and ${[\tilde{\pi}]}^k_{J} = \pi_g$, where $J+1=k+T+3$. Each vertex tuple represents a short trajectory, and a sequence of neighboring vertex tuples represents a continuous trajectory. We associate with each $k$-degree tuple $[\tilde{\pi}]$ a \textit{strictly positive} cost function $f_{k,\Delta_t}:{[\tilde{\pi}]}^k \to \mathbb{R}_+$. Note that the cost function has to be strictly positive to cope with the possible repetition in $\pi$. We state the problem as follows:
\begin{problem}\label{prob:optimal_search}
	\normalfont Given a uniform B-spline of degree $k$ and knot separation $\Delta_t$, initial state $\pi_s$ and goal state $\pi_g$, find admissible control point placement $\pi=(v_0, v_1,\ldots, v_T)$ on $\mathcal{G}$ such that the following cost function is minimized:
	\vspace{-0.1cm}
	\begin{equation*}
	\mathcal{J}_{k,\Delta_t}(\tilde{\pi}) = \sum_{j=0}^{J}f_{k,\Delta_t}({[\tilde{\pi}]}^k_j),
	\end{equation*}
	\vspace{-0.1cm}
	where ${[\tilde{\pi}]}^k_j$ is a $k$-degree vertex tuple, whose corresponding trajectory should satisfy collision-free and dynamical feasibility requirements.
\end{problem}

We adopt a cost function following the idea of the linear quadratic minimum time control problem~\cite{liu2017smp, verriest1991linear}, where the control cost is represented by Eq.~\ref{eq:control_cost} with a tradeoff penalty on the execution time $\Delta_t$. Mathematically, the cost function $f_{k,\Delta_t}$ is represented as follows:
\vspace{-0.1cm}
\begin{equation}\label{eq:node_cost}
	f_{k,\Delta_t}({[\tilde{\pi}]}^k_j)=\lambda \Delta_t + \int_{0}^{1} {\left( \frac{d \mathbf{s}_{{[\tilde{\pi}]}^k_j}(u)}{d^{l}u} \right)}^2 du,
\end{equation}
\vspace{-0.1cm}
where ${[\tilde{\pi}]}^k_j$ can be rewritten into matrix form, as in Sect.~\ref{sec:Bspline_property}; the integral can be evaluated according to Eq.~\ref{eq:control_cost}; and $\lambda$ is the weight for the trajectory execution time. The criterion to check the collision-free and dynamical feasibility of the vertex tuple is introduced in Sect.~\ref{sec:feasibility_condition}.

\subsection{Optimal B-spline-Based Kinodynamic Search}\label{sec:search_algorithm}
The difficulty of solving Prob.~\ref{prob:optimal_search} is that the placement of any control point depends on the placed $k$ control points (not only the predecessor) within the vertex tuple. Regarding the placement of a single control point as an action, the Markovian assumption does not hold, which makes traditional graph search algorithms inapplicable. However, we find that Problem~\ref{prob:optimal_search} can be transformed into an equivalent standard shortest path problem on a higher dimensional directed graph $\mathcal{G}_H=(V_H,E_H)$ induced by $\mathcal{G}$ by regarding the whole vertex tuple as a ``vertex''. Since the vertex tuple is the basic evaluation unit of the B-spline, the placement of the next vertex tuple will only depend on its predecessor, which follows the Markovian assumption. The transformation enables the usage of well-characterized shortest path search algorithms, such as Dijkstra's~\cite{dijkstra1959note} and A*~\cite{hart1968astar}. In the following, we elaborate the transformation.
\begin{figure}[t]
	\begin{center}
		\includegraphics[trim={0cm 0cm 0cm 0cm},clip,width=0.40\textwidth]{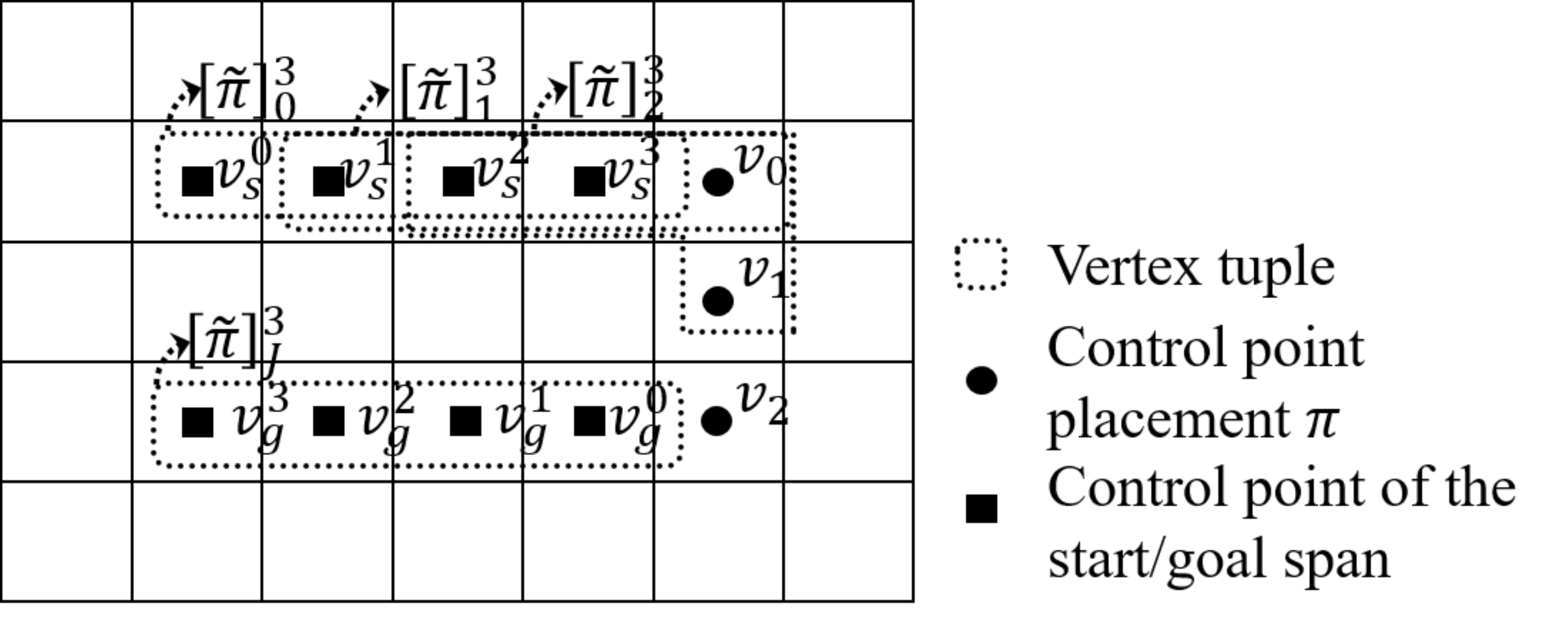}
	\end{center}
	\caption{Illustration of the construction process of 3-degree vertex tuples from an admissible path. Each vertex tuple is formed by combining four consecutive control points. ${[\tilde{\pi}]}_0^{3}$ and ${[\tilde{\pi}]}_1^{3}$ are two neighboring vertex tuples since they overlap for three vertices.\label{fig:vertex_tuple}}
	\vspace{-0.5cm}
\end{figure}

Similar to the construction process of the vertex tuple in Sect.~\ref{sec:problem_formulation}, we construct $\mathcal{G}_H$ as follows: 1) each vertex tuple $\hat{v}_i=(v_i, v_{i+1}, \ldots, v_{i+k}) \in V_H$ is created by combining a sequence of adjacent vertices of $\mathcal{G}$ satisfying $(v_{j-1},v_{j})\in E$ for $j=i+1,\ldots, i+k$; and 2) two vertices $\hat{v}_i$ and $\hat{v}_j$ on $V_H$ are adjacent if and only if the last $k$ vertices of $\hat{v}_i$ overlap with the first $k$ vertices of $\hat{v}_j$. The construction of $\mathcal{G}_H$ groups the associated control point coordinates into a high-dimensional state. Note that each dimension of the combined state has the same physical meaning, i.e., the spatial coordinates of the control point. This observation motivates us to come up with a low-dispersion search algorithm, as presented in Sect.~\ref{sec:modification_efficiency}.

In Fig.~\ref{fig:graph_construction}, we show an example of how the trace of control points on $\mathcal{G}$ can be mapped to the path on $\mathcal{G}_H$. The initial vertex tuple $\hat{v}_s$ is shown by squares. Starting from $\hat{v}_s$, we consider two directions of the next-step placement, which form $\hat{v}_0$ and $\hat{v}_1$, respectively. In a similar way, starting from $\hat{v}_0$, two expansion directions are considered, forming $\hat{v}_3$ and $\hat{v}_4$, while for $\hat{v}_1$, one direction ($\hat{v}_5$) is considered. Note that although the last vertex of $\hat{v}_4$ and $\hat{v}_5$ are in the same spatial cell, in the induced high-dimensional graph $\mathcal{G}_H$, they are two distinct vertex tuples. Following the expansion of control points, we form a tree of vertex tuples on $\mathcal{G}_H$. From the expansion process, we observe that, given the initial and goal vertex tuple, the problem of finding the optimal control point placement is equivalent to finding the shortest path on the graph $\mathcal{G}_H$. We refer interested readers to Appendix~\ref{sec:appendix_graph_relation} for further details.

\begin{algorithm}[t]
	\begin{algorithmic}[1]
	\Function{Initialize}{$\pi_s, \pi_g, k, \Delta_t$}
	\State $\mathcal{O} \leftarrow \emptyset$; $\mathcal{L} \leftarrow \emptyset$
	\State $n \leftarrow \Call{Index}{\pi_s}$;
	\State $h(\pi_s) \leftarrow \Call{Heuristic}{\pi_s, \pi_g}$; $g(\pi_s) \leftarrow f_{k,\Delta_t}(\pi_s)$;
	\State $\mathcal{O} \leftarrow \Call{Insert}{\mathcal{O}, g(\pi_s)+h(\pi_s), \pi_s} $;
	\State $\mathcal{L} \leftarrow \Call{Insert}{\mathcal{L}, n, \pi_s}$
	\EndFunction
	\Function{Main}{${\pi_s}, {\pi_g}, {\mathcal{G},k, \Delta_t}$}
	\State$(\mathcal{O}, \mathcal{L}) \leftarrow \Call{Initialize}{\pi_s, \pi_g, k, \Delta_t} $;
	\While{$\mathcal{O}\not= \emptyset$}
		\State ${(m, \hat{v}_i)} \leftarrow \Call{Pop}{\mathcal{O}} $;
		\If{$m = \Call{Index}{\pi_g}$}
			\State \textbf{return} \textbf{success}
		\EndIf
		\For{$\hat{v}_j \in  \Call{\underline{FeasibleSuccs}}{\hat{v}_i, \mathcal{G}, k, \Delta_t} $}
			\State $n \leftarrow \Call{\underline{Index}} {\hat{v}_j} $;
			\If{\textbf{not} $\Call{Visited}{n, \mathcal{L}} $ }
				\State $g(\hat{v}_j) \leftarrow \infty$
			\EndIf
			\If{$g(\hat{v}_j) > g(\hat{v}_i) + \underline{f_{k,\Delta_t}(\hat{v}_j)} $ }
				\State $g(\hat{v}_j) \leftarrow g(\hat{v}_i) + f_{k,\Delta_t}(\hat{v}_j) $
				\State $h(\hat{v}_j) \leftarrow \Call{Heuristic}{\hat{v}_j, \pi_g}$
				\State $\mathcal{O} \leftarrow \Call{Insert}{\mathcal{O}, g(\hat{v}_j)+h(\hat{v}_j), \hat{v}_j}$
			\EndIf
		\EndFor
	\EndWhile
	\EndFunction
	\end{algorithmic}
\caption{\label{algo:optimal_search} Optimal B-spline-Based Kinodynamic Search}
\end{algorithm}

Recall that in Sect.~\ref{sec:problem_formulation}, we allow the repetition of vertices since some necessary vertex tuples rely on repetition; for instance, the same vertex being repeated $k+1$ times actually represents a static state (for $\Delta_t$). According to the definition of $f_{k,\Delta_t}({[\tilde{\pi}]}^k)$, the cost of $\mathcal{G}_H$ is defined on vertices $V_H$ instead of the edges. Although repetition is allowed, each vertex $\hat{v} \in V_H$ of $\mathcal{G}_H$ is associated with a \textit{strictly positive} cost so that repetition is properly penalized.

Note that any path on $\mathcal{G}_H$ is a time-parameterized B-spline trajectory instead of a geometric path. The dynamical feasibility and control cost are taken into account by evaluating the corresponding short trajectories of the nodes of $\mathcal{G}_H$.
Problem~\ref{prob:optimal_search} can be solved optimally using traditional label-correcting algorithms such as Dijkstra's~\cite{dijkstra1959note} and A*~\cite{hart1968astar} on the induced graph $\mathcal{G}_H$. The optimal control point placement $\pi^*$ can then be re-constructed. The optimal B-spline-based kinodynamic (OBK) search algorithm is outlined in Algo.~\ref{algo:optimal_search}.

Typically, label-correcting algorithms maintain one or multiple sets of vertices as so-called \textit{fringes}~\cite{russell2016artificial}. For example, A*~\cite{hart1968astar} maintains two fringes (known as the \textit{OPEN} set and the \textit{CLOSED} set) to reduce the expansion of nodes and save computation. Similarly, Algo.~\ref{algo:optimal_search} maintains two fringes, namely, the \textit{OPEN} set (as denoted by $\mathcal{O}$) and the \textit{VISITED} set (as denoted by $\mathcal{L}$). The visited set $\mathcal{L}$ provides query and retrieving functions for vertices. We use \textproc{Index}$(\hat{v})$ (Algo.~\ref{algo:index_optimal}) to assign a unique integer index to each distinct vertex tuple. Specifically, Algo.~\ref{algo:index_optimal} collects the extracted coordinates $\mathcal{I}(\hat{v})$ (using the \textproc{Coord}$(v)$ function) for each vertex $v$ in the tuple $\hat{v}$, and uses the \textproc{UniqueEncode}$(\cdot)$ function to generate a unique hash encoding for a series of integer coordinates $\mathcal{I}(\hat{v})$.

We construct the nodes of $\mathcal{G}_H$ only on demand during the search process as the graph size may be prohibitively large. At the beginning of the search, the whole $\mathcal{G}_H$ is not explicitly constructed, and the visited set $\mathcal{L}$ and the open set $\mathcal{O}$ are both initialized to empty. After the insertion of the initial state $\pi_s$, a tree of nodes is gradually expanded based on the $\Call{FeasibleSuccs}{\cdot}$ function and the priority queue structure maintained by $\mathcal{O}$. Every time a new node is found (whether or not a successor node is a new node is identified by $\mathcal{L}$, which is implemented using a hash map structure), the node is added to $\mathcal{L}$ to trace its open/closed status. In practice, only a small proportion of the nodes of $\mathcal{G}_H$ are constructed before the search succeeds. A toy example\footnote{The setup of this example is the same as the qualitative experiment in Fig.~\ref{fig:simple3d_tuple}, where the grid size is $51\times51\times5$ and the B-spline degree is five.} which compares the number of actual expanded nodes with the estimated graph size is shown in Tab.~\ref{tab:ondemand_graph_construct}. Note that the EBK method used in Tab.~\ref{tab:ondemand_graph_construct} is an efficient version of Algo.~\ref{algo:optimal_search}, which is discussed in detail in Sect.~\ref{sec:modification_efficiency}. The estimated graph size is computed based on the graph aggregation technique introduced in Sect.~\ref{sec:bk_performance_analysis}.

\begin{table}[b]
	\centering
	\caption{Illustration of the on-demand graph construction.\label{tab:ondemand_graph_construct}}
	\begin{tabular}{@{}lcccc@{}}
	\toprule
	\textbf{\scriptsize{Method}}
	&\textbf{\makecell{\scriptsize{Run}\\ \scriptsize{Time} (s)}}
	&\textbf{\makecell{\scriptsize{Actual $\#$ of}\\ \scriptsize{Exp. Nodes}}}
	&\textbf{\makecell{\scriptsize{Total $\#$ of}\\ \scriptsize{Nodes (Esti.)}}}
	&\textbf{\makecell{\scriptsize{Percentage}\\($\%$)}} \\
	\midrule
	\scriptsize{EBK-D1}  & \scriptsize{0.034}   & \scriptsize{$2,334$}     & \scriptsize{$13,005$}       & \scriptsize{17.9} \\
	\scriptsize{EBK-D2}  & \scriptsize{0.383}   & \scriptsize{$36,818$}    & \scriptsize{$351,135$}      & \scriptsize{10.5} \\
	\scriptsize{EBK-D3}  & \scriptsize{7.422}   & \scriptsize{$460,469$}   & \scriptsize{$9,480,645$}    & \scriptsize{4.9} \\
	\scriptsize{EBK-D4}  & \scriptsize{173.710} & \scriptsize{$9,096,338$} & \scriptsize{$255,977,415$}  & \scriptsize{3.6} \\
	\bottomrule
	\end{tabular}
\end{table}

Note that there are three functions making Algo.~\ref{algo:optimal_search} different from the traditional geometric path search: 1) the cost function $f_{k,\Delta_t}(\cdot)$ evaluates the control effort of the k-degree vertex tuple, instead of a simple path length measure; 2) the function \textproc{Index}$(\cdot)$ regards the k-degree vertex tuple as the ``state'', which expands the high-dimensional state space supported by the control point coordinates, while the geometric path search typically regards positions as states; and 3) the function \textproc{FeasibleSuccs}$(\cdot)$ will expand to the neighboring k-degree tuples with dynamical feasibility checking, while the geometric path search cannot check feasibility without parameterization. As for the heuristic function $\Call{Heuristic}{\cdot}$, we adopt the admissible minimum time heuristic  in~\cite{liu2017smp}.

\begin{algorithm}[t]
	\begin{algorithmic}[1]
	\Function{Index}{$\hat{v}$}
	\State $\mathcal{I}(\hat{v})=\emptyset$;
	\For{\textbf{all} $ v \in \hat{v}$}
		\State $\mathcal{I}(\hat{v}) \leftarrow \mathcal{I}(\hat{v}) \cup \Call{Coord}{v}$
	\EndFor
	\State \textbf{return} $\Call{UniqueEncode}{\mathcal{I}(\hat{v})}$
	\EndFunction
	\end{algorithmic}
\caption{\label{algo:index_optimal} Given a k-degree vertex tuple $\hat{v}\in V_H$, assign a unique index to each distinct tuple.}
\end{algorithm}

It is worth noting that the design of Algo.~\ref{algo:optimal_search} heavily relies on the properties of the B-spline. Thanks to the local control property, the cost evaluation and feasibility checking can be done locally based on the k-degree vertex tuple. Instead of solving the BVP, the proposed search method expands to new states on the high-dimensional graph $\mathcal{G}_H$ by expanding low-dimensional control point coordinates, which are associated with closed forms for cost evaluation. Moreover, checking for collision is time consuming in traditional kinodynamic planners~\cite{webb2013kinodynamic,liu2017smp, allen2016real}. By using B-spline parameterization, the process can be avoided by characterizing the B-spline deviation, as introduced in Sect.~\ref{sec:feasibility_condition}. The limitation of our method is that the expansion of control points is restricted by the resolution and connection of the grid, which results in limited representations of B-spline trajectories.

\subsection{Feasibility Condition} \label{sec:feasibility_condition}
The dynamical feasibility of the k-degree vertex tuple can be validated via checking the extrema of the derivative of the B-spline. Since the derivative of the B-spline is another B-spline with decreasing degree, the velocity spline is of degree $k-1$ and the acceleration spline is of degree $k-2$. Considering that the convex hull property in Prop.~\ref{prop:linear_derivative_bound} is a sufficient but not necessary condition, directly using Prop.~\ref{prop:linear_derivative_bound} for feasibility checking may be conservative.
Actually, there is a non-conservative approach for feasibility checking by using the closed-form solutions of the extremas of the uniform B-spline. Take the fifth-degree uniform B-spline as an example. The B-spline can be rewritten in a monomial basis according to Eq.~\ref{eq:bspline_derivative}. The velocity profile is a degree-$4$ polynomial whose extremas can be checked by finding the roots of its derivative (degree-$3$) in closed form.

For traditional kinodynamic planners~\cite{liu2017smp,webb2013kinodynamic,allen2016real}, collision checking is an expensive process and may become the computation bottleneck of the algorithm~\cite{kleinbort2016collision}. Position-only shortest path search on the graph of the cell decompositions does not require collision checking since the piecewise linear connection between cell centers is restricted to collision-free cells. For the proposed method, given the B-spline parameterization of degree $k$ and cell size of the decomposed environment, the B-spline may deviate from the piecewise linear connection due to the fact that it does not exactly pass through the control points. However, the maximum distance that the B-spline curve deviates from the piecewise linear collection can be characterized offline, which is compensable by moderate obstacle inflation. The inflation needed is characterized in Appendix~\ref{sec:appendix_prop_collifree}. In practice, since the degree of the B-spline is fixed and the cell size is not tuned frequently, the inflation can be calculated once and then used for many experiments.

\begin{figure}[t]
	\begin{center}
		\subfigure[Mapping from the traces on $\mathcal{G}$ to $\mathcal{G}_H$ \label{fig:graph_construction}]{\includegraphics[trim={0cm 0cm 0cm 0cm},clip,width=0.42\textwidth]{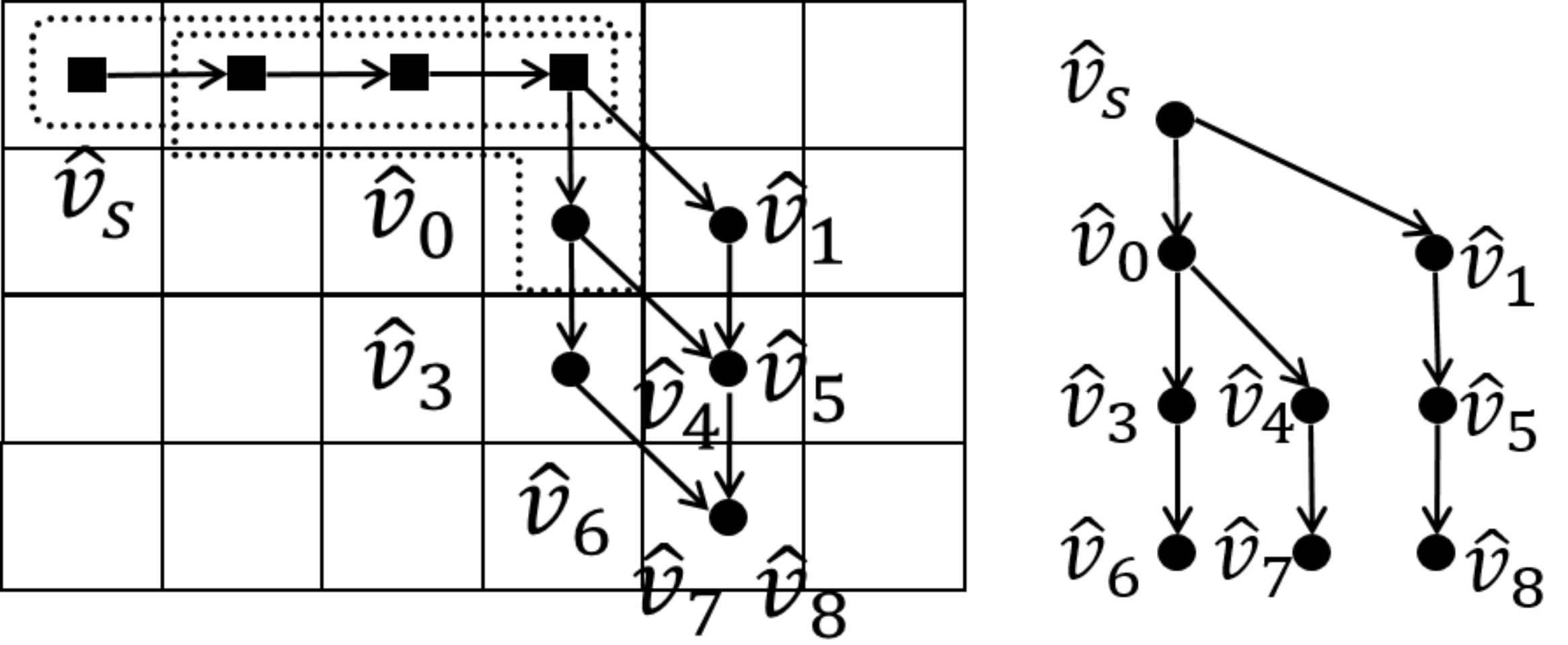}}
		\subfigure[Graph aggregation\label{fig:graph_aggregation}]{\includegraphics[trim={0cm 0cm 0cm 0cm},clip,width=0.42\textwidth]{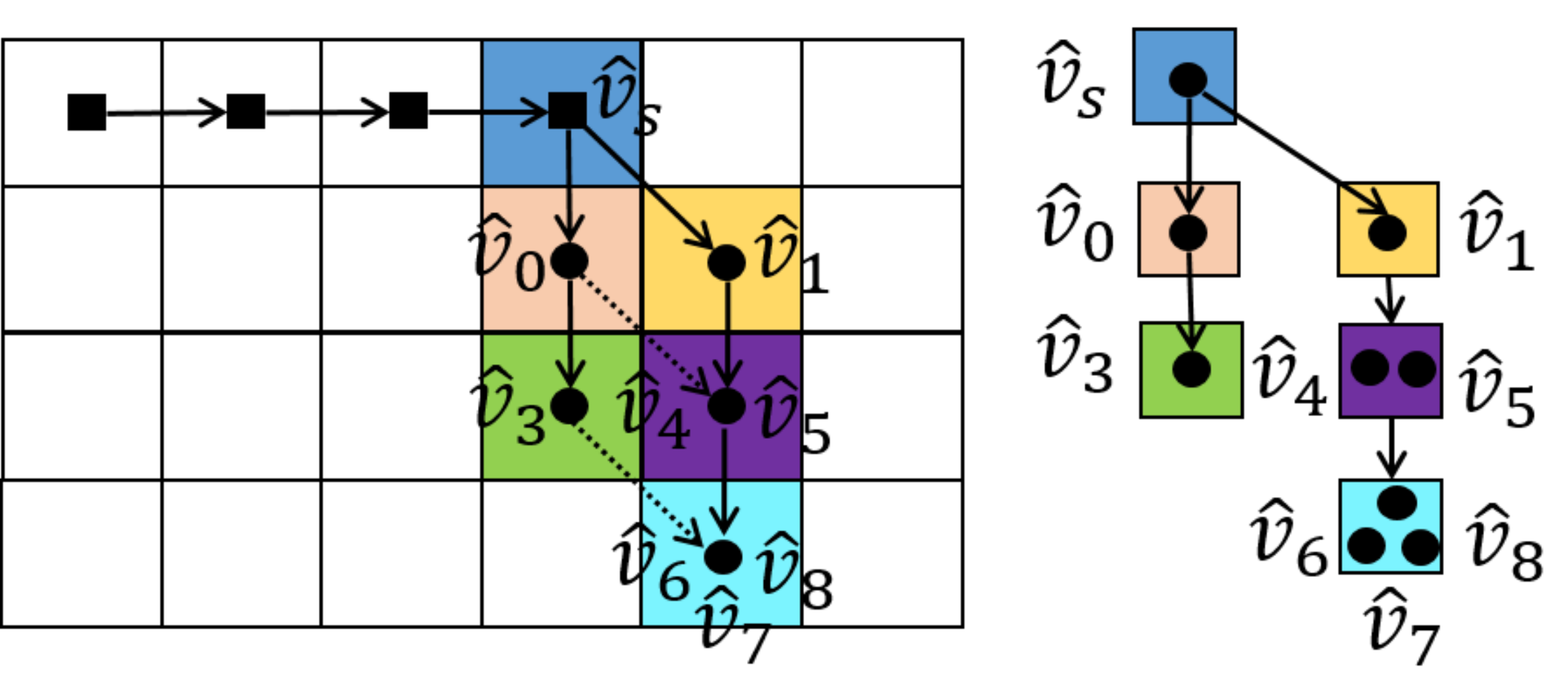}}
	\end{center}
	\caption{Illustration of the mapping process to $\mathcal{G}_H$ and the graph aggregation process used by the EBK search.}
	\vspace{-0.3cm}
\end{figure}

\subsection{Efficient Low Dispersion Search} \label{sec:modification_efficiency}
Before proposing the efficient methods, we present a complexity analysis of the OBK search. According to the definition of the k-degree vertex tuple, the total number of vertices $V_H$ of the graph $\mathcal{G}_H$ grows exponentially w.r.t. the degree $k$ of B-spline parameterization, i.e., $|V_H| = O(|V|^{k+1})$. Note that according to Prop.~\ref{prop:graph_transformation}, Algo.~\ref{algo:optimal_search} shares a similar complexity with the execution of Dijkstra's algorithm on the graph $\mathcal{G}_H$ if the heuristic is set to zero. According to the known result that each vertex is expanded at most once for Dijkstra's algorithm (see~\cite{bertsekas1995dynamic, cormen2009introduction}), the maximum number of iterations of Algo.~\ref{algo:optimal_search} is upper bounded by $|V_H|=O(|V|^{k+1})$, which characterizes the worst-case execution time of the OBK search. Actually, since $|V_H|$ and $|E_H|$ grow exponentially with $k$, the worst-case execution time of \textit{any} algorithm that optimally solves Problem~\ref{prob:optimal_search} scales exponentially with $k$.

Given the observation that the state in the OBK search is homogeneous (i.e., all the dimensions are control point coordinates), we can aggregate the nodes of $\mathcal{G}_H$ based on the proximity of the coordinates to gain efficiency. Specifically, according to Algo.~\ref{algo:index_optimal}, each vertex $\hat{v}\in V_H$ is marked with a unique integer index. In the modified \textproc{Index}$(\cdot)$ function in Algo.~\ref{algo:index_efficient}, we encode the vertex $\hat{v}$ only based on the coordinates for the last $d$ vertices such that the vertices which share the same partial coordinates will be regarded as the same node. By aggregating the nodes of $\mathcal{G}_H$, the dimension of the search space is directly controlled by the user-specified parameter $d$. The modification essentially conducts a low dispersion search on the high-dimensional graph $\mathcal{G}_H$ with a control on the number of expanded nodes.

Different $d$ values will determine how the vertex tuples are aggregated, which in turn affects the solution quality and algorithm efficiency. Note that although the search space is reduced, the continuity and smoothness of the resultant trajectory is maintained since the modification preserves the sharing of the B-spline coordinates. The resultant search method is called EBK search and a formal analysis of the EBK search is provided in Sect.~\ref{sec:bk_performance_analysis}.
\begin{algorithm}[t]
	\begin{algorithmic}[1]
	\Function{Index}{$\hat{v}, d$}
	\State $\mathcal{I}(\hat{v})=\emptyset$;
	\For{\textbf{all} $i \in \{k-d+1, \ldots, k\}$}
		\State $\mathcal{I}(\hat{v}) \leftarrow \mathcal{I}(\hat{v}) \cup \Call{Coord}{ \hat{v}[i]}$
	\EndFor
	\State \textbf{return} $\Call{UniqueEncode}{\mathcal{I}(\hat{v})}$
	\EndFunction
	\end{algorithmic}
\caption{\label{algo:index_efficient} Given a k-degree vertex tuple $\hat{v}\in V_H$, assign an integer index based on the selected coordinates of the tuple.}
\end{algorithm}

\subsection{Analysis of the EBK Search}\label{sec:bk_performance_analysis}
As introduced in Sect.~\ref{sec:modification_efficiency}, the user-specified parameter $d$ determines how vertex tuples in the graph $\mathcal{G}_H$ are aggregated. We provide a toy example of $d=1$ in Fig.~\ref{fig:graph_aggregation} to understand the graph aggregation.
When choosing $d=1$, as in Fig.~\ref{fig:graph_aggregation}, vertex tuples are aggregated based on the last vertex of the tuple. For instance, $\hat{v}_4$ and $\hat{v}_5$ share the same last vertex and are aggregated into the same node, as marked in \textit{purple}. In the same way, the three distinct nodes $\hat{v}_6$, $\hat{v}_7$ and $\hat{v}_8$ are aggregated into the same \textit{cyan} node.

The edges of $\mathcal{G}_H$ are also reduced accordingly. For example, the edges $(\hat{v}_4, \hat{v}_7)$ and $(\hat{v}_5, \hat{v}_8)$ are aggregated since they are connecting the same two aggregated nodes. Note that once a path to the aggregated node is determined, such as the path \textit{blue}-\textit{yellow}-\textit{purple}-\textit{cyan}, the vertex tuple associated with each aggregated node is determined. In this example, the \textit{purple} node is associated with $\hat{v}_5$ and the \textit{cyan} node is associated with $\hat{v}_8$. Therefore, given the initial vertex tuple, any path on the aggregated graph, can be uniquely transformed to a path on the graph $\mathcal{G}_H$. And, apparently, a path on the graph $\mathcal{G}_H$ can be transformed to a path of aggregated nodes. The equivalence states that we are actually conducting a graph search on the aggregated graph with a controllable number of vertices, i.e., a low-dispersion search on the original high dimensional graph $\mathcal{G}_H$. The resultant path on the aggregated graph can be reconstructed as an admissible path on $\mathcal{G}_H$.

For the simple case of $d=1$, as illustrated in Fig.~\ref{fig:graph_aggregation}, the size of the aggregated graph is the same as the original graph $\mathcal{G}$. Therefore, the EBK search can be as efficient as a shortest path search on the spatial grid by choosing a small $d$. And the advantage of the EBK search is that it directly outputs a time-parameterized dynamically feasible trajectory. It turns out that the EBK search is resolution complete with respect to the aggregated graph, and we refer interested readers to a detailed analysis of the EBK search in Appendix~\ref{sec:appendix_resolution_complete}.

By choosing a small $d<k+1$, a large number of vertex tuples are aggregated into one group, and the ``resolution'' of the graph becomes large. Due to the aggregation, the representation of the trajectory is limited. An intuitive example is that, when choosing $d=1$, the search process will never choose to place the same control point in the same grid cell due to the strictly positive cost. As a result, the trajectory obtained may fail to reach the exact end state, such as a static state. However, the issue can be addressed by choosing a larger $d$ and sacrificing efficiency.

Compared to the preliminary version, i.e., the RBK search in~\cite{ding18replanning}, the EBK search is more flexible and allows for control of the algorithm efficiency and solution quality. The connection is that the RBK search is essentially the EBK search using $d=1$.
\section{Elastic Optimization} \label{sec:elastic_optimization}
To compensate for the discretization introduced by the EBK search and further improve the trajectory quality, we present the EO approach, which refines the control point placement to the optimal location w.r.t. the free space. Our approach is motivated by the seminal work~\cite{quinlan1993elastic}, in which a collision-free ``tube'' around the initial path is identified and the path is ``stretched'' within the tube so that the shape is optimized.
Mathematically, the tube is defined as a series of balls, with the ball centers denoted as $\mathcal{P} \coloneqq \left\lbrace \point{p}_0, \point{p}_1, \ldots, \point{p}_T \right\rbrace$ and corresponding radiuses denoted as $\mathcal{R} \coloneqq \left\lbrace r_0, r_1, \ldots, r_T \right\rbrace$, where $\mathbf{p}_i$ denotes the ball center and $r_i$ denotes the radius. The tube is defined to be ``well-connected'' if and only if $\norm{\point{p}_i - \point{p}_{i+i}} \le r_i + r_{i+1}, \, \forall i\in   \{0,\ldots,T-1\} $. Compared to~\cite{quinlan1993elastic} which cannot handle dynamical feasibility constraints for complex kinodynamic systems such as quadrotors, we propose a convex optimization formulation based on B-spline parameterization, which uses the convex hull property to enforce feasibility.

Note that Zhu~\textit{et al.}~\cite{zhu2015ces} also propose a convex elastic smoothing formulation for car-like robots. There are two major differences: 1) The formulation in~\cite{zhu2015ces} is based on the dynamics of car-like robots and cannot be applied to complex dynamic systems, while our formulation uses high-order B-spline parameterization, which can be directly used to control quadrotors. 2) In~\cite{zhu2015ces}, the smoothed trajectory may collide with obstacles due to the geometric incompleteness of the tube constraint (as shown in Fig.~\ref{fig:eo_inflation}) and only a heuristic waypoint insertion/obstacle inflation scheme is provided, while the EO approach has a theoretical safety guarantee, which is achieved by a two-level inflation scheme to ensure the connectivity of the tube and a finite iterative control point insertion process.

\subsection{Elastic Tube Expansion}\label{sec:elastic_expansion}
In~\cite{quinlan1993elastic} and~\cite{zhu2015ces}, the elastic tube is a series of connected balls which are centered at the waypoints of the reference path. Intuitively, the tube generated in this way cannot fully utilize the free space around it, as shown in Fig.~\ref{fig:tube_expansion}. We therefore propose a lightweight tube expansion algorithm so that the tube can roughly represent the locally largest free space. Given the initial control point placement $\pi=(v_0, \ldots, v_T)$ provided by Algo.~\ref{algo:optimal_search}, we first extract the coordinates of $\pi$, and denote the collection of coordinates as $\mathcal{P} \coloneqq \left\lbrace \point{p}_0, \point{p}_1, \ldots, \point{p}_T \right\rbrace$ following the notation used in Sect.~\ref{sec:Bspline_property}.

\begin{algorithm}[t]
	\begin{algorithmic}[1]
	\Function{TubeExpansion}{$\mathcal{P}, \mathcal{C}^{\text{ELAS}}$}
	\State Initializes: $d_\text{infl}^{\min}, d_\text{infl}^{\max}, d_{\text{thres}}.\, \set{R} = \set{R^{\prime}} = \set{Q} = \emptyset$;
	\For{\textbf{all} $ \mathbf{p}_i \in \mathcal{P}$}
		\State $(\point{n}_i,r_i)\leftarrow \Call{NNSearch}{\point{p}_i,\set{C}^{\text{ELAS}}}$;
		\State $\overrightarrow{n} = (\point{p}_i - \point{n}_i)/\Call{norm}{\point{p}_i - \point{n}_i}$;
		\While{$d_\text{infl}^{\max} - d_\text{infl}^{\min} > d_\text{infl}^{\text{tol}}$}
			\State $d \leftarrow \left( d_\text{infl}^{\max} + d_\text{infl}^{\min} \right) /2 $, $\point{p}_{i,\text{infl}}  \leftarrow \point{p}_i + d \cdot \overrightarrow{n} $;
			\State $(\point{n}_i^{\prime},r_i^{\prime})\leftarrow \Call{NNSearch}{\point{p}_{i,\text{infl}},\set{C}^{\text{ELAS}}}$;
			\If{$\Call{Abs}{r_i^{\prime} - d - r_i} > d_{\text{thres}} $}
				\State $d_\text{infl}^{\max}  \leftarrow d $
			\Else
				\State $d_\text{infl}^{\min}  \leftarrow d $
			\EndIf
		\EndWhile
		\State $\mathcal{Q} \leftarrow \mathcal{Q} \cup \point{p}_{i,\text{infl}}$, $\set{R}^{\prime} \leftarrow \set{R}^{\prime} \cup r_i^{\prime}$
	\EndFor
	\State \textbf{return} $\mathcal{Q}, \mathcal{R}^{\prime}$
	\EndFunction
	\end{algorithmic}
\caption{\label{algo:tube_expansion}Expand elastic tube in configuration space $\mathcal{C}^{\text{ELAS}}$.}
\end{algorithm}

The elastic tube expansion algorithm (Algo.~\ref{algo:tube_expansion}) can be divided into two steps: First, we construct the initial tube, by conducting a radius search for the initial placement $\set{P}$, and obtain the nearest obstacle position $\mathbf{n}_i$. Second, we push the center of the bubbles in the direction $\overrightarrow{n}$ (away from the nearest obstacle) while satisfying the criterion that the new bubble contains the original bubble, as required by condition \textproc{Abs}$(r_i^{\prime} - d - r_i )\leq d_{\text{thres}}$, as shown in Fig.~\ref{fig:tube_expansion}. The inflation process is implemented in a binary search manner. Algo.~\ref{algo:tube_expansion} will finally find a series of local maximum volume bubble centers $\set{Q} \coloneqq \left\lbrace \point{q}_0, \point{q}_1,\ldots,\point{q}_T \right\rbrace$ based on the initial tube $\set{P}$. For the parameter settings, $d_\text{infl}^{\max}$ and $ d_\text{infl}^{\min}$ are the maximum and minimum inflation distance, respectively; $d_{\text{thres}}$ is the threshold for checking whether the new bubble contains the original one, and should be set to a small value, e.g., less than the map resolution; and $d_\text{infl}^{\text{tol}}$ is the binary search end condition, which can be set to the resolution of the map. The function \textproc{NNSearch} is the nearest neighborhood search, which can be done efficiently if a KD-tree is maintained. The efficiency of Algo.~\ref{algo:tube_expansion} is verified in Section.~\ref{sec:analysis}.

\begin{figure}[t]
	\begin{center}
		\subfigure[\label{fig:tube_expansion}]{\includegraphics[trim={0cm 0cm 0cm 0cm},clip,width=0.233\textwidth]{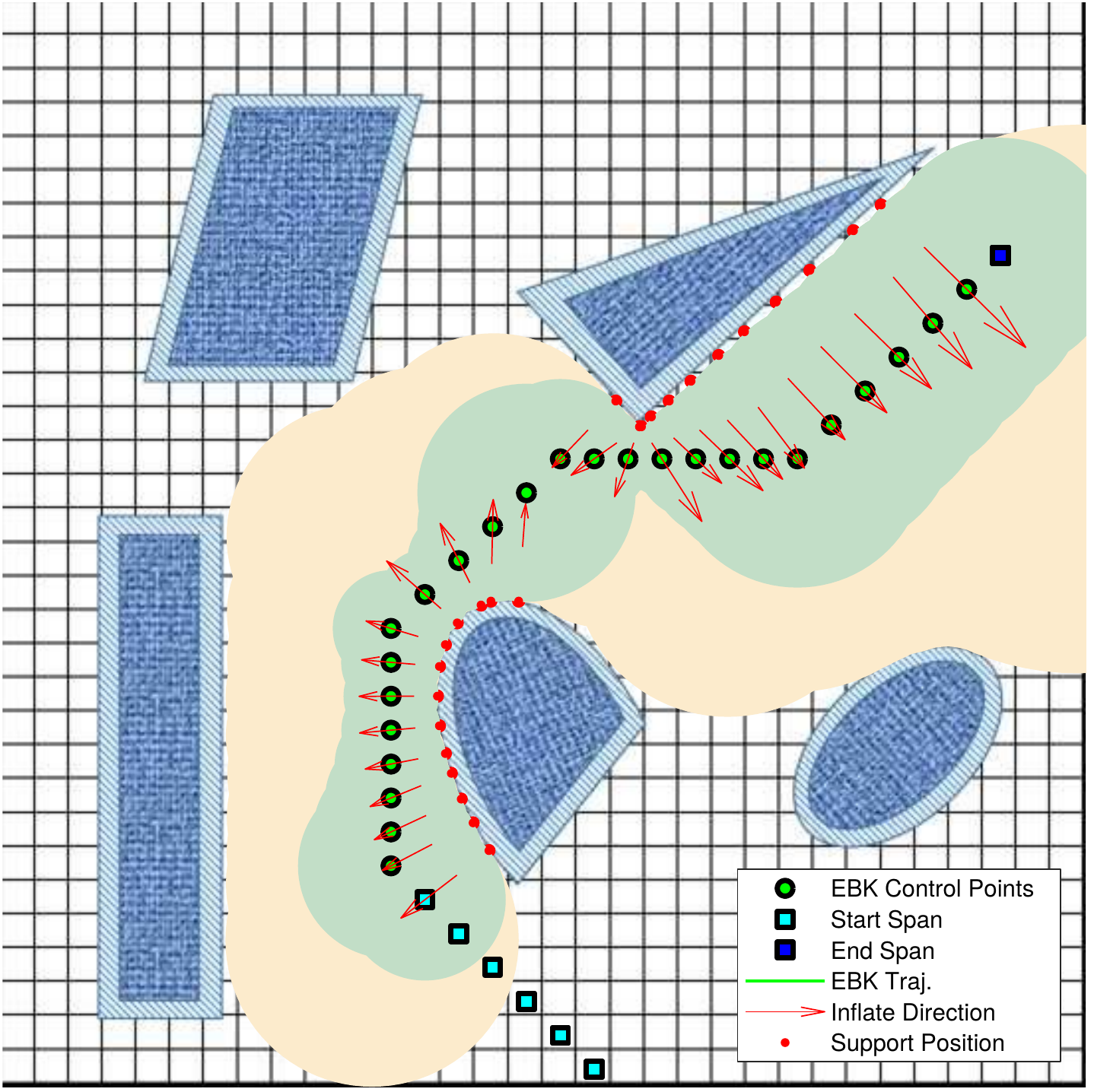}}
		\subfigure[\label{fig:elastic_optimization}]{\includegraphics[trim={0cm 0cm 0cm 0cm},clip,width=0.233\textwidth]{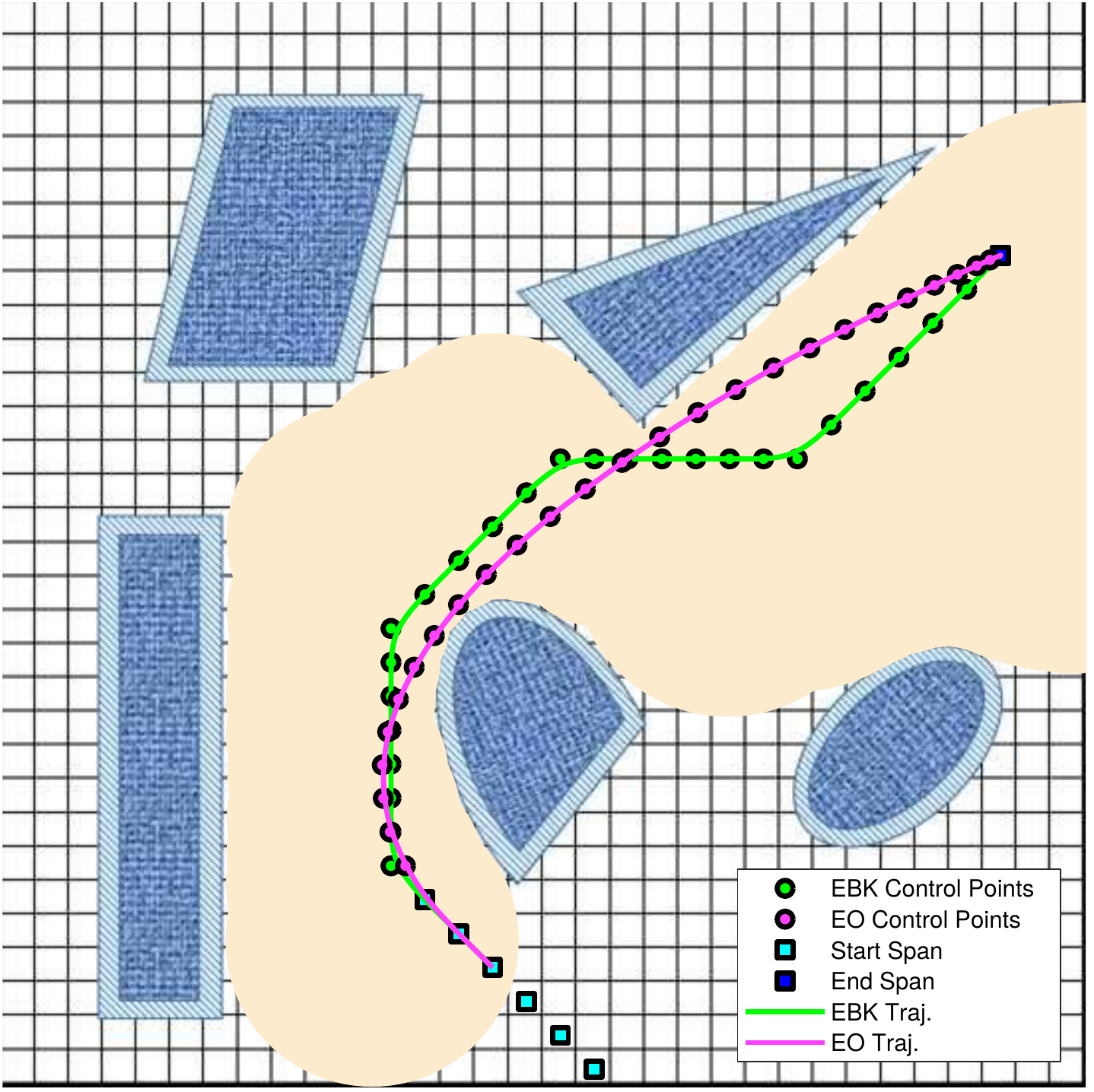}}
	\end{center}
	\caption{Illustration of the EO approach: (a) shows the elastic tube expansion process (Sect.~\ref{sec:elastic_expansion}); (b) shows the optimization process (Sect.~\ref{sec:elastic_opt_formulation}). In (a), the original tube is marked in~\textit{green}, while the inflated tube is marked in~\textit{yellow}. \label{fig:eo_pipeline}}
	\vspace{-0.5cm}
\end{figure}

\subsection{Elastic Optimization Formulation}
\label{sec:elastic_opt_formulation}
Given the inflated ball centers $\mathcal{Q}=\left\lbrace \point{q}_0, \point{q}_1, \ldots, \point{q}_T \right\rbrace$ and corresponding radiuses $\mathcal{R}^{\prime}=\left\lbrace r^{\prime}_0, r^{\prime}_1, \ldots r^{\prime}_T \right\rbrace $, the EO formulation minimizes the total control effort by finding the optimal placement $\mathcal{P}^{\ast}=\left\lbrace \point{p}^{\ast}_{0}, \point{p}^{\ast}_{1}, \ldots, \point{p}^{\ast}_{T} \right\rbrace$ while satisfying the safety and dynamical feasibility constraints. The safety constraints are enforced by constraining the control point position inside the 3-D balls, and in Sect.~\ref{sec:safety_guarantee}, we will discuss how to theoretically guarantee the safety of the resultant trajectory. The dynamical feasibility constraints are enforced using Prop.~\ref{prop:linear_derivative_bound}.

Note that like the EBK search, we consider the initial and goal state of the quadrotor. Denote the coordinates of $\pi_s$ and $\pi_g$ as $\mathcal{P}^{s}= \left\lbrace \point{p}^{0}_s, \point{p}^{1}_s, \ldots, \point{p}^{k}_s \right\rbrace$ and $\mathcal{P}^{g}= \left\lbrace \point{p}^{0}_g, \point{p}^{1}_g, \ldots, \point{p}^{k}_g \right\rbrace$, respectively. These coordinates are fixed during optimization. Similar to the construction of the k-degree vertex tuple in Sect.~\ref{sec:problem_formulation}, we concatenate $\mathcal{P}^s$, $\mathcal{P} $ and $\mathcal{P}^g$ and formalize the concatenation in terms of the control point spans. Following the similar notation in Sect.~\ref{sec:problem_formulation}, we denote by ${[\tilde{\mathcal{P}}]}^k_j$ the $j$-th control point span of the concatenated coordinates $\tilde{\mathcal{P}}$. It follows that ${[\tilde{\mathcal{P}}]}^k_0 = \mathcal{P}^s$ and ${[\tilde{\mathcal{P}}]}^k_J = \mathcal{P}^g$, where $J=k+T+2$. We denote by $\mathbf{P}_j$ the stacked coordinates matrix of ${[\tilde{\mathcal{P}}]}^k_j$, with $\mathbf{P}^D_j$ denoting the $D\in \{x,y,z\}$ axis. The optimization problem can be expressed as follows:
\begin{equation}
	\begin{aligned}
		& \text{min} & & \sum_{j=0}^{J} f_{k,\Delta_t}({[\tilde{\mathcal{P}}]}^k_j)  \\
		& \text{s.t} 						   & & {[\tilde{\mathcal{P}}]}^k_0 = \mathcal{P}^s, {[\tilde{\mathcal{P}}]}^k_J = \mathcal{P}^g\\
		&								       & & \left\Vert \point{p}_i - \point{q}_i \right\Vert_2 \leq r_i^{\prime}, & \forall i \in \left\lbrace 0,\ldots, T\right\rbrace\\
		&									   & &|\mathbf{S}\mathbf{P}_j^{D}| \leq u_{l,D}^{\text{max}}\bm{1}_{(k+1)\times 1}, &\forall j, D\in\{x,y,z\},
	\end{aligned}
\end{equation}
where the first constraint expresses that the initial and goal states need to be fixed. And the second constraint (quadratic) restricts the control points inside the expanded tube. The third constraint (linear) is to ensure the dynamical feasibility using the sufficient condition in Prop.~\ref{prop:linear_derivative_bound}. Since the control points can be adjusted in continuous free space, the potential conservativeness brought by Prop.~\ref{prop:linear_derivative_bound} is minor in practice.
The overall formulation is a QCQP which can be solved efficiently using off-the-shelf convex solvers. The time cost is fixed given the initial placement $\mathcal{P}$, so it is not included in the objective.

\subsection{Enforcing Safety Guarantee}
\label{sec:safety_guarantee}
\begin{figure}[t]
	\begin{center}
		\subfigure[\label{fig:eo_inflation}]{\includegraphics[trim={0cm 7cm 0cm 0cm},clip,width=0.233\textwidth]{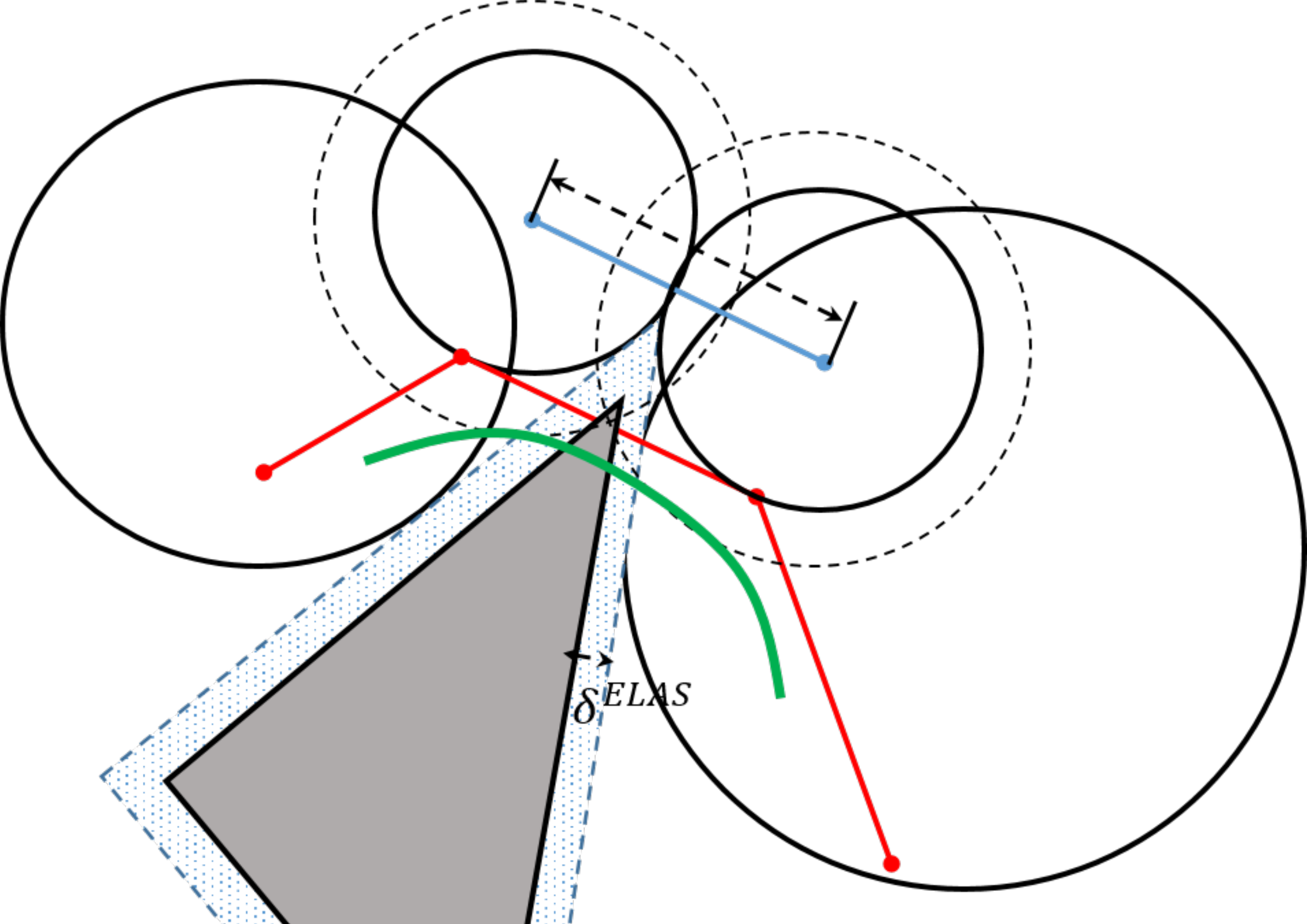}}
		\subfigure[\label{fig:twolevel_inflation}]{\includegraphics[trim={0cm 0cm 0cm 0cm},clip,width=0.233\textwidth]{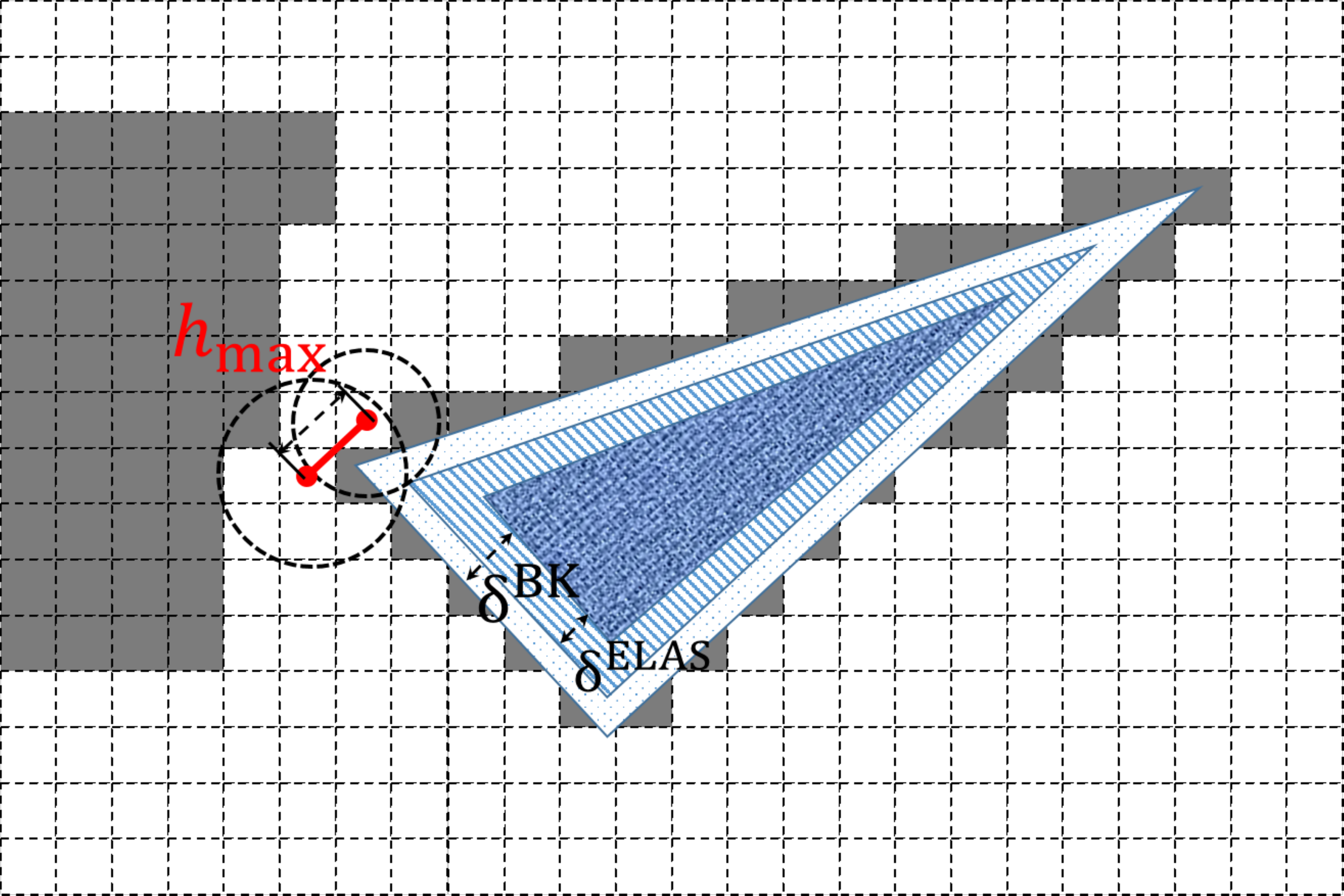}}
		\subfigure[\label{fig:iterative_insertion}]{\includegraphics[trim={0cm 0cm 0cm 0cm},clip,width=0.43\textwidth]{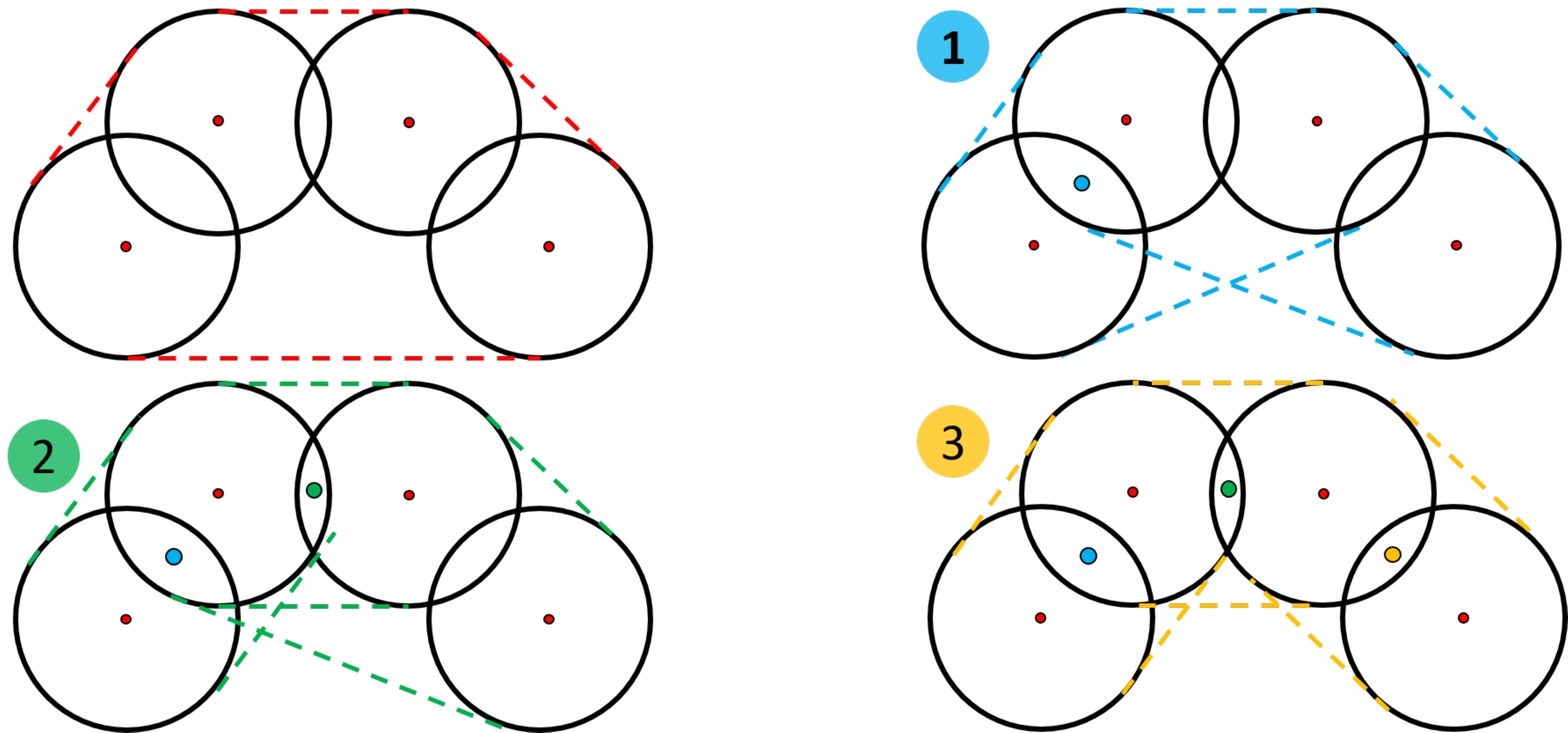}}
	\end{center}
	\caption{Illustration of (a) the geometric incompleteness of the ball constraint, (b) the two-level inflation scheme, and (c) the iterative convex hull shrinking process for enforcing the safety guarantee. It can be observed in (a) that even the straight-line segments between the balls may collide with the obstacle.}
	\vspace{-0.7cm}
\end{figure}
Restricting control points inside the balls is not a sufficient condition for safety, for the following two reasons: 1) the balls are placed with a finite density and constraining the control points in the balls cannot preserve the collision-free characteristic even for straight-line segments between control points (Fig.~\ref{fig:eo_inflation}), and 2) the B-spline does not exactly pass the control points and deviates from the straight-line segments. The first issue is also observed in~\cite{zhu2015ces}, which proposes using waypoint insertion/obstacle inflation to handle the problem. However, there is no quantification of how much inflation or how many insertions are needed and only a heuristic is given. The second issue is inherently similar to one common issue faced by piecewise polynomial parameterization~\cite{richter2016polyunqp, chen2016online, gao2016online, liu2017sfc}: the polynomial may deviate from the collision-free straight-line segments between waypoints or exceed the safe flight corridor. Chen \textit{et al.}~\cite{chen2016online} propose a finite iterative process by adding constraints on polynomial extremas based on a cube corridor (linear constraints), but this is not directly applicable to B-spline parameterization with quadratic constraints.

In our case, the issues can be resolved using a \textit{two-level inflation scheme} and an iterative \textit{convex hull shrinking process}. The two-level inflation scheme is to ensure the connectivity of the elastic tube, and furthermore, the inflation needed is quantified in the scheme. For simplicity, we first consider the original tube before applying the tube expansion algorithm (Algo.~\ref{algo:tube_expansion}). The two-level obstacle inflation scheme is as follows: the configuration space in which we conduct kinodynamic search with larger obstacle inflation $\delta^{\text{BK}}$ is $\mathcal{C}^{\text{BK}}$, while we generate the elastic tube and optimize the control points in the configuration space $\mathcal{C}^{\text{ELAS}}$ with smaller obstacle inflation $\delta^{\text{ELAS}}$, as shown in Fig.~\ref{fig:twolevel_inflation}. The difference adds additional clearance to any point in the configuration space $\mathcal{C}^{\text{BK}}$. Without the difference, the minimum radius of the ball at the grid center is $\min(c_x, c_y, c_z)$, where $c_x$, $c_y$ and $c_z$ are the grid size. Denote the maximum separation distance of two neighboring control points in the 3-D grid as $h_{\max}$. It follows that, if the difference $\delta^{\text{BK}}-\delta^{\text{ELAS}} > h_{\text{max}}/2 -\min(c_x, c_y, c_z)$ holds, the additional clearance will ensure that the two neighboring balls overlap, thus ensuring the connectivity of the elastic tube. Note that Algo.~\ref{algo:tube_expansion} maintains the connectivity of the tube by ensuring that the inflated ball contains the original ball while keeping the same support point on the obstacle. For the low-level inflation, given $h_{\max}$, $\delta^{\text{ELAS}} \geq \frac{\sqrt{2}-1}{2} h_{\max}$ is sufficient for the straight-line segments to be contained in the free space~\cite{ding18replanning}.

Based on the two-level inflation scheme, we propose an iterative \textit{convex hull shrinking} process, which pulls the B-spline trajectory back to the free space in the case of collision. The idea of the process is to iteratively add control points to the original B-spline control point sequence. The newly added control points are constrained in the intersection of two consecutive balls.
\footnote{According to the introduction in Sect.~\ref{sec:Bspline_property}, the total number of knots should satisfy $m+1 = (n+1)+k+1$. Since uniform B-spline is used, when a new control point is inserted, the number of knots is increased by one, while the knot separation $\Delta_t$ remains the same. The new control point sequence still matches the new knot vector.}
The process is built upon the convex hull property of the B-spline and constrains the B-spline trajectory by shrinking the convex hull.
We highlight that only a \textit{finite} number of control points are needed to enforce the safety of the B-spline trajectory. This could save a significant amount of computation power compared to the methods which apply dense constraint points based on a conservative heuristic~\cite{zhu2015ces, mellinger2011minsnap}. We conclude this feature in Theorem~\ref{thm:finite_safety}.
\begin{theorem}\label{thm:finite_safety}
Given a well-connected elastic tube, a B-spline trajectory that fits within the tube can be generated by iteratively adding constrained control points to the original B-spline control point sequence. The newly added control point is constrained inside the intersection of neighboring balls. The iterative process succeeds in a finite number of iterations, or infeasibility is reported when the dynamical feasibility cannot be satisfied, given the current tube and control point sequence.
\end{theorem}
\begin{proof}
Please refer to Appendix~\ref{sec:appendix_thm_safety} for the detailed proof.
\end{proof}

In Fig.~\ref{fig:iterative_insertion}, we provide a toy example of the convex hull shrinking process. The initial placement is constrained in four balls, and the convex hull envelope exceeds the free space. If collision is detected, we add one additional control point (\textit{blue dot}) which is constrained to be inside the intersection of the first and second ball. The extended EO is expected to be executed again, and the convex hull shrinks. Similarly, if the collision is still not resolved, we iteratively add control points to the intersection space. It can be shown that at most nine iterations are needed to resolve the collision in this case.

\section{Implementation Details}\label{sec:implementation_details}
\subsection{Receding Horizon Replanner Using Local Control}
As shown in Fig.~\ref{fig:bspline_replan}, B-spline has the local control property which facilitates the receding horizon (re)planning. Specifically, all the control points are organized in a sliding window. The control points corresponding to the executed and executing trajectory are committed and fixed. The disturbance caused by the optimization will not affect the feasibility of the executing trajectory due to the local control. A stopping policy will be activated if no feasible solution is found before the end of the executing trajectory.

When replanning is activated, the EBK search is called to update the placement for the control points under optimization. Note that the initial and goal state of the EBK search can be determined by the control points inside the window according to the local planning range. Note that the control points from the sliding window are in the continuous space after reshaping. However, the EBK search should use the discretized control points as the reference initial/ goal state to preserve optimality. The strategy of getting the reference states for the EBK search are to find the closest span pattern in terms of position and velocity error while matching the last control point to the grid cell. We constantly gather a fixed number of control points (e.g., twelve) for EO as the window moves forward.

There are two modes for the activation of replanning, namely, active mode and passive mode. For the passive mode, the EBK search is only called when collision is detected, while for the active mode, the EBK search is constantly activated as the sliding window moves forwards. Since the active mode can constantly improve the trajectory quality, it is more robust when the mapping quality is limited, but it is more computationally expensive.

\subsection{Monocular Vision-Based Testbed}
The monocular quadrotor testbed is equipped with a monocular camera (30 Hz), one IMU (100 Hz), an Intel i7 processor and an NVIDIA Jetson TX1 (Fig.~\ref{fig:drone}). The localization, mapping and planning modules are all running onboard. The localization module is based on our Monocular Visual Inertial Navigation System (VINS-Mono)~\cite{qin2017vins}, and the mapping module is based on our monocular dense mapping method~\cite{wang18quadtree} and truncated signed distance field (TSDF) fusion. No prior knowledge of the environment is required.

\subsection{Dual-Fisheye Vision-Based Testbed}
The dual-fisheye quadrotor testbed is equipped with two fisheye cameras (30 Hz), one IMU (100 Hz), an Intel i7 processor and an NVIDIA Jetson TX2 (Fig.~\ref{fig:fisheye_drone}). All modules are running onboard. It is worth noting that by using the two fisheye cameras, the system can provide omnidirectional perception and the quadrotor is able to fly a round-trip with a fixed yaw angle. The mapping module is based on our dual-fisheye omnidirectional stereo system~\cite{gao2017dual}. Also no prior knowledge of the environment is required.

\section{Analysis}\label{sec:analysis}
In this section, we present an analysis of the proposed kinodynamic planning framework. We begin with two individual analyses for the EBK search and the EO approach, respectively. To analyze the EBK search, we compare the proposed method with two kinodynamic planning algorithms, namely, search-based motion primitive (SMP)~\cite{liu2017smp} and kinodynamic RRT* (kRRT*)~\cite{webb2013kinodynamic}, representing both the search-based method and sampling-based method, respectively. For the EO, we compare the EO approach with two state-of-the-art trajectory optimization techniques, namely, continuous trajectory (CT) optimization~\cite{oleynikova2016ct} and gradient-based safe (GS) trajectory optimization~\cite{feigao2017hg}, which are popular non-linear optimization techniques for trajectory refinement. After the individual tests, we analyze the run-time efficiency of the proposed replanning framework, and compare the whole replanning system with the SMP~\cite{liu2017smp} method. We run all the simulations on a desktop computer equipped with an Intel I7-8700K CPU.

\subsection{Analysis of the B-spline-Based Kinodynamic Search} \label{sec:analysis_bk}

Recall that the motivation for introducing kinodynamic search to replanning is to facilitate dealing with the non-static initial states of the quadrotors. To this end, we analyze the performance of the kinodynamic search by planning from a given non-static initial state to varying goal states in a $10 \times 10 \times 2$ m test field. We compare our results with SMP~\cite{liu2017smp} and kinodynamic RRT*~\cite{webb2013kinodynamic} in terms of the trajectory quality and time efficiency, under the same planning setup. At the same time, we also illustrate the results of a hierarchical geometric planner as a common baseline. Specifically, the geometric planner first uses A* under the Euclidean distance measure to find the shortest path, and then parameterizes the path using the unconstrained QP formulation introduced in~\cite{richter2016polyunqp}.

The kinodynamic planners from~\cite{webb2013kinodynamic},~\cite{liu2017smp} and our method all have a tuning parameter to control the algorithm complexity and solution quality. For instance, SMP~\cite{liu2017smp} can control the discretization resolution for the control input. To further investigate the optimality-efficiency tradeoff achieved by each algorithm, we also demonstrate the results for these algorithms under different parameter setups. Note that we focus on the real-time replanning scenario, so we are concerned with the operation region where the run-time efficiency is close to real-time. Specifically, we compare the following methods:
\begin{itemize}[leftmargin=*]
\item \textit{$\text{Geometric}^{\dagger}$}: A path finder using A* and a fifth-degree polynomial parameterization using the unconstrained QP formulation~\cite{richter2016polyunqp} by minimizing the average integral of acceleration.
\item \textit{$\text{kRRT*-T100}^{\dagger}$}: The kinodynamic RRT* planner~\cite{webb2013kinodynamic} using a 3-D acceleration-controlled double integrator system under a 100 ms termination condition for the sampling.
\item \textit{kRRT*-T600}: The kinodynamic RRT* planner under a 600 ms termination condition.
\item \textit{$\text{SMP-U3}^{\dagger}$}: The SMP planner~\cite{liu2017smp} using a 3-D acceleration-controlled double integrator system with three discrete control inputs for each axis.
\item \textit{SMP-U5}: The SMP planner with five discrete control inputs for each axis.
\item \textit{$\text{EBK-D1}^{\dagger}$}: Our EBK planner using fifth-degree B-spline parameterization and $d=1$.
\item \textit{EBK-D2}: The EBK planner using $d=2$.
\end{itemize}

The reason for using the acceleration-controlled double integrator system for kinodynamic RRT* and SMP is that if a higher-order system were used, the run-time would be too large, making it inapplicable to replanning. For a similar reason, a termination time larger than $600$ ms for kRRT* and number of control inputs larger than five for SMP are not considered. The methods marked with $\dagger$ are amenable to real-time, and the other methods serve as showcases illustrating how the trajectory quality can be improved given a larger computation time budget.
\begin{figure}[t]
	\begin{center}
		\subfigure[Geometric method\label{fig:simple3d_geo}]{\includegraphics[width=0.49\columnwidth]{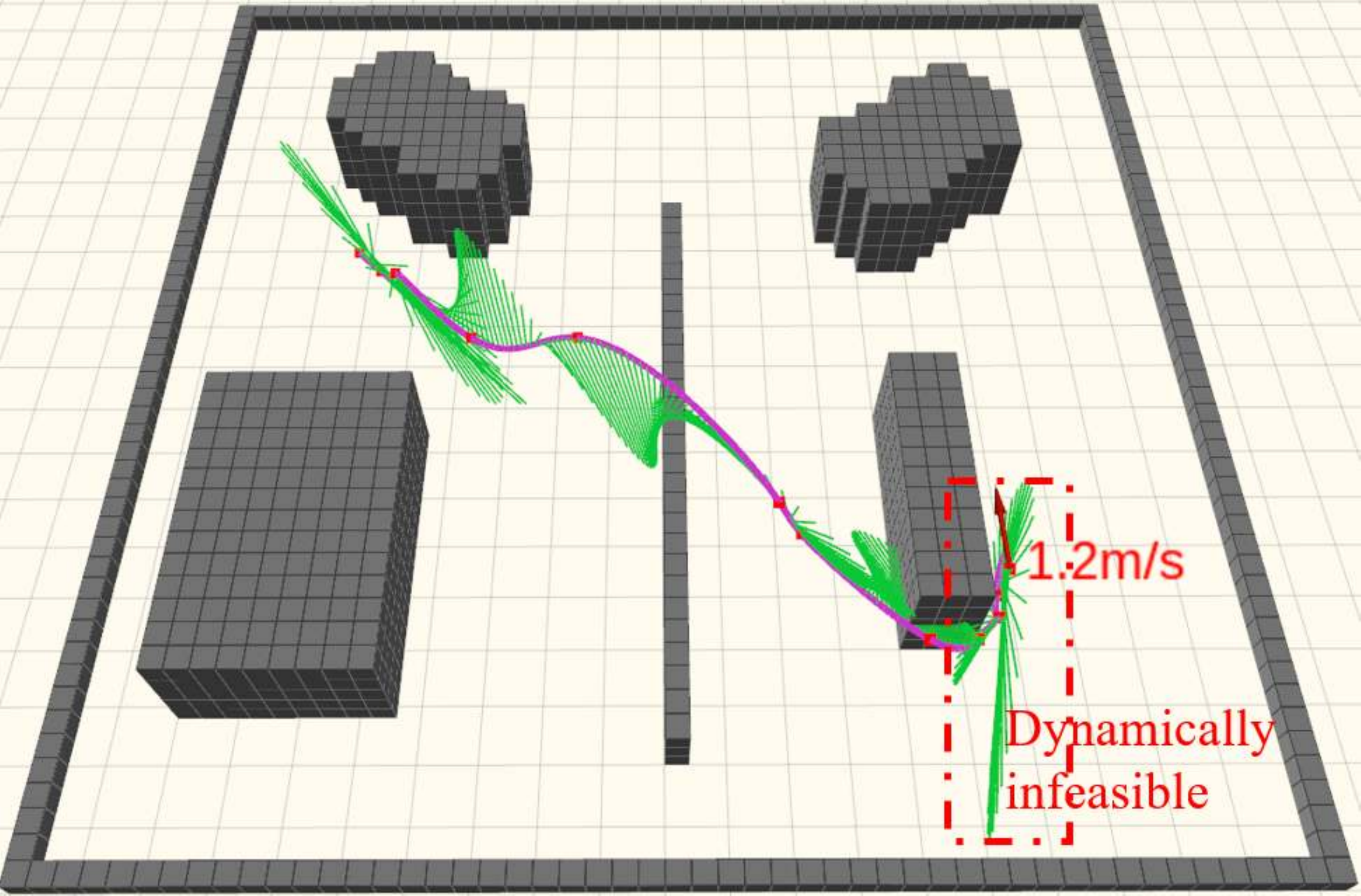}}
		\subfigure[Kinodynamic RRT*-T100\label{fig:simple3d_krrt100} ]{\includegraphics[width=0.49\columnwidth]{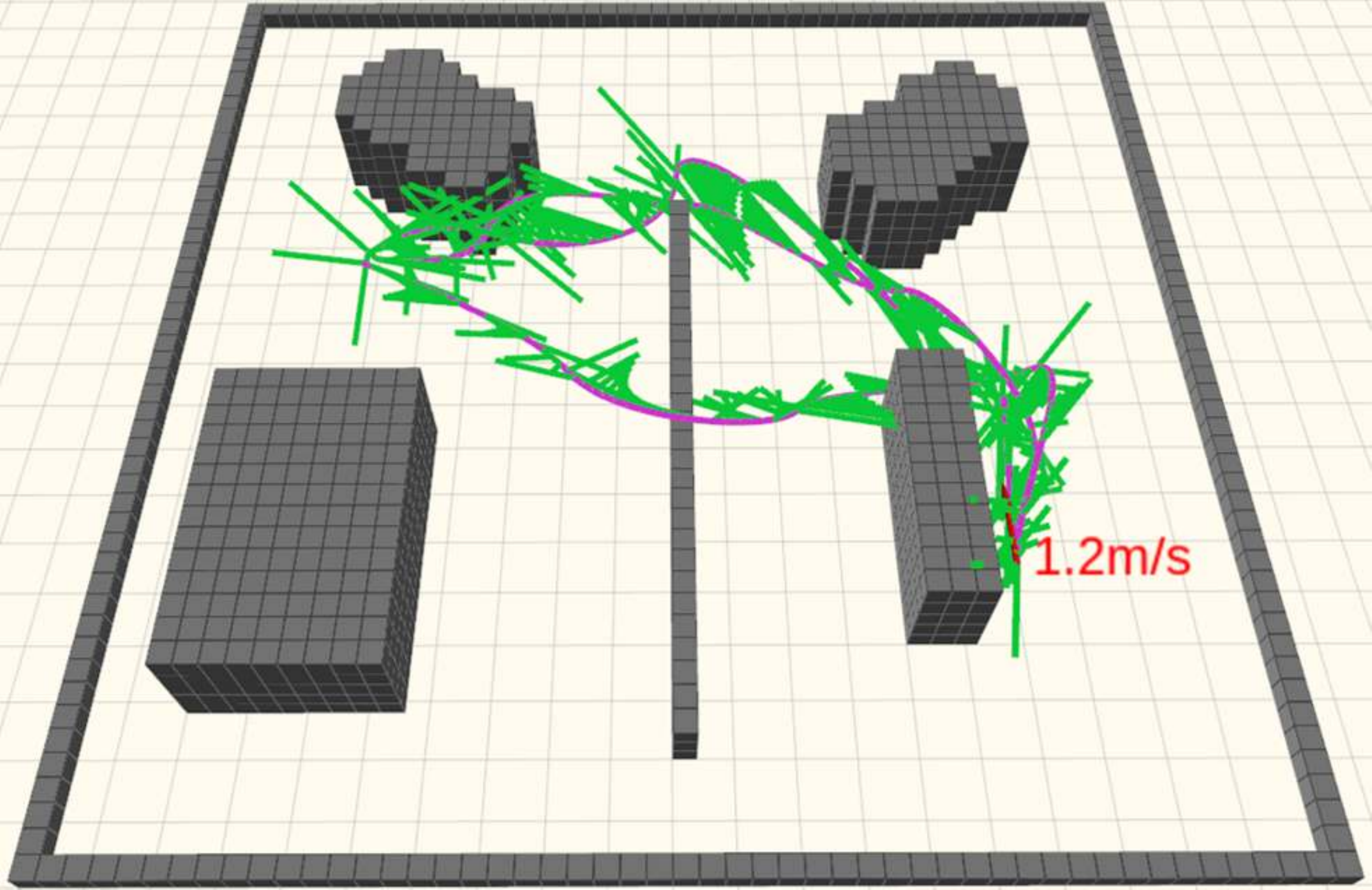}}
		\subfigure[SMP-U3\label{fig:simple3d_smpu3}]{\includegraphics[width=0.49\columnwidth]{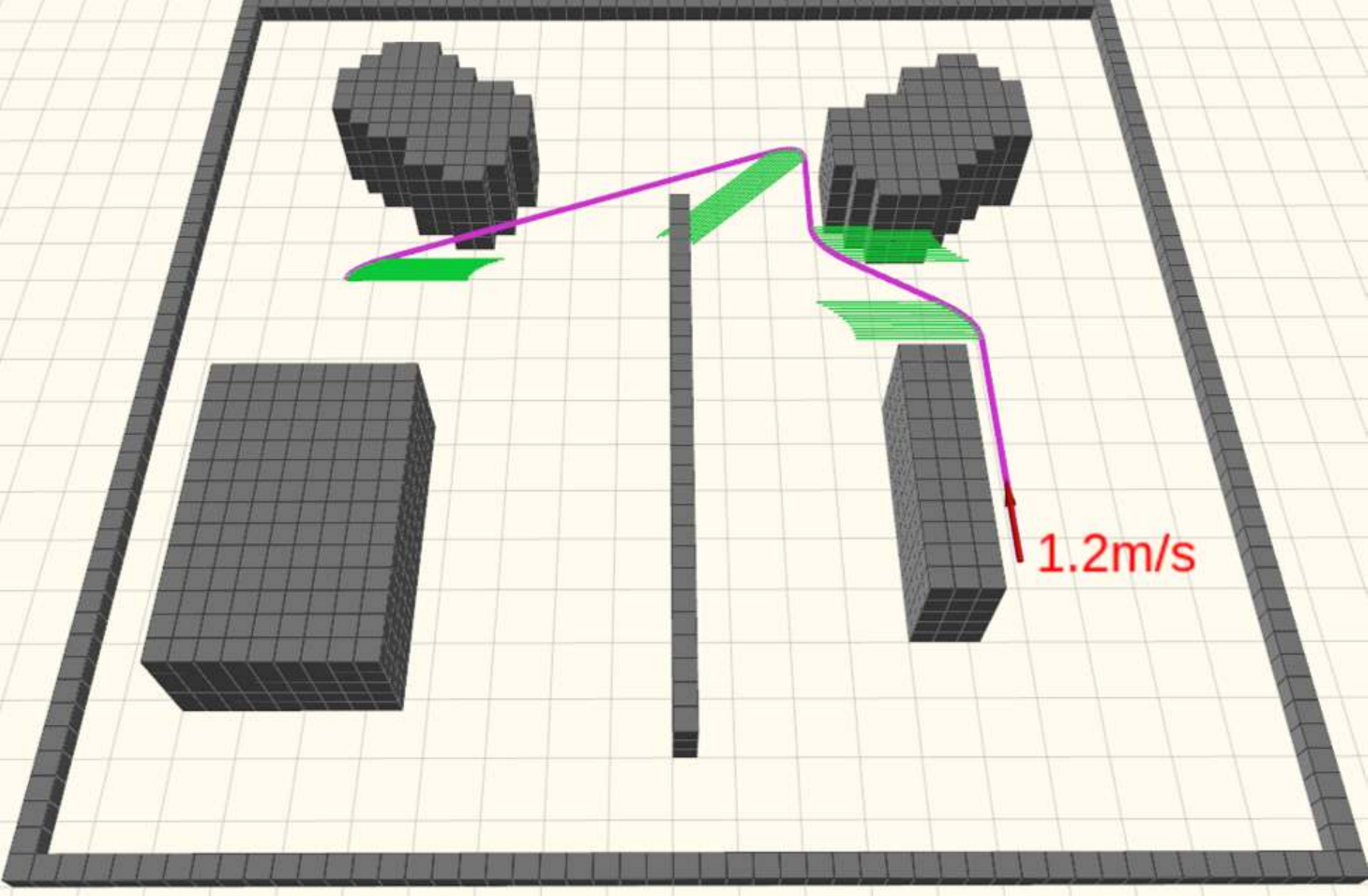}}
		\subfigure[SMP-U5\label{fig:simple3d_smpu5}]{\includegraphics[width=0.49\columnwidth]{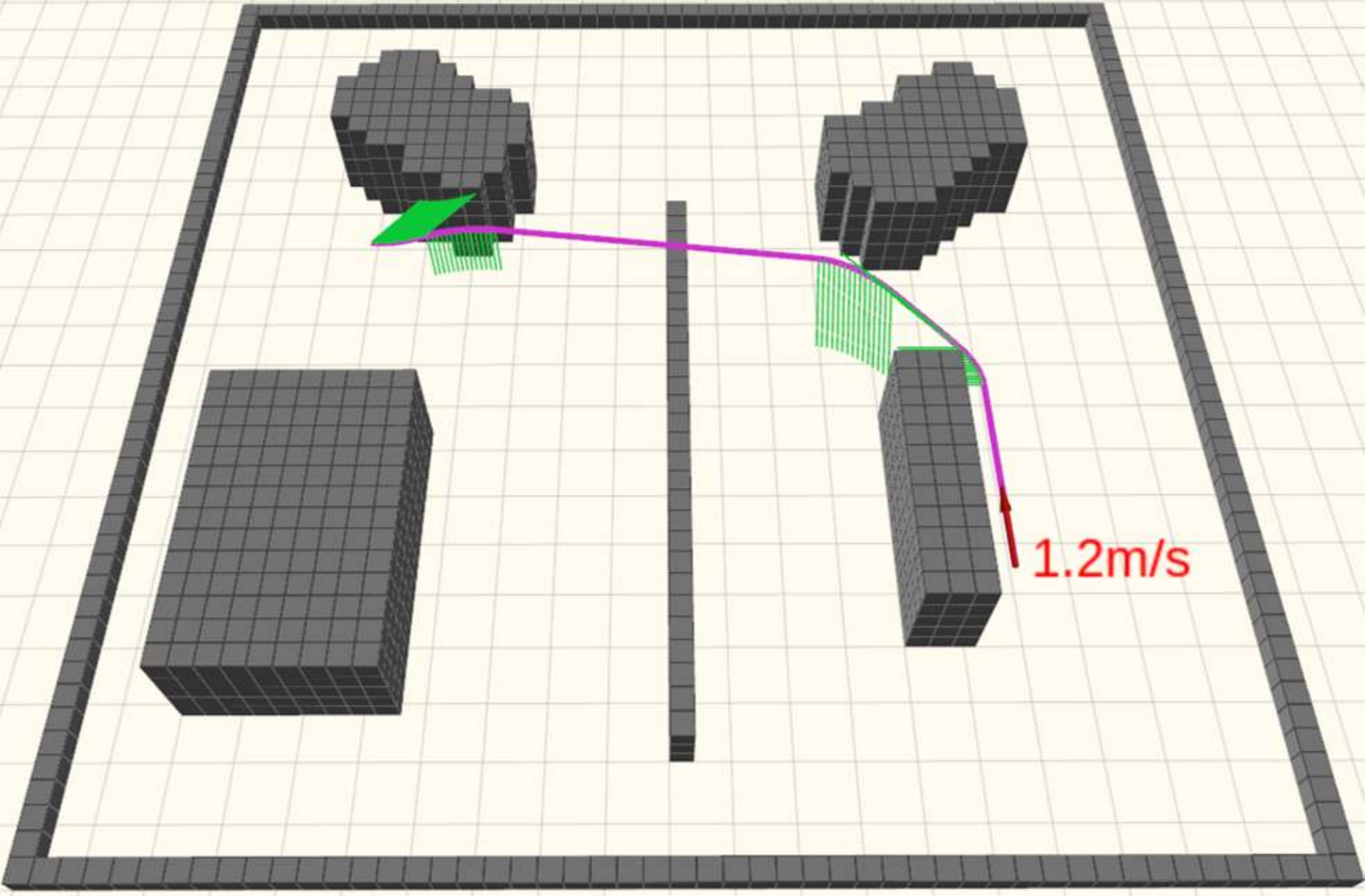}}
		\subfigure[EBK-D1\label{fig:simple3d_bkd3}]{\includegraphics[width=0.49\columnwidth]{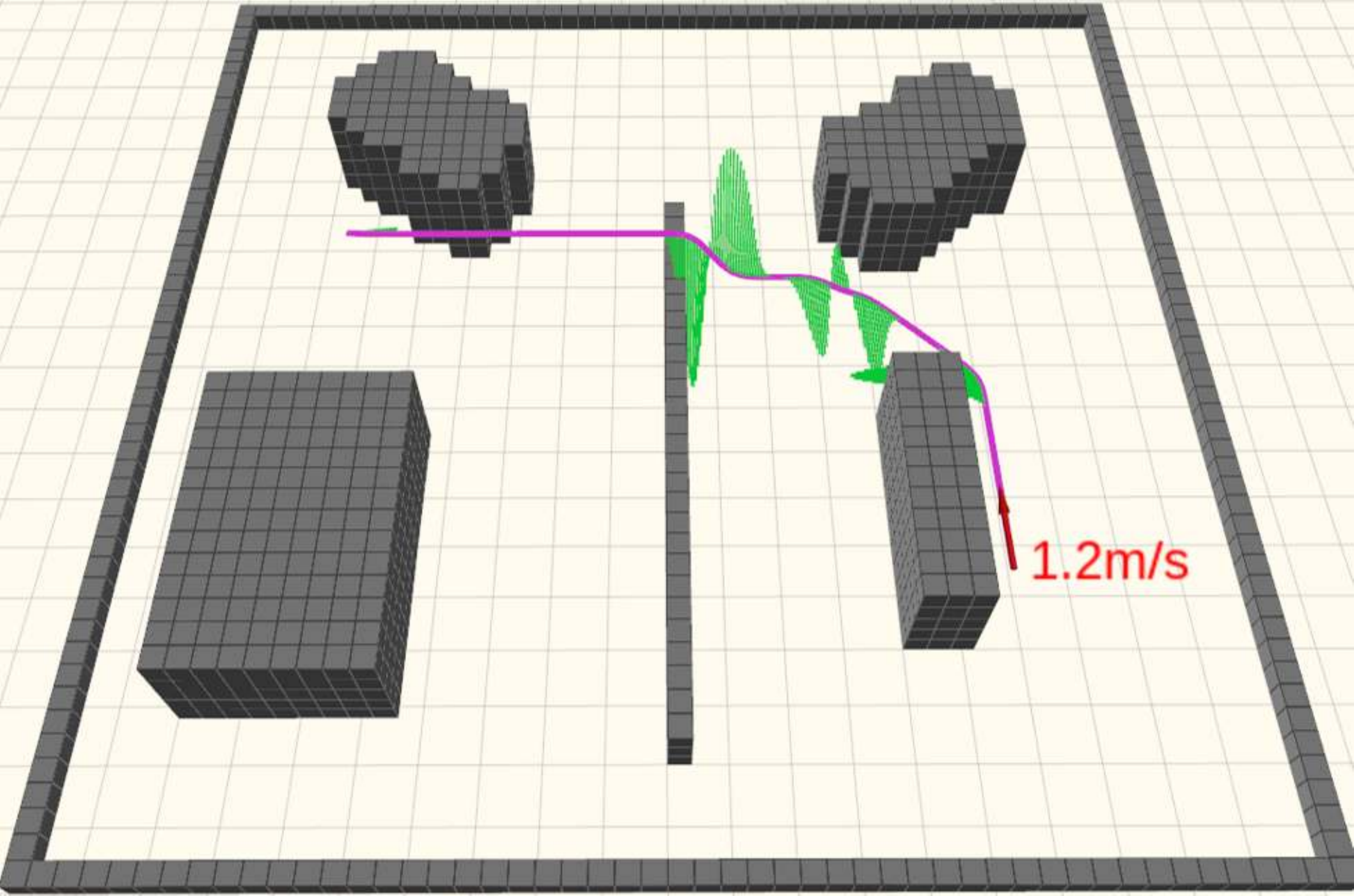}}
		\subfigure[EBK-D2\label{fig:simple3d_bkd6} ]{\includegraphics[width=0.49\columnwidth]{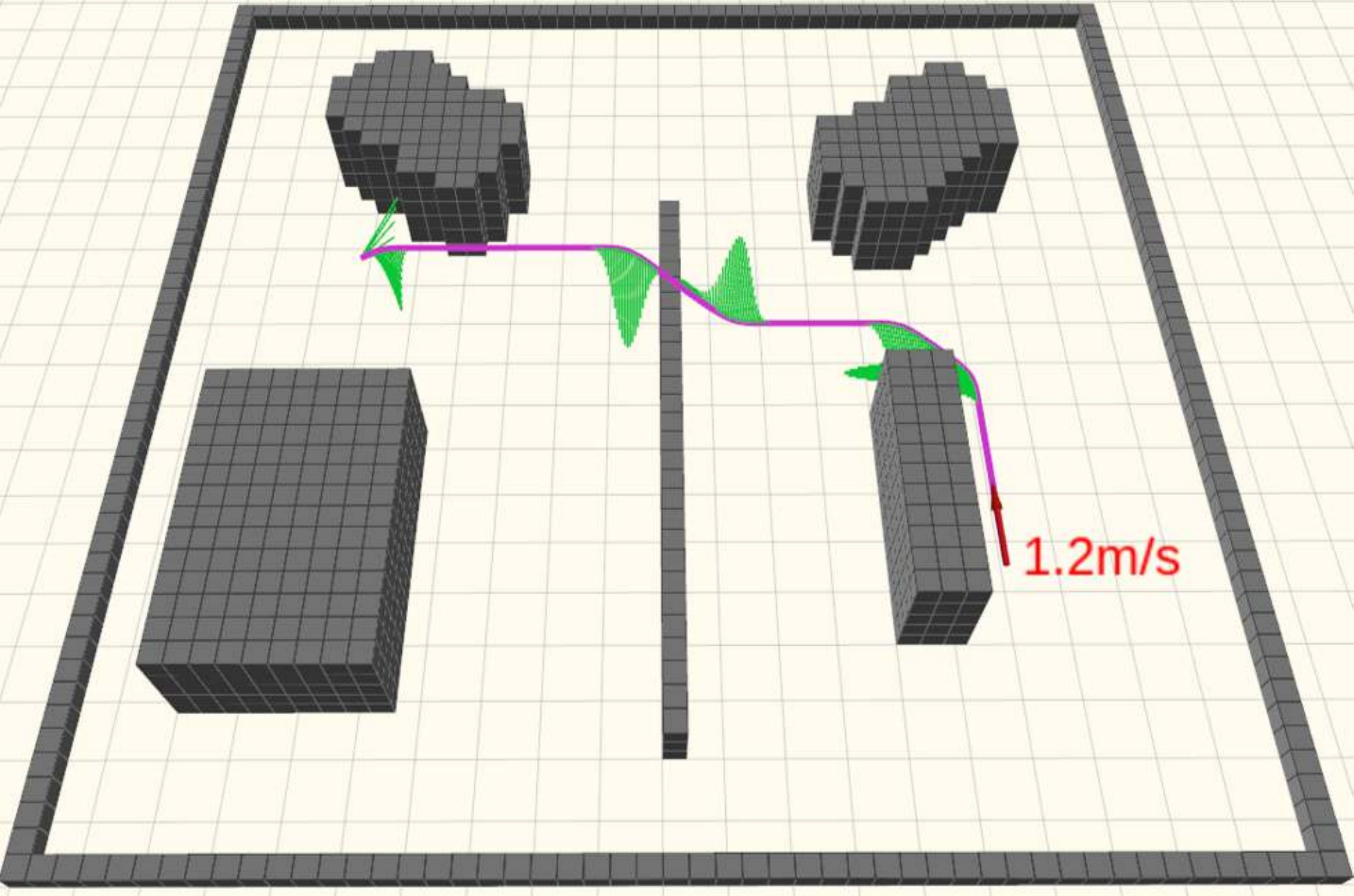}}
	\end{center}
	\caption{Comparisons of different kinodynamic planning approaches.\label{fig:simple3d_tuple} }
	\vspace{-0.4cm}
\end{figure}
All the kinodynamic planners have a similar form of the cost function according to Eq.~\ref{eq:control_cost}. Since both kRRT* and SMP adopt the acceleration-controlled double integrator system, the order of the derivative $l$ we can take is $2$ for all the experiments. The weight $\lambda$ of the trajectory time is set to $20$. We consider two additional metrics, namely, average acceleration and maximum acceleration, which represent the smoothness and feasibility of the resultant trajectory.

For the methods where cell decomposition is needed, such as A* for the geometric method and EBK, the environment is decomposed into cells with a fixed size of 0.2 $m$ for each dimension. The geometric planner requires a \textit{time allocation} module since the path does not contain any time information. Generating the optimal time profile for a path is actually not a trivial problem~\cite{gao18timeall}, and the common practice is based on a heuristic allocation method, such as a trapezoid velocity profile~\cite{liu2017sfc, oleynikova2016ct}. We follow this practice and set the average speed to achieve a similar trajectory duration to the kinodynamic planners. Time allocation is not needed for kinodynamic planners since they directly produce time-profiled trajectories. The initial state is set to a fixed position with non-zero velocity 1.2 $m/s$ and zero acceleration, as shown in Fig.~\ref{fig:simple3d_tuple}. The goal state is static, and its position is regularly sampled with a distance separation of $0.7$ m. In total, there are $136$ collision-free goals available. The velocity and acceleration limit are set to 2 $m/s$ and 4.7 $m/s^2$ for each axis. The $\Delta_t$ for the EBK method is set to $0.17$ s. The statistics averaged over the $136$ rounds of planning are shown in Tab.~\ref{tab:kino_compare}, and the qualitative results for the planning to the same goal state are shown in Fig.~\ref{fig:simple3d_tuple}.

\begin{table}[t]
  \vspace{+0.12cm}
	\centering
	\caption{Comparison of different kinodynamic planning approaches.\label{tab:kino_compare}}
	\begin{tabular}{@{}lccccc@{}}
	\toprule
	 \textbf{Method}
	&\textbf{\makecell{\scriptsize{Run}\\ \scriptsize{Time} (s)}}
	&\textbf{\makecell{\scriptsize{Traj.}\\ \scriptsize{Dur.} (s)}}
	&\textbf{\makecell{\scriptsize{Acc. Cost}\\ ($m^2/s^3$)}}
	&\textbf{\makecell{\scriptsize{Ave. Acc.}\\($m/s^2$)}}
	&\textbf{\makecell{\scriptsize{Max. Acc.}\\($m/s^2$)}} \\
	\midrule
	kRRT*-T600  			     & 0.661 & 4.04 & 30.6  & 2.42 & 4.18\\
	SMP-U5      				 & 2.878 & \textbf{3.86}   & \textbf{12.8}  & 1.36 & 2.97\\
	EBK-D2       				 & \textbf{0.383} & 4.52 & 14.4  & 0.95 & 4.60\\
	\hline
	$\text{Geometric}^\dagger$ & 0.018 & 4.34 & 108.8 & 4.60 & 9.48\\
	$\text{kRRT*-T100}^\dagger$  & 0.110 & 4.60 & 38.8  & 2.54 & 4.28\\
	$\text{SMP-U3}^\dagger$      & 0.111 & 4.47 & 32.3  & 1.70 & 4.70\\
	$\text{EBK-D1}^\dagger$       & \textbf{0.034} & \textbf{4.31} & \textbf{15.2}  & 1.10 & 4.56\\
	\bottomrule
	\end{tabular}
	\vspace{-0.7cm}
\end{table}

\begin{table*}[h]
	\centering
	\caption{Comparison of different trajectory optimization approaches.\label{tab:optimization_comparison}}
	\begin{tabular}{@{}llccccccccc@{}} \toprule
	 \textbf{Map}
	&\textbf{Method}
	&{\makecell{\textbf{Density}\\(pillars$/m^2$)} }
	&\textbf{\makecell{Success. \\Frac. ($\%$)}}
	&\textbf{\makecell{Run-\\Time (s)}}
	&\textbf{\makecell{Traj.\\Dura. (s)}}
	&\textbf{\makecell{Jerk Cost\\ ($m^2/s^5$)}}
	&\textbf{\makecell{Ave. Vel.\\($m/s$)}}
	&\textbf{\makecell{Ave. Acc.\\($m/s^2$)}}
	&\textbf{\makecell{Max. Acc.\\($m/s^2$)}}
	&\textbf{\makecell{Traj. Len.\\($m$)}}\\
	\midrule
	\multirow{9}{*}{\makecell{Random\\Map}}
								&\multirow{3}{*}{CT~\cite{oleynikova2016ct}} & 0.1 & 95  & 0.030 & 8.1 & 70.2  & 1.73 & 0.93 & 2.63 & 14.3 \\
								&                                            & 0.2 & 68  & 0.082 & 8.0 & 91.6  & 1.72 & 1.01 & 2.90 & 13.9 \\
								&                                            & 0.4 & 46  & 0.073 & 7.9 & 105.6 & 1.69 & 1.05 & 3.10 & 13.7 \\
								\cline{2-11}
	              &GS~\cite{feigao2017hg} & 0.1 & \textbf{100} & 0.062          & 11.1 			 	  & 186.8 				 & 1.39 & 0.87 & 2.99 & 15.6 \\
								&EO (proposed)          & 0.1 & \textbf{100} & \textbf{0.012} & \textbf{10.6} & \textbf{174.0} & 1.43 & 0.38 & 1.35 & 14.1 \\
								\cline{2-11}
								&GS~\cite{feigao2017hg} & 0.2 & \textbf{100} & 0.075 				 & 12.8          & 205.9          & 1.28 & 0.85 & 2.89 & 15.9 \\
								&EO (proposed)          & 0.2 & \textbf{100} & \textbf{0.012} & \textbf{11.3} & \textbf{181.2} & 1.38 & 0.49 & 1.85 & 14.5\\
								\cline{2-11}
								&GS~\cite{feigao2017hg} & 0.4 & \textbf{100} & 0.206          & 12.9          & 190.8          & 1.33 & 0.89 & 3.10 & 17.4\\
								&EO (proposed)          & 0.4 & \textbf{100} & \textbf{0.013} & \textbf{12.4} & \textbf{132.4} & 1.36 & 0.55 & 1.34 & 16.2 \\
								\bottomrule
	\end{tabular}
	\vspace{-0.2cm}
\end{table*}

According to the qualitative results in Fig.~\ref{fig:simple3d_tuple}, there is a significant difference in the performance. For the geometric method (Fig.~\ref{fig:simple3d_geo}), since the shortest path diverges from the initial velocity direction, the parameterization is jerky at the beginning. Moreover, as the unconstrained QP cannot enforce the dynamical feasibility, the generated trajectory is infeasible due to the non-static initial state. As for the kRRT* with a $100$ ms time budget (Fig.~\ref{fig:simple3d_krrt100}), it can respect the initial state of the quadrotor since the control effort is directly considered in the sampling process. It also guarantees the dynamical feasibility of the resultant trajectory by constraining the control input and states along the tree edges. However, the trajectory quality is unsatisfactory due to the limited sampling. We also observe some unpredictable randomized behavior as shown by the three rounds of planning with exactly the same initial state and goal state in Fig.~\ref{fig:simple3d_krrt100}. Meanwhile, for the SMP method, the shape of the trajectory of SMP-U3 is not natural since the resolution of the control input is large.\footnote{The acceleration bound we use is larger than that in~\cite{liu2017smp}.} Some maneuvers such as flying over the little step in the middle are not included in the solution space of SMP-U3. SMP-U5 performs better than SMP-U3 due to finer discretization. Finally, EBK-D1 and EBK-D2 both generate an initial-state-aware smooth trajectory. EBK-D2 finds a slightly better trajectory than EBK-D1 according to the acceleration profile.

It is notable that the trajectory provided by the EBK method has continuity up to snap, which is good for controlling quadrotors. We can observe from Fig.~\ref{fig:simple3d_tuple} that the kRRT* trajectory only has continuity up to acceleration and that of the SMP has continuity up to velocity, due to the restriction of computation complexity and order of the system model.

From the quantitative results in Tab.~\ref{tab:kino_compare}, among the real-time methods (with $\dagger$), EBK-D1 finds the lowest cost trajectory, achieving a $15.2$ $m^2/s^3$ total cost within $0.034$ s. Our method shows superior performance given the real-time requirement. It is of interest to examine the situation where we have a time budget in the range of seconds. In that case, SMP-U5 achieves a lower cost than EBK-D1 and EBK-D2. The reason is that the trajectory duration of EBK-D2 cannot be efficiently reduced since the control points have to be expanded step by step on the discrete grid. This illustrates the limitation induced by the discretization for the EBK method. However, the difference between the SMP method and EBK method is minor, meaning the EBK method can provide competitive solutions given a run-time budget of seconds.

\subsection{Analysis of the Elastic Optimization}
\label{sec:eo_performance}
To evaluate the performance of EO, we compare our method with two state-of-the-art trajectory optimization methods: the CT method~\cite{oleynikova2016ct} and GS method~\cite{feigao2017hg}. Two test environments are provided, namely, a random map (shown in Fig.~\ref{fig:analysis_opt_qualit_randmap}) and a Perlin noise map\footnote{https://github.com/HKUST-Aerial-Robotics/mockamap} (shown in Fig.~\ref{fig:analysis_opt_qualit_perlin}). The map size is fixed to $20\,m \times 20\,m \times 4\,m$ for both maps. The start location is fixed to the center of the test field for the qualitative experiments in Fig.~\ref{fig:analysis_eo_qualit_randmap} and Fig.~\ref{fig:analysis_eo_qualit_perlin}, and is fixed to the left bottom corner for all the experiments in Tab.~\ref{tab:optimization_comparison} to test the performance for long trajectories. The goal location is regularly sampled in the test field with a distance separation of $1$ m. We vary the obstacle density of the random map from 0.1 pillars/$m^2$ to 0.4 pillars/$m^2$, where the pillar size is set to $0.5\,m \times 0.5\,m$. The qualitative results are provided in Fig.~\ref{fig:optimization_qualitative} and the quantitative statistics of the optimization performance are organized in Tab.~\ref{tab:optimization_comparison}. Note that we omit the statistics for the Perlin noise map in Tab.~\ref{tab:optimization_comparison} since the trend is similar to that of the random map. All the optimization methods can handle high-order parameterization and they are set to minimizing the integral of the squared jerk. According to Eq.~\ref{eq:control_cost}, the jerk cost listed in Tab.~\ref{tab:optimization_comparison} has unit $m^2/s^5$ and represents the accumulated integral of the squared change rate of the acceleration, which represents the total control efforts.

The CT method uses a gradient-based non-linear optimization process to optimize the polynomial coefficients such that the resultant trajectory is collision free. It requires a Euclidean signed distance field (ESDF) to evaluate the collision cost. Following the practice in~\cite{oleynikova2016ct}, we provide an initial polynomial trajectory as an initial guess, which is parameterized by a straight-line guiding path. The number of segments is determined by a $3$ m distance separation. The GS method shares a similar formulation to the CT method, with the difference being that the GS method starts from a collision-free initial guess, which is provided by A* search in the experiment. Both the CT method and GS method require \textit{time allocation}, and in the experiment, we use the trapezoid velocity profile and scale the total allocated time such that different optimization methods achieve a similar average velocity, as shown in Tab.~\ref{tab:optimization_comparison}. The CT method has a larger average velocity due to the cases where it fails to resolve collision and the trajectory does not contain necessary deceleration. The success fraction is calculated by counting the collision-free and dynamically feasible trajectories among the total number of rounds.

Firstly, we examine the optimization reliability, i.e., the success fraction for the different methods. As shown in Tab.~\ref{tab:optimization_comparison}, the CT method is sensitive to the obstacle density. For a density of $0.1$ pillars/$m^2$, the success fraction of the CT method is 95$\%$, which means that in obstacle-sparse environments, the CT method can resolve collision efficiently. However, when the density increases to $0.4$ pillars/$m^2$, we observe that the success fraction drops significantly to $46\%$. The reason is that the CT method easily gets stuck in the infeasible local minimum when the trajectory is inside the cluttered obstacle. As shown in Fig.~\ref{fig:analysis_opt_qualit_randmap}, the trajectory of the CT method gets stuck between two pillars where there is not enough clearance considering the quadrotor size. On the other hand, we do not observe a drop in the success fraction for either the GS method or EO method. The reason is that the GS method starts from a collision-free initial guess and has a dominant collision penalty compared to its smoothness cost. The EO method has a high success fraction according to its theoretical guarantee.

\begin{figure}[t]
	\begin{center}
		\subfigure[Random map (0.2 pillars/$m^2$)\label{fig:analysis_eo_qualit_randmap}]{\includegraphics[width=0.49\columnwidth]{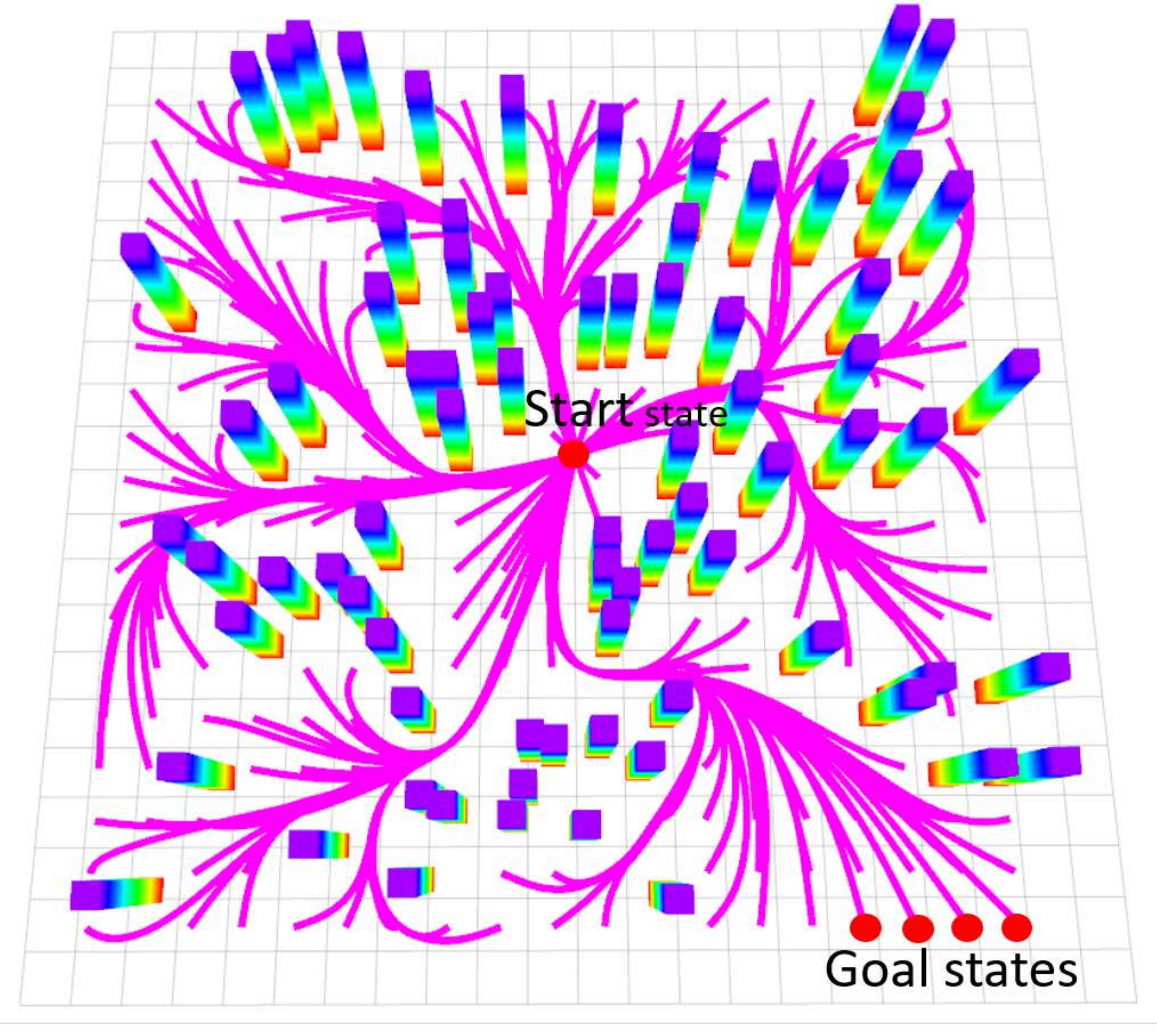}}
		\subfigure[Perlin noise map\label{fig:analysis_eo_qualit_perlin}]{\includegraphics[width=0.49\columnwidth]{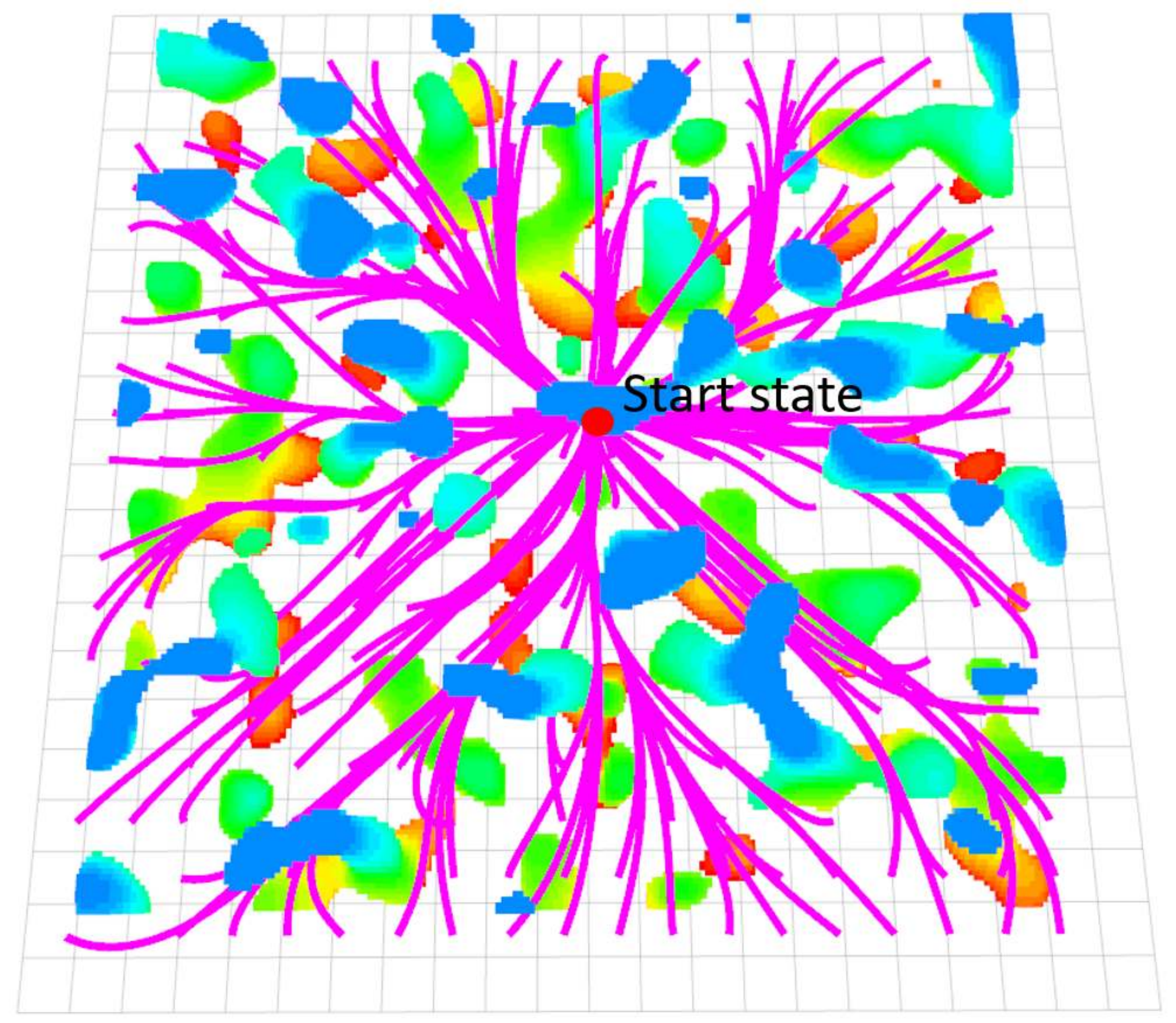}}
		\subfigure[Random map (0.2 pillars/$m^2$)\label{fig:analysis_opt_qualit_randmap}]{\includegraphics[width=0.49\columnwidth]{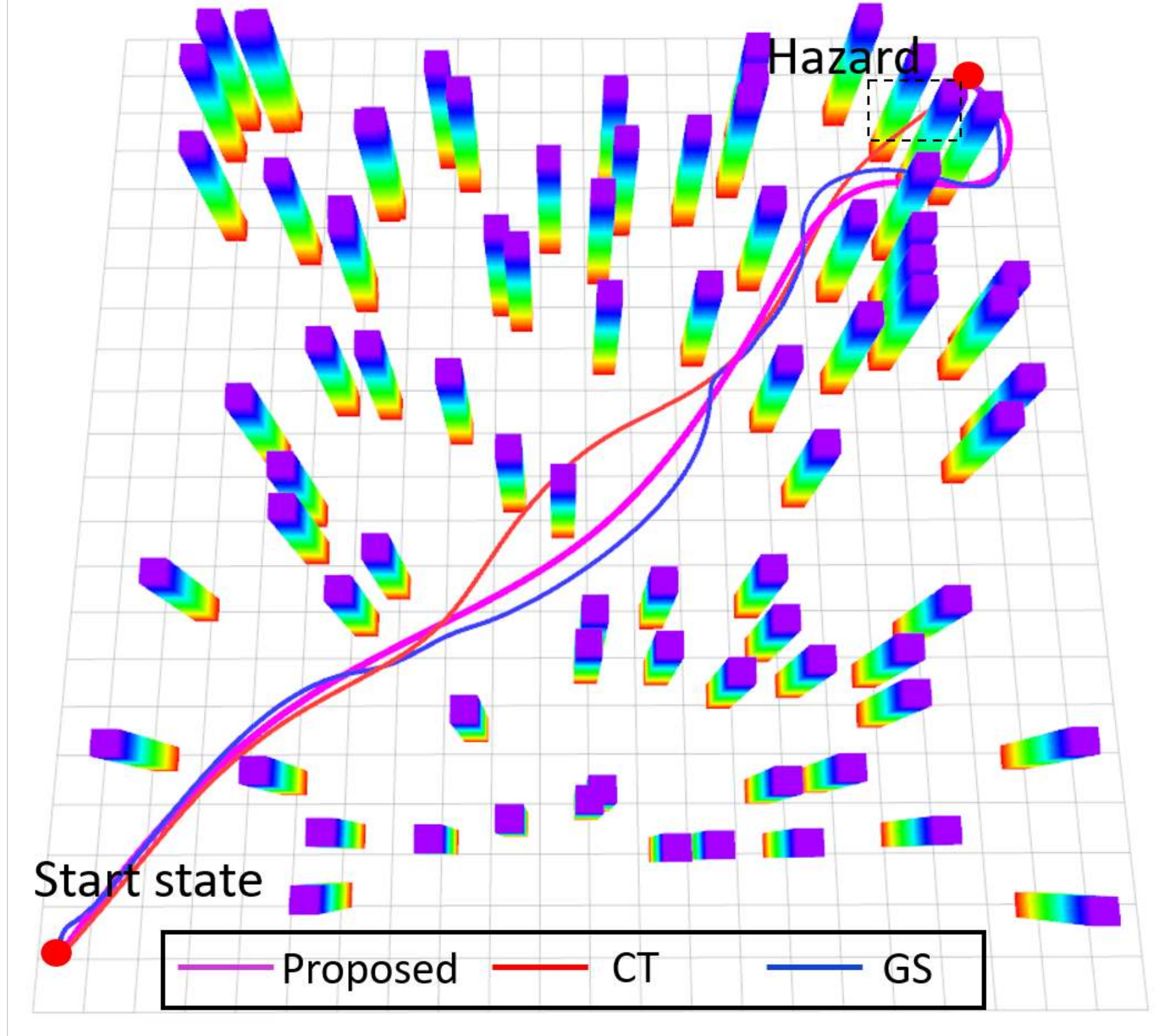}}
		\subfigure[Perlin noise map\label{fig:analysis_opt_qualit_perlin}]{\includegraphics[width=0.49\columnwidth]{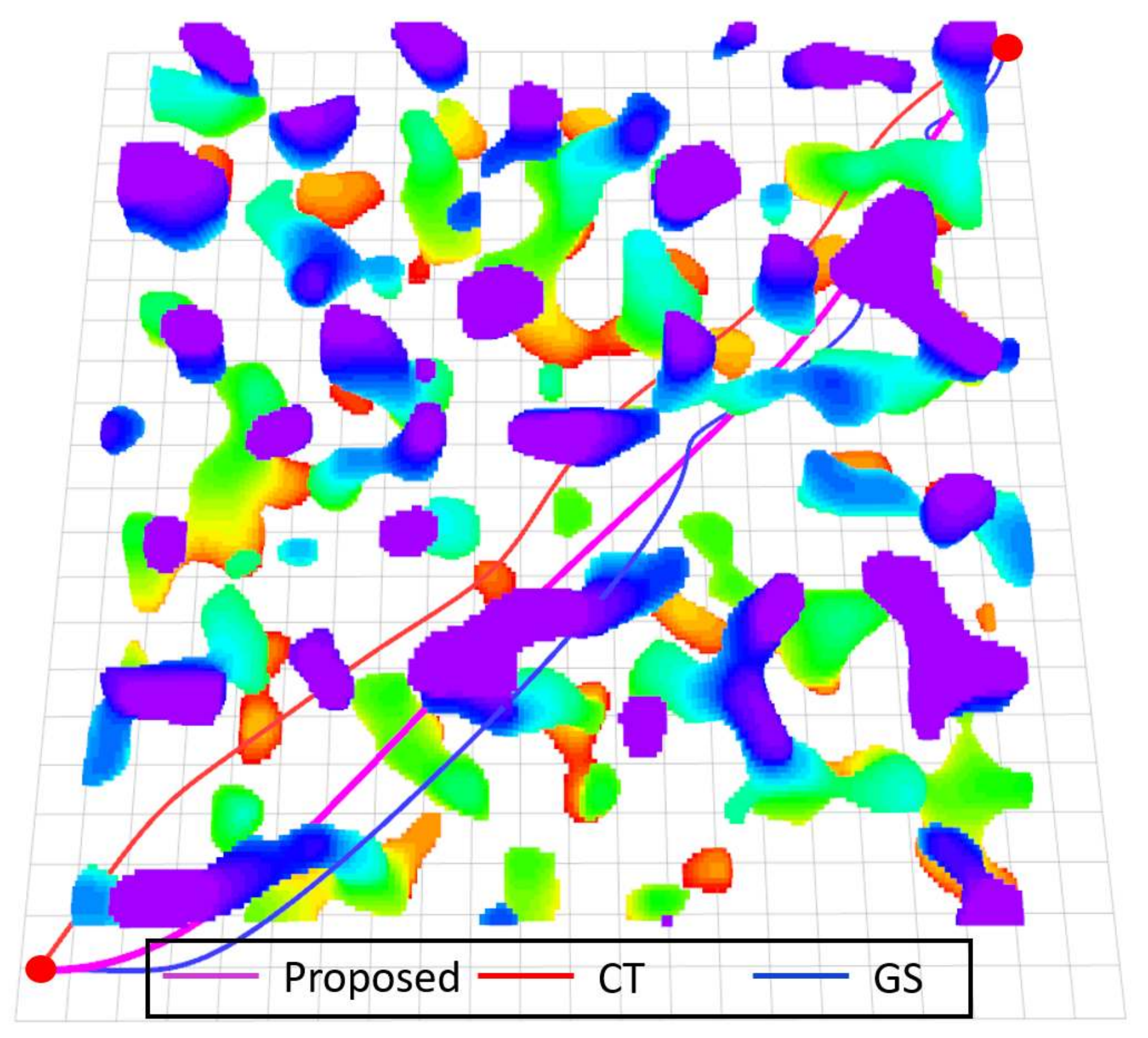}}
	\end{center}
	\caption{Comparisons of different trajectory optimization methods on two different maps. The EO method is shown in \textit{purple}, the CT method is shown in \textit{red} and the GS method is shown in \textit{blue}. \label{fig:optimization_qualitative}}
 \vspace{-0.6cm}
\end{figure}

Secondly, we focus on the two reliable methods, namely, the GS method and EO method, and further investigate the trajectory statistics. As shown in Tab.~\ref{tab:optimization_comparison}, the average trajectory durations of the two methods are close, so the comparison of the jerk cost is fair. For an obstacle density of $0.2$ pillars/$m^2$, averaged over the $135$ rounds, the jerk cost of the EO method is $181.2$ $m^2/s^5$, which is smaller than that of the GS method. Similar results are also observed for different densities. The reason is that the GS method is sensitive to the time allocation, since it also starts with an unparamterized path. However, for unknown environments, the shape of the initial path may vary and there is no systematic way to allocate the time. The heuristic trapezoid velocity profile may fail to provide a good initial guess when the initial path is not in a regular shape. For the EO method, since it starts from a time-parameterized B-spline trajectory, there is no need for time allocation.

From the efficiency perspective, for the GS method to achieve a similar jerk cost to the EO method, it needs run-times of $0.062$ s, $0.075$ s and $0.206$ s on average for the three densities,~which are significantly longer than the EO method.
Note that the GS method is based on a non-linear optimization formulation. In the experiments, the non-linear optimization of the GS method is terminated when the non-linear solver (NLOPT~\cite{Johnson2011nlopt}) reports convergence. For the GS method to converge to a compatible solution to that of our EO method, it requires a significantly longer run-time.
Moreover, the run-time of the GS method is sensitive to the density, since for cluttered environments, more segments of the polynomial are needed, which may affect the convergence rate of the GS method. However, the EO method has lower run-times for the three different densities. The efficiency of the EO method is affected by the total number of control points, but we observe that the efficiency difference is minor for the different densities.

\subsection{Analysis of the Run-Time Efficiency}
\label{sec:run_time_analysis}
In this section, we test our replanning system, combining the EBK search with EO refinement. For the EBK search, we adopt EBK-D1 since it is the most efficient and the trajectory quality is satisfactory as verified in Sect.~\ref{sec:analysis_bk}. We provide two different maps, namely, the random map and Perlin noise map. We evaluate the run-time efficiency of our replanning system, and list the statistics of all components in Tab.~\ref{tab:runtime_analysis}. The overall trajectory illustrating the whole round trip is shown in Fig.~\ref{fig:runtime_sim}.
\begin{table}[b]
	\caption{Run-time analysis on different maps}
	\label{tab:runtime_analysis}
	\resizebox{\columnwidth}{!}{
		\begin{tabular}{ ccccccccc} \toprule
			\textbf{Maps}    &
			\makecell{$\#$ \\ \textbf{Replans} }&
			\textbf{Time (s)} &
			\textbf{\makecell{EBK\\Search}} &
			\makecell{$\#$ \\ \textbf{Opt.}}&
			\textbf{Time (s)} &
			\textbf{\makecell{Tube\\Expan.}} &
			\textbf{\makecell{Traj.\\Opt.}} &
			\textbf{\makecell{Total\\ Opt. } } \\
			\midrule
			\makecell{Random map\\ ( 0.25 pillars/$m^2$ ) }  &  76  & \makecell{Avg\\Max\\Std}  & \makecell{\textbf{0.017}\\0.049\\0.010} & 993 &\makecell{Avg\\Max\\Std}  & \makecell{\textbf{0.002}\\0.009\\0.001} & \makecell{\textbf{0.021}\\0.043\\0.010}  & \makecell{\textbf{0.023}\\0.044\\0.010}\\
			\hline
			Perlin Map  &  19  & \makecell{Avg\\Max\\Std}  & \makecell{\textbf{0.014}\\0.026\\0.006} & 1044 &\makecell{Avg\\Max\\Std}  & \makecell{\textbf{0.002}\\0.008\\0.001} & \makecell{\textbf{0.028}\\0.058\\0.010}  & \makecell{\textbf{0.030}\\0.061\\0.011}\\
			\bottomrule
		\end{tabular}
	}
\end{table}
As shown in Tab.~\ref{tab:runtime_analysis}, on the random map, the EBK-D1 method consumes an average computing time of $0.017\,s$ with a standard deviation of $0.01\,s$. The elastic tube expansion method can be finished in $0.002\,s$, and the optimization can be done in $0.021\,s$. On the Perlin noise map, which contains unstructured 3D obstacles, our method has a similar performance, showing that our method works well in complex 3-D environments.
\begin{figure}[t]
	\begin{center}
		\subfigure[Replan on the random map]{\includegraphics[trim={13.5cm 2cm 15cm 0cm},clip,width=0.22\textwidth]{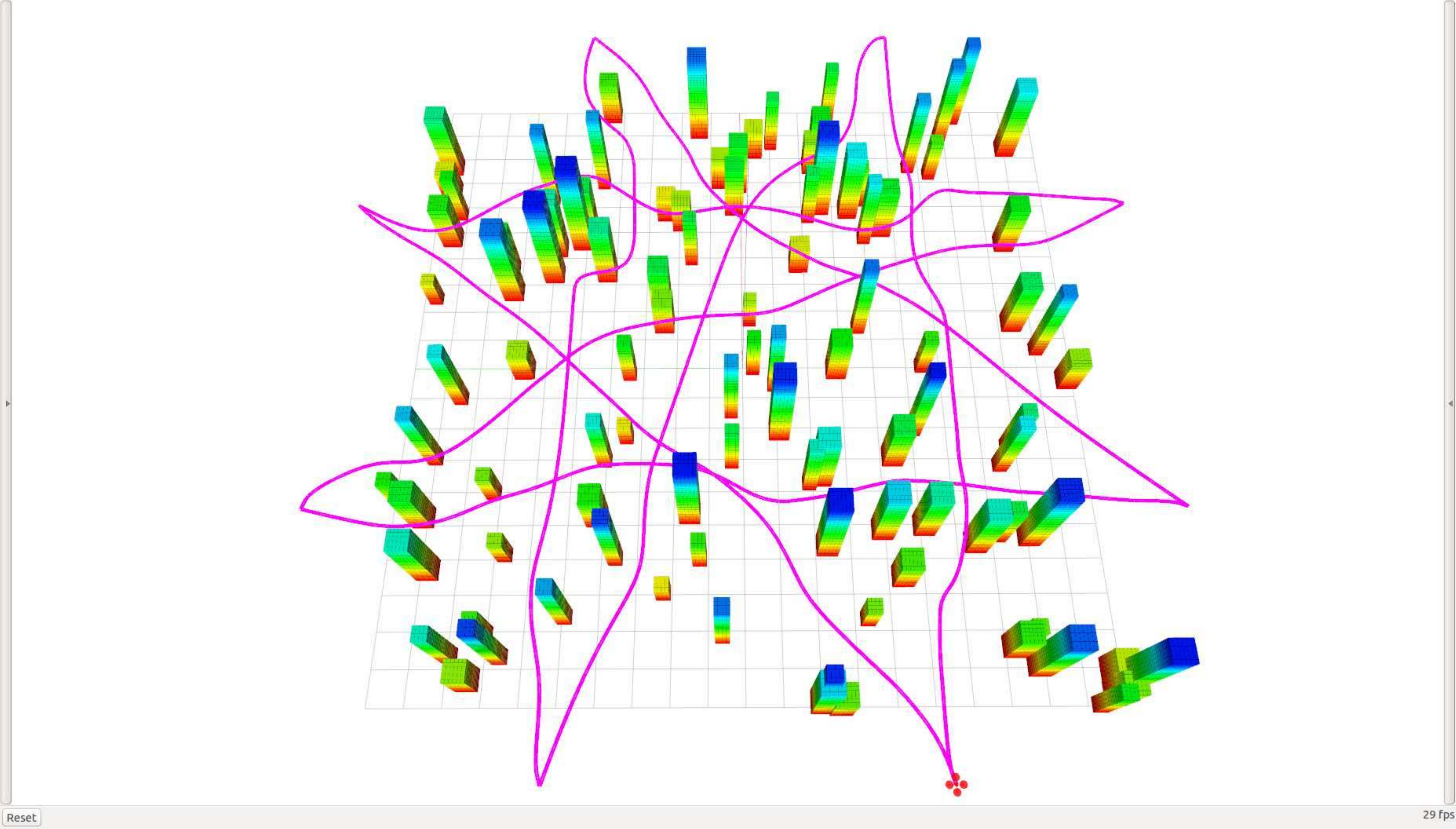}}
		\subfigure[Replan on the Perlin noise map]{\includegraphics[trim={10cm 1.5cm 14cm 1cm},clip,width=0.24\textwidth]{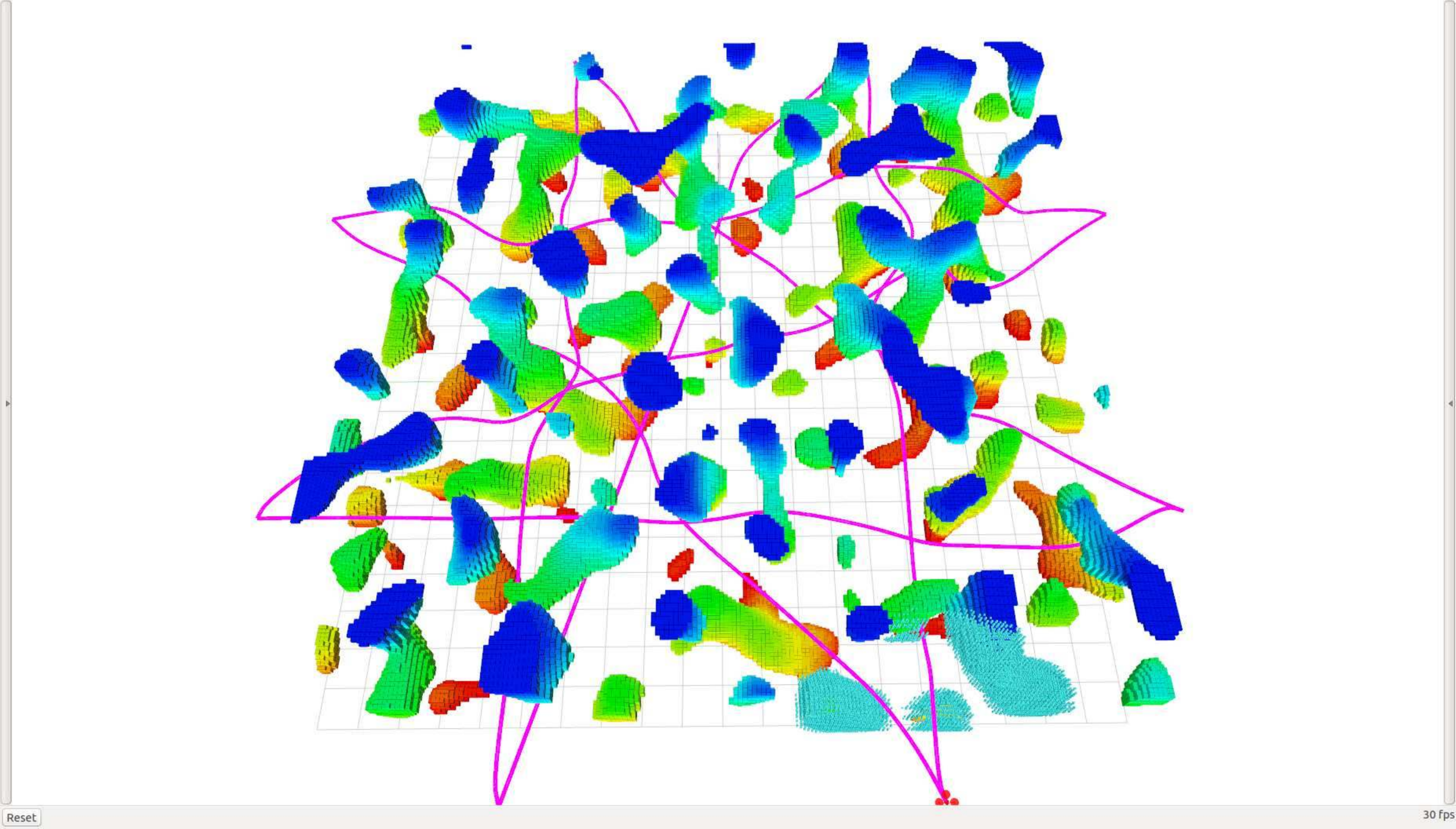}}
	\end{center}
	\caption{Illustration of our replanning system on different maps.\label{fig:runtime_sim}}
	\vspace{-0.7cm}
\end{figure}

\subsection{Comparison of the Replanning Framework}
In this section, we conduct a system comparison with the SMP method~\cite{liu2017smp}. We use SMP-U3 since SMP-U5 cannot work in real-time. Moreover, we conduct an ablation test, which excludes the initial-state aware EBK search and adopts the naive position-only A* search combined with EO local reshaping. We call the method for the ablation test A*-EO. The ablation test validates the critical role of the kinodynamic search in replanning.

We set up a challenging obstacle-cluttered 3-D complex simulation environment containing walls, 3-D steps (free space below) and pillars, as shown in Fig.~\ref{fig:replan_map}. The replanning strategy is choosing a local goal state on a given straight-line guiding path with a local replanning range of 5 $m$. The simulated quadrotor is equipped with a depth camera which has a sensing range of $4$ m. The maximum velocity and acceleration bound are set to 2 $m/s$ and 3.2 $m/s^2$ for each axis. For SMP-U3, we use a conservative bound with maximum acceleration 1.0 $m/s^2$ for each axis since SMP-U3 can only work with a narrow velocity and acceleration range. Using a large dynamic range for SMP-U3 will result in a very sparse primitive graph, which does not even contain one feasible solution. For the EBK search of our method, we use EBK-D1 with a $60\times 60 \times 20$ uniform grid. For the EO approach, 12 control points are refined as the window moves forward. We evaluate the replanning system from the trajectory statistics and time efficiency perspectives, as shown in Tab.~\ref{tab:simulation_replan}.
\begin{figure}[H]
	\begin{center}
		\includegraphics[trim={0cm 0cm 0cm 0cm},clip,width=0.35\textwidth]{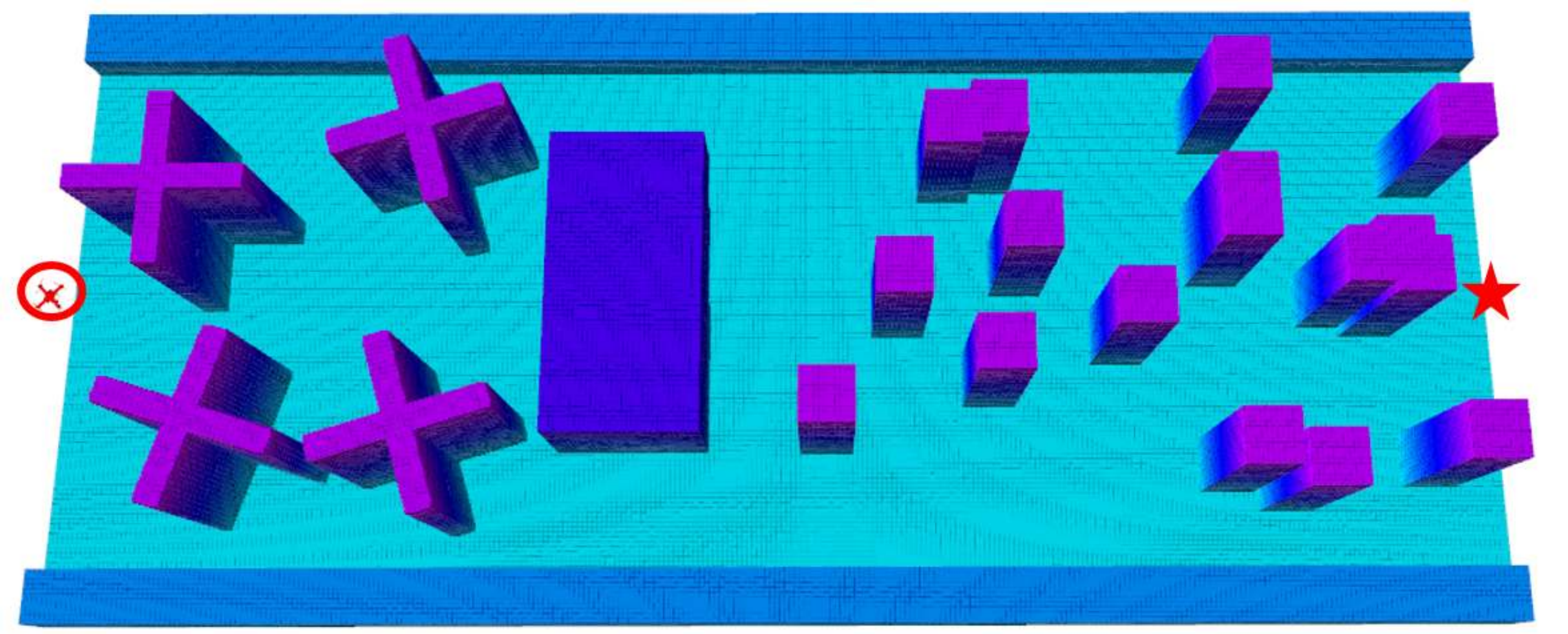}
	\end{center}
	\caption{Illustration of the simulated environment for benchmarking.\label{fig:replan_map}}
\end{figure}
\begin{figure*}[t]
	\begin{center}
		\subfigure[SMP-U3 (Conservative)\label{fig:replan_tuple_SMP}]{\includegraphics[trim={0cm 0cm 0cm 0cm},clip, width=0.30\textwidth]{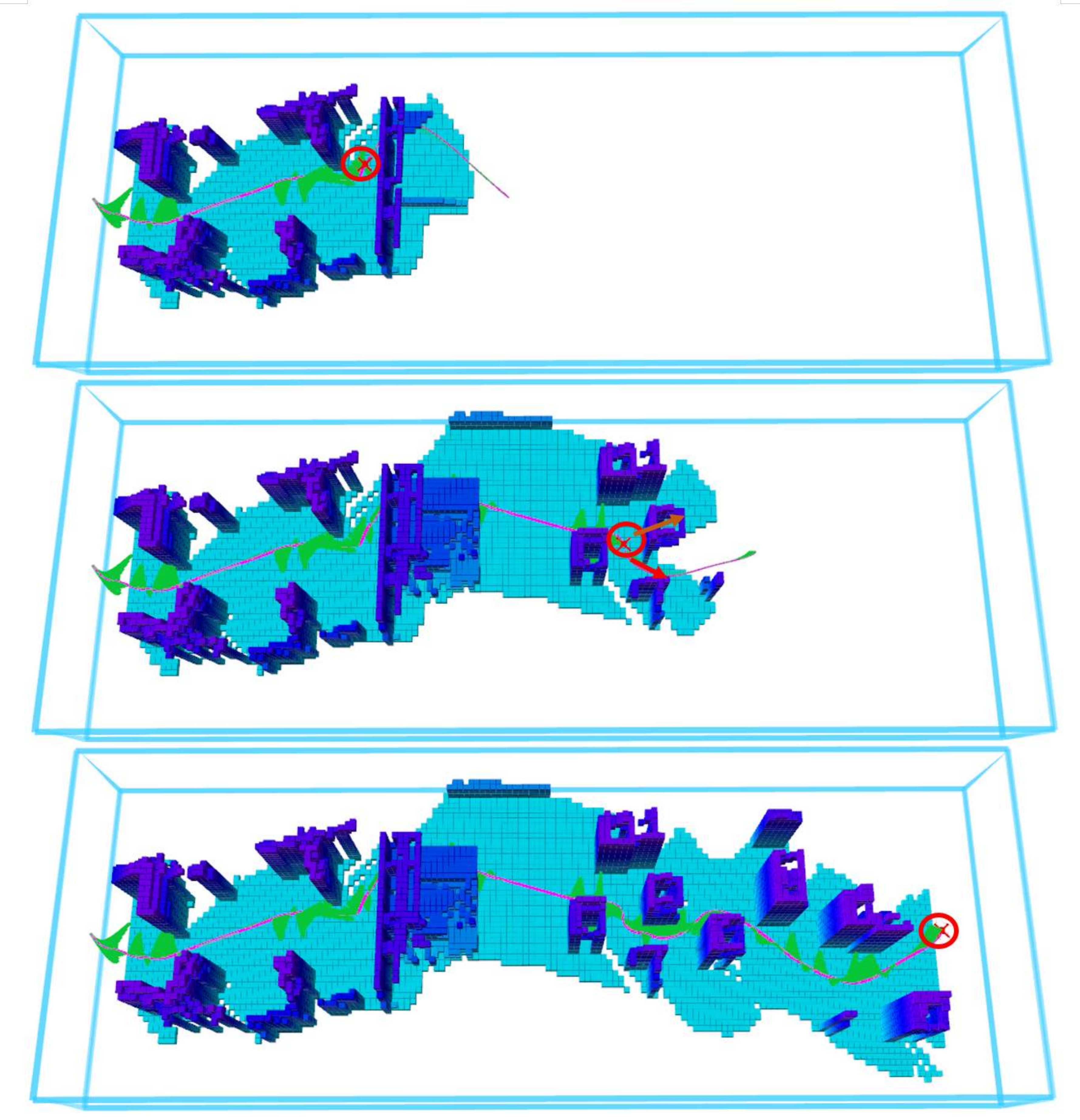}}
		\subfigure[A*-EO\label{fig:replan_tuple_astar} ]{\includegraphics[trim={0cm 0cm 0cm 0cm},clip,width=0.30\textwidth]{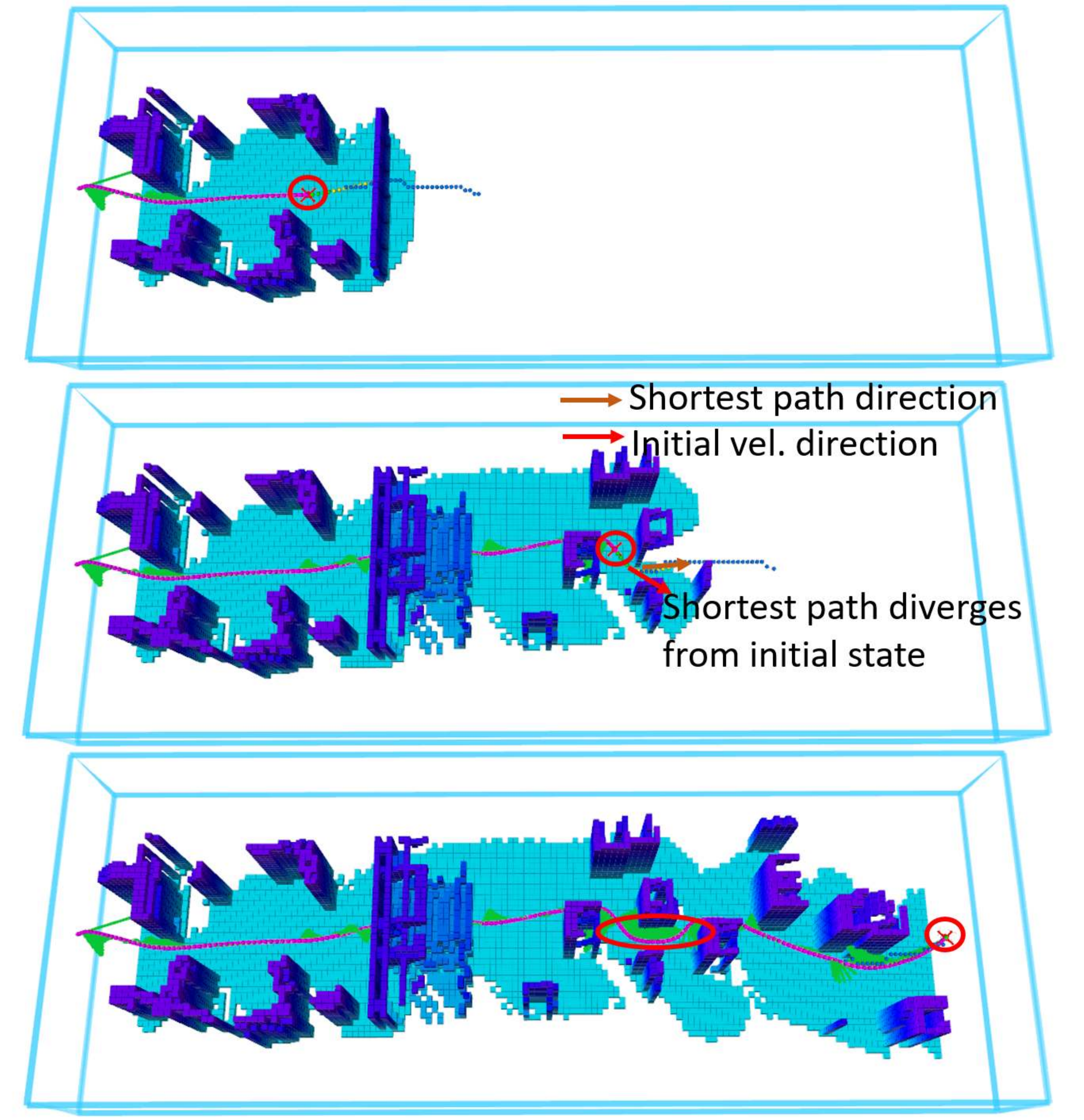}}
		\subfigure[Our method\label{fig:replan_tuple_our}]{\includegraphics[trim={0cm 0cm 0cm 0cm},clip,width=0.30\textwidth]{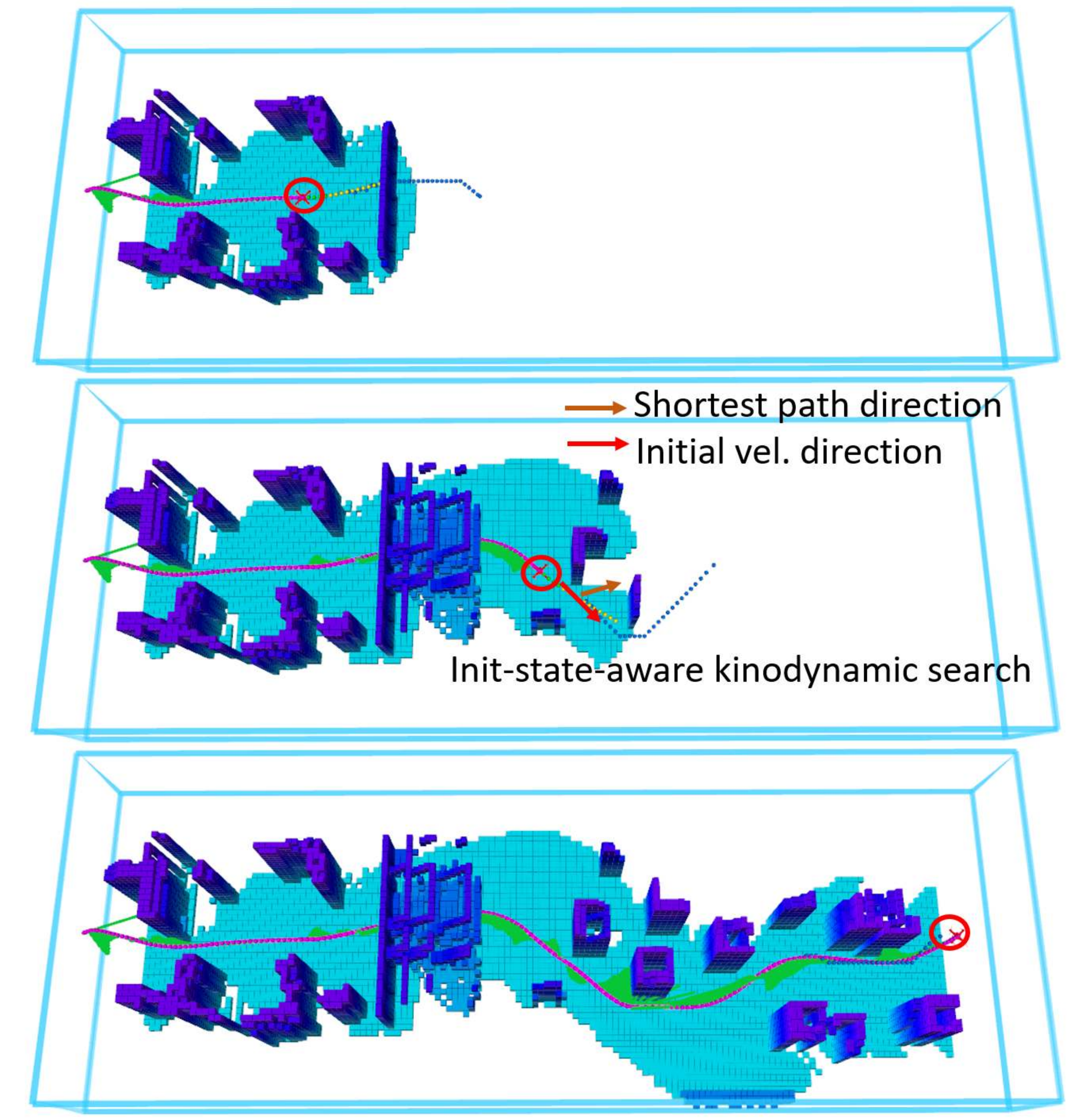}}
	\end{center}
	\caption{Illustration of different replanning methods in the same simulated environment. The trajectory is shown in \textit{purple}, and the acceleration profile is marked in \textit{green}. For the replanning cases where the shortest path direction is inconsistent with the initial state, we mark the initial velocity with a \textit{red arrow} and the shortest path direction in \textit{brown}.\label{fig:replan_tuple}}
\end{figure*}

Critical snapshots of the three methods are shown in Fig.~\ref{fig:replan_tuple}.
For SMP-U3 (Conservative) in Fig.~\ref{fig:replan_tuple_SMP}, when the quadrotor observes the 3-D step, it chooses to circle around instead of directly flying through the space below since the latter action requires a large acceleration range. This phenomenon demonstrates that SMP-U3 (Conservative) sacrifices maneuverability due to the restriction of the dynamic range. SMP-U3 takes the initial state into account, as shown in the middle snapshot in Fig.~\ref{fig:replan_tuple_SMP}.
The necessity of using the kinodynamic search instead of the position-only A* search is identified by the totally different maneuvers in Fig.~\ref{fig:replan_tuple_astar} and Fig.~\ref{fig:replan_tuple_our}. When the quadrotor enters the region of the pillars, it has two distinct choices, namely, ``pass on the left'' or ``pass on the right''. For the A*-EO method, as shown in the middle snapshot in Fig.~\ref{fig:replan_tuple_astar}, the quadrotor tends to choose the direction purely based on the shortest path. So there are cases where the quadrotor is passing in one direction and suddenly switches to the opposite direction due to finding a new shortest path, which results in an inconsistent and non-smooth replanning trajectory. Compared to A*-EO, the kinodynamic EBK search provides an initial-state-aware trajectory for the local reshaping, as shown in Fig.~\ref{fig:replan_tuple_our}. With the EBK search, the overall trajectory is clearly more natural and the replanning is more consistent.

As shown in Tab.~\ref{tab:simulation_replan}, SMP-U3 (Conservative) is efficient and has an average computation time of $0.010$ s. However, since the acceleration bound is conservative, the maneuverability is sacrificed and the trajectory duration is $33.3$ s, longer than the other two methods. Note that we use acceleration-controlled SMP~\cite{liu2017smp} with unconstrained QP~\cite{richter2016polyunqp} reparameterization. Since the unconstrained QP has no dynamical feasibility guarantee, the maximum acceleration is 1.3 $m/s^2$, which exceeds its designed maximum acceleration (1 $m/s^2$). Compared to SMP-U3 and A*-EO, our method has the lowest total jerk cost and the improvement is achieved by incorporating the kinodynamic search. The average velocity of our method is 1.37 $m/s$, slightly higher than the A*-EO, due to the fact that the kinodynamic search can reduce sharp decelerations. The maximum acceleration of our method is 3.18 $m/s^2$, which obeys the dynamical feasibility constraint.
Compared to the position-only A*-EO method, the jerk cost is reduced by $10\,\%$ by using the kinodynamic search. Considering that the advantages of the EBK search are outstanding for part of the trajectory where potential inconsistency exists, this quantitative improvement still faithfully identifies the gain of using the kinodynamic search. Note that navigating through this challenging environment with enough agility already requires considerable control efforts, and the $10\,\%$ cost reduction represents a reasonable overall improvement.

\begin{table}[t]
	\vspace{+0.5cm}
	\centering
	\caption{Performance of different replanning systems}
	\label{tab:simulation_replan}
	\resizebox{\columnwidth}{!}{
		\begin{tabular}{cccccccc} \toprule
			\multirow{3}{*}{Method} & \multicolumn{4}{c}{\textbf{Trajectory Statistics}} & \multicolumn{3}{c}{\textbf{Time Efficiency}(s)} \\
			\cline{2-8}
			&Traj.      & Jerk Cost  & Mean	Vel.	   & Max Acc.          & 	\multirow{2}{*}{Ave} & \multirow{2}{*}{Max} & \multirow{2}{*}{Std} \\
			&Dura.($s$) &  ($m^2/s^5$)  &  ($m/s$) & ($m/s^2$) &                       &                      &  \\
			\midrule
			\multicolumn{1}{c}{SMP-U3 (Conservative)} & 33.3 & 956.3 & 1.14  & \textbf{1.3}   & \textbf{0.010} & 0.027  & 0.004 \\
			\hline
			\multicolumn{1}{c}{A*-EO}              & 25.2 & 723.0 & 1.31  & 3.18  & 0.014 & 0.062  & 0.011 \\
			\hline
			\multicolumn{1}{c}{Our method}            & \textbf{24.4} & \textbf{648.1} & \textbf{1.37}  & 3.18  & 0.019 & 0.080  & 0.012  \\
			\bottomrule
		\end{tabular}
	}
	\vspace{-2.5cm}
\end{table}

\section{Experimental Results}\label{sec:experimental}
We conduct onboard experiments\footnote{\url{https://www.youtube.com/watch?v=sg46XT9-o1k}} with the two vision-based testbeds to show the general applicability of the proposed framework. For onboard testing, the parameters are as follows: the time step $\Delta_t$ is set to $0.35\,s$; the maximum velocity and maximum acceleration are set to 1.2 $m/s$ and 2.0 $m/s^2$, respectively;~\footnote{The speed limit and acceleration limit are set to be slightly conservative considering the perception delay in onboard experiments.}
and the local planning range is set to $10 m \times 6 m \times 1.1 m$. (The corresponding grid size is $55 \times 35 \times 6$.)

\subsection{Indoor Replanning Performance}
\subsubsection{Monocular-vision-based indoor navigation}
As shown in Fig.~\ref{fig:onboard_replan_mono}, our replanning system works in complex 3-D environments with only a local map. The quadrotor is commanded to navigate to a 3-D position where the environment is previously unknown. The whole trajectory and the final accumulated map is shown in Fig.~\ref{fig:onboard_replan_global}. The trajectory length of the final trajectory is $18.6\,m$ and total trajectory execution time is 43.4 $s$. The average velocity of the quadrotor is 0.45 $m/s$ with a maximum velocity of 0.79 $m/s$. The maximum acceleration of the trajectory is 0.58 $m/s^2$, the whole trajectory is dynamically feasible, and there are a total $125$ calls of the EBK search (active mode), with an average computation time of 0.010 $s$. There are $105$ calls of EO. The average computation times of the elastic tube expansion and the optimization are 0.001 $s$ and 0.031 $s$, respectively.
\begin{figure}
	\begin{center}
		\subfigure{\includegraphics[trim={0cm 0cm 0cm 0cm},clip, width=0.21\textwidth]{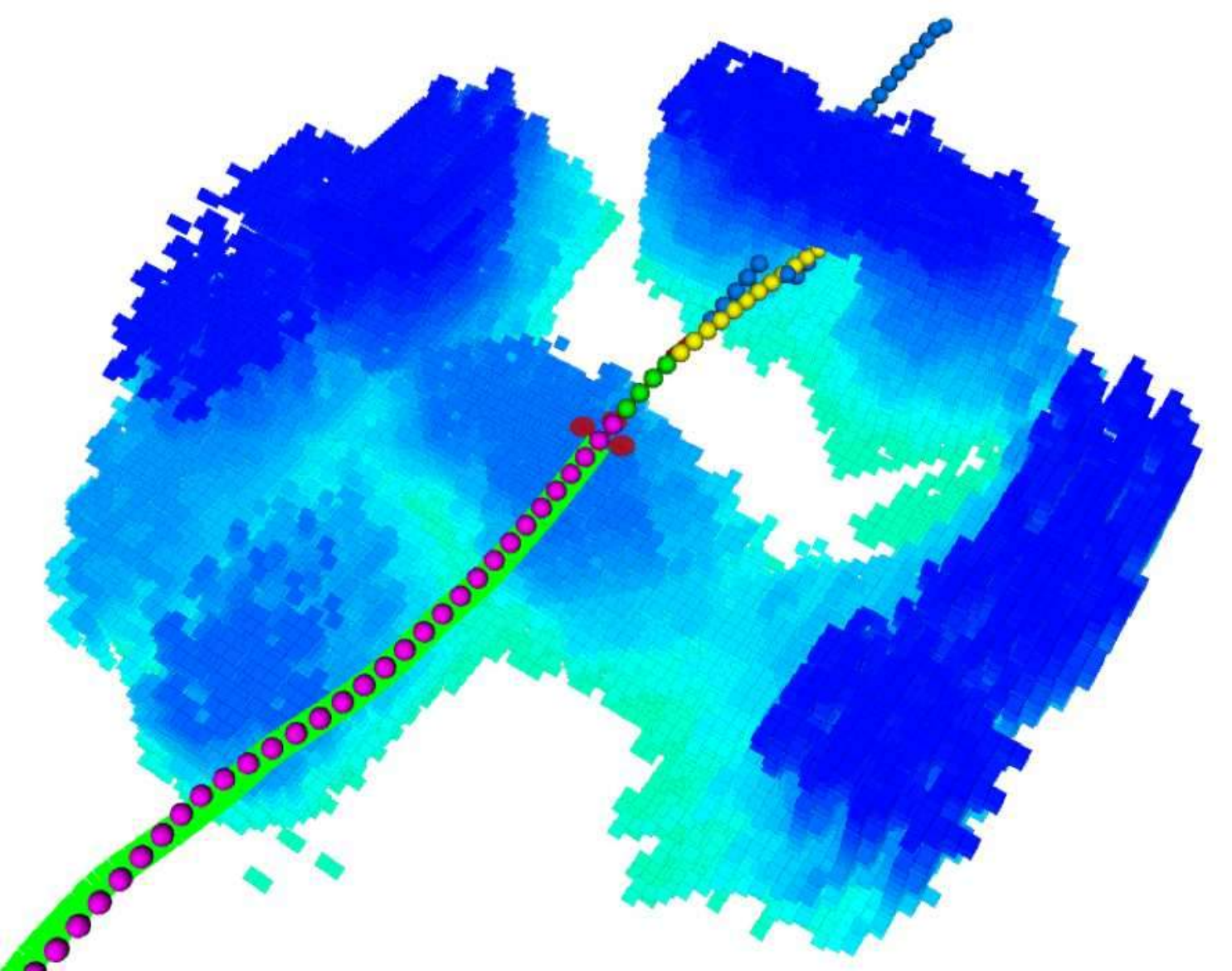}}
		\subfigure{\includegraphics[trim={5cm 0cm 0cm 0cm},clip, width=0.23\textwidth]{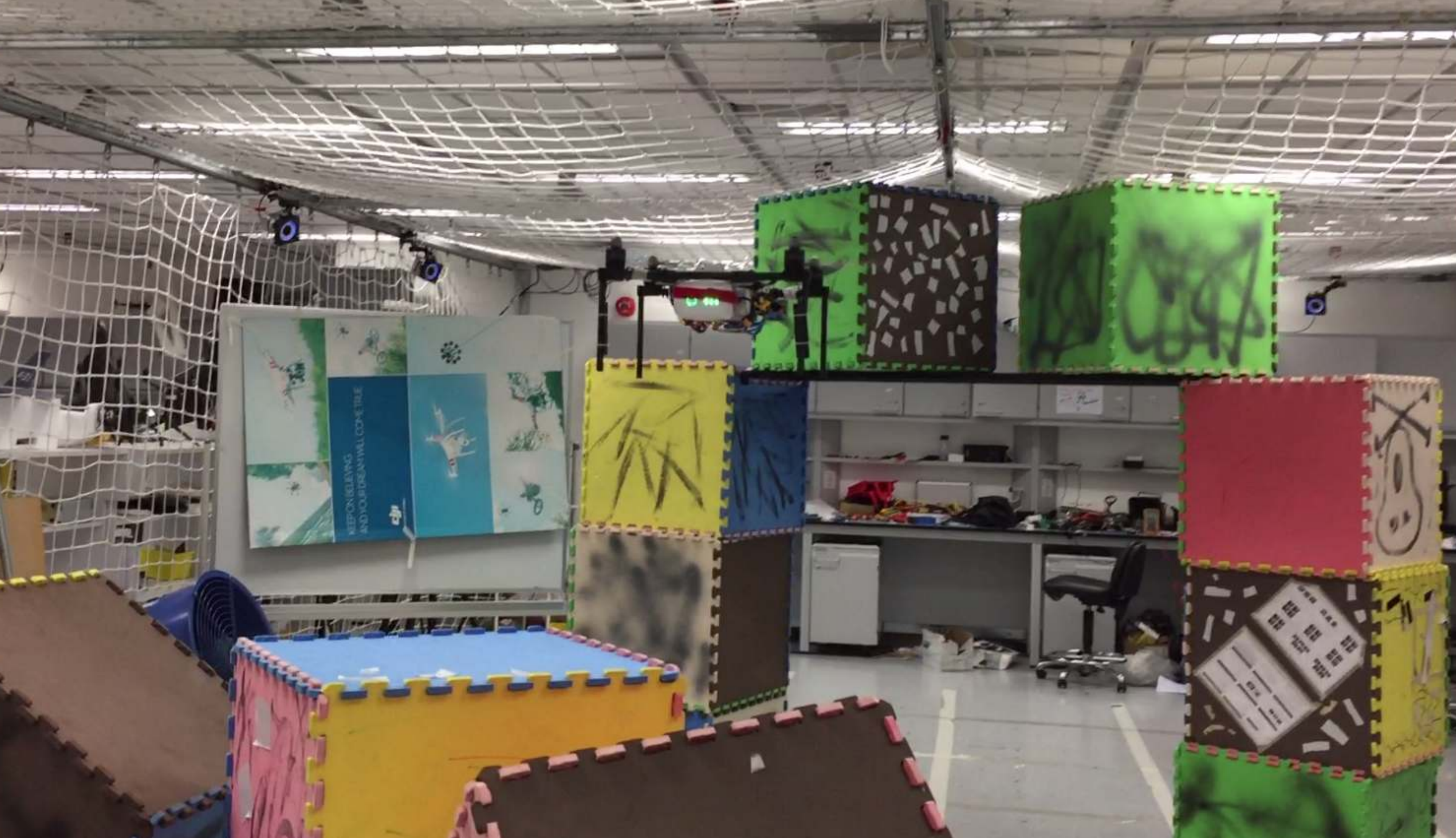}}
	\end{center}
	\caption{Illustration of the snapshot of the indoor replanning with the monocular vision. In (a), the control points found by EBK (\textit{blue}), executed control points (\textit{pink}), committed control points (\textit{green}), and control points under optimization (\textit{yellow}) are marked. The corresponding indoor environment is shown in (b).\label{fig:onboard_replan_mono}}
	\vspace{-0.3cm}
\end{figure}

\begin{figure}
	\centering
	\includegraphics[width=0.32\textwidth]{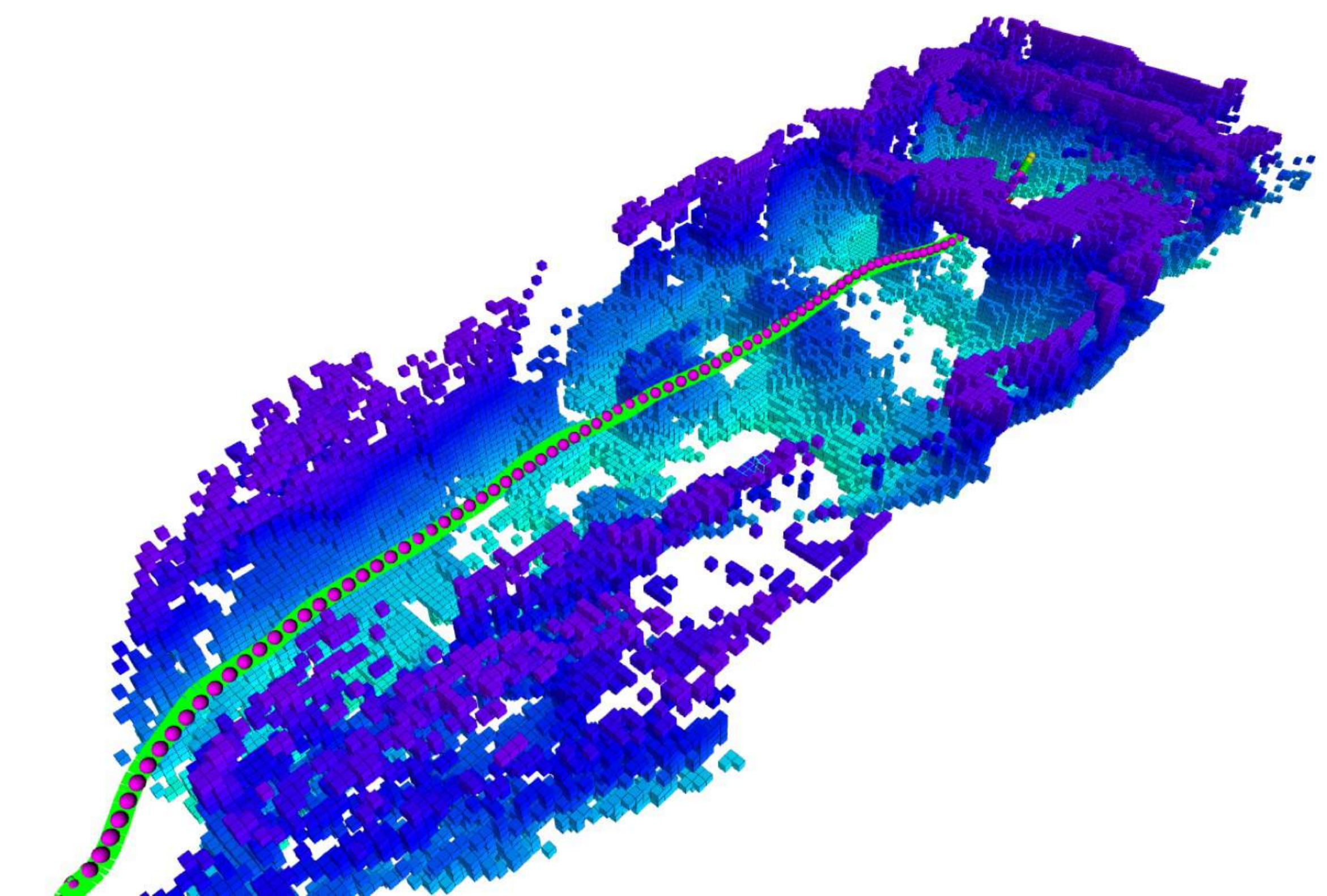}
	\caption{Illustration of the whole trajectory and final accumulated map for the indoor replanning using the monocular perception system.}\label{fig:onboard_replan_global}
	\vspace{-0.3cm}
\end{figure}

\subsubsection{Dual-fisheye-based indoor round-trip navigation}
\begin{figure}
	\begin{center}
		\subfigure{\includegraphics[trim={0cm 0cm 0cm 0cm},clip, width=0.21\textwidth]{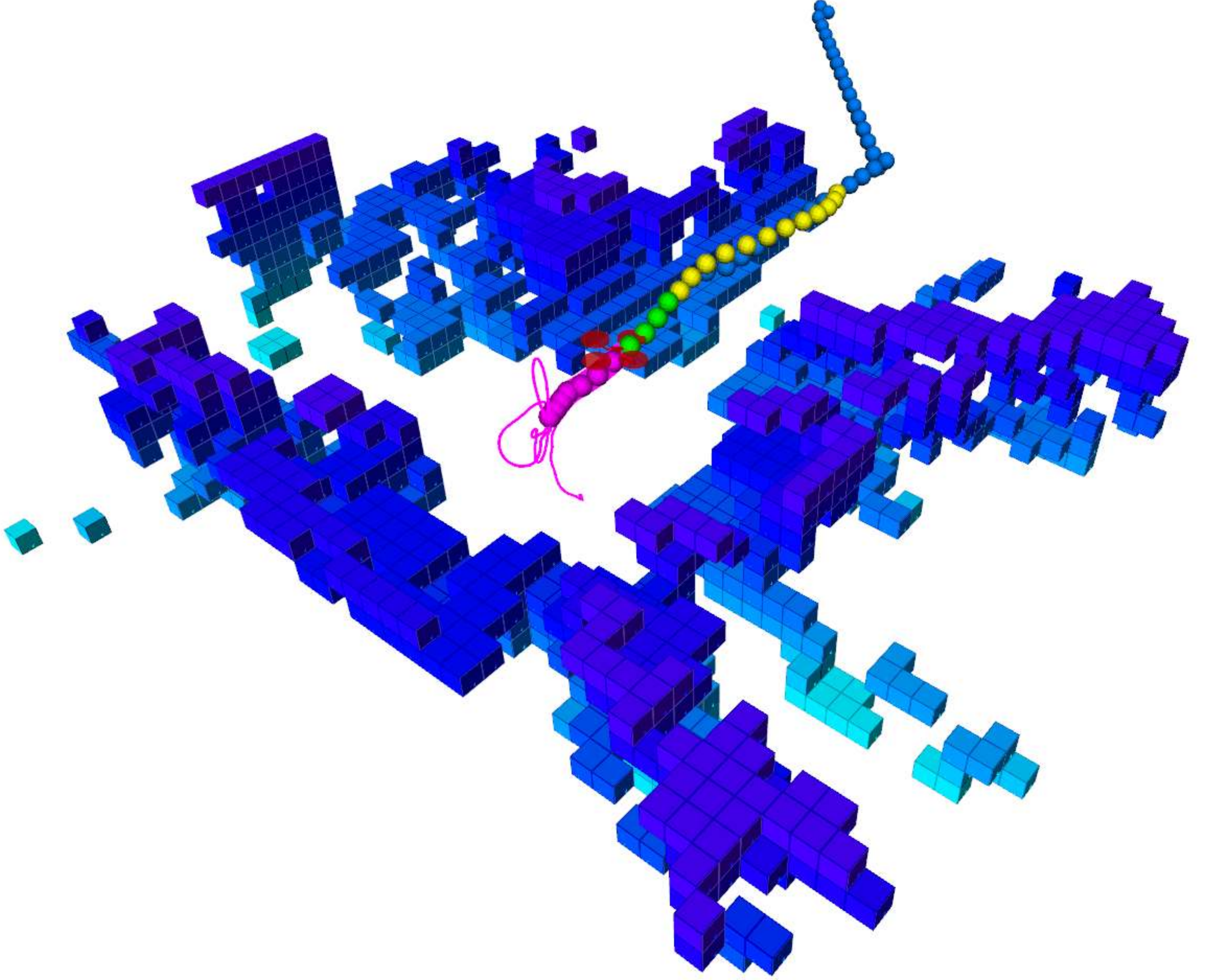}}
		\subfigure{\includegraphics[trim={0cm 0cm 0cm 0cm},clip, width=0.23\textwidth]{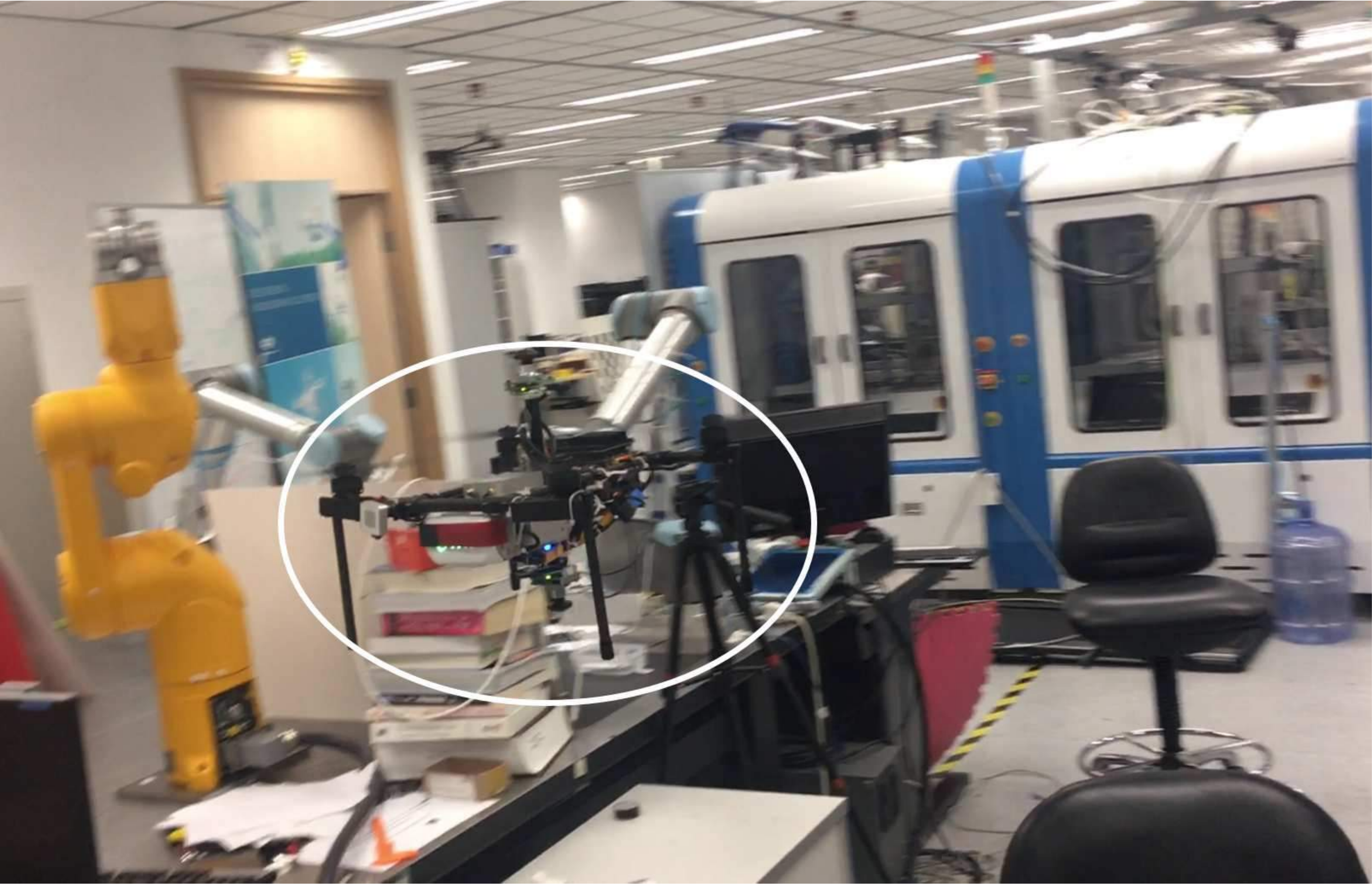}}
	\end{center}
	\caption{Illustration of the snapshot of the indoor replanning with the omnidirectional vision. With a fixed yaw angle, the obstacles around the quadrotor can be mapped.\label{fig:onboard_replan_omi}}
\end{figure}

\begin{figure}
	\centering
	\includegraphics[width=0.47\textwidth]{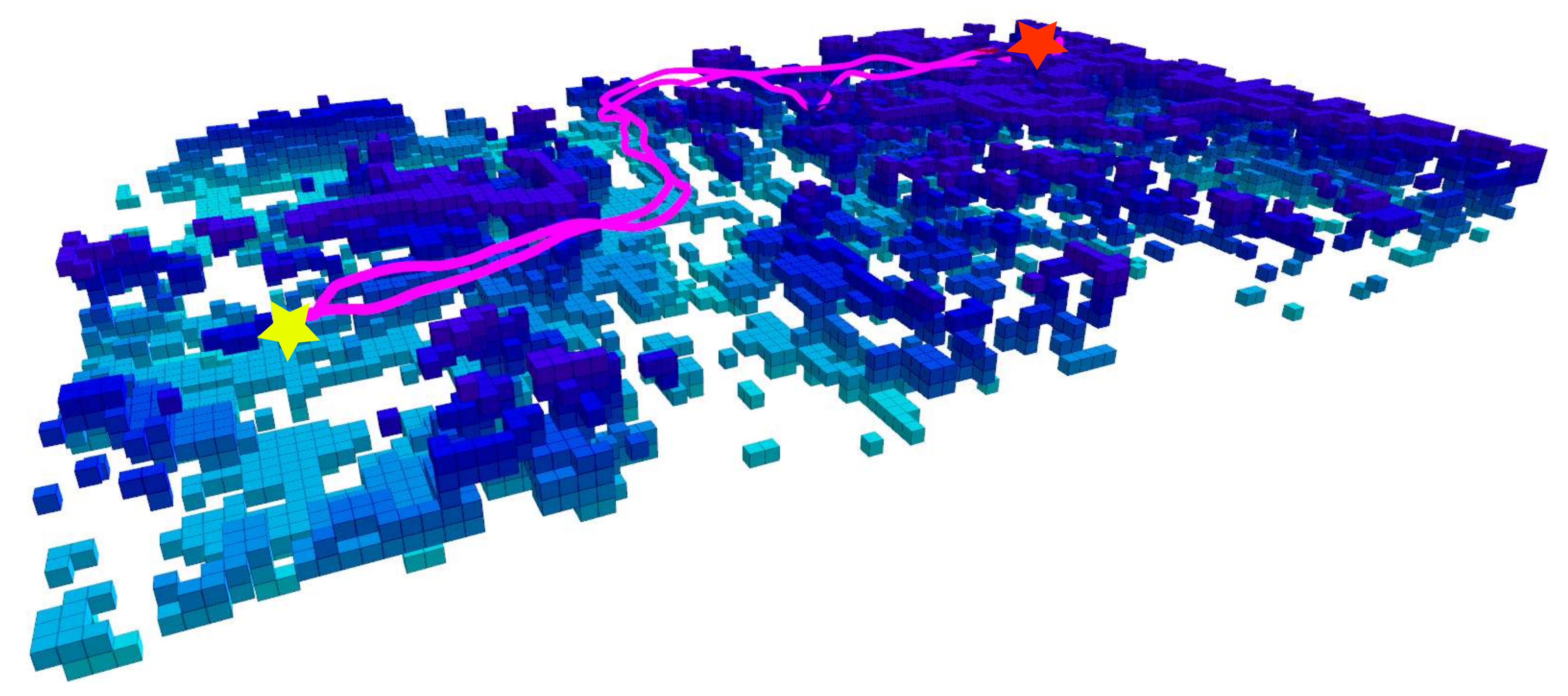}
	\caption{Illustration of the whole trajectory for the replanning using the dual-fisheye omnidirectional perception system. Unlike the experiments using the monocular testbed, we can achieve round-trip flight in an unknown complex indoor environment.}\label{fig:omi_indoor}
\vspace{-0cm}
\end{figure}
As shown in Fig.~\ref{fig:onboard_replan_omi} and Fig.~\ref{fig:omi_indoor}, with omnidirectional perception, our quadrotor testbed is able to fly a round-trip without controlling the yaw angle. Online mapping using the dual-fisheye cameras is challenging due to the high distortion of the images acquired from the fisheye cameras. Although the uncertainty of the map is larger than the monocular case, our replanning system is still able to robustly avoid the unexpected obstacles and navigate in the unstructured cluttered environment.
\subsection{Outdoor Replanning Performance}
As shown in Fig.~\ref{fig:onboard_plan_outdoor}, we demonstrate outdoor experiments using the monocular vision testbed. For Fig.~\ref{fig:outdoor_pavilion}, the trajectory length is $19.6\,m$ and total execution time is $41.3\,s$. The average velocity of the quadrotor is $0.49\,m/s$. The maximum acceleration of the trajectory is $1.06\,m/s^2$, and the whole trajectory is dynamically feasible. There are a total $35$ calls of EBK search (passive mode) with an average computation time of $0.025\,s$. The average computation times of the elastic tube expansion and optimization are $0.001\,s$ and $0.030\,s$, respectively. For the experiment shown in Fig.~\ref{fig:outdoor_trees}, the performance and trajectory statistics are similar.
\begin{figure}[t]
	\begin{center}
	\vspace{-0.2cm}
		\subfigure[\label{fig:outdoor_pavilion}]{\includegraphics[trim={0cm 0cm 0cm 0cm},clip,width=0.2\textwidth]{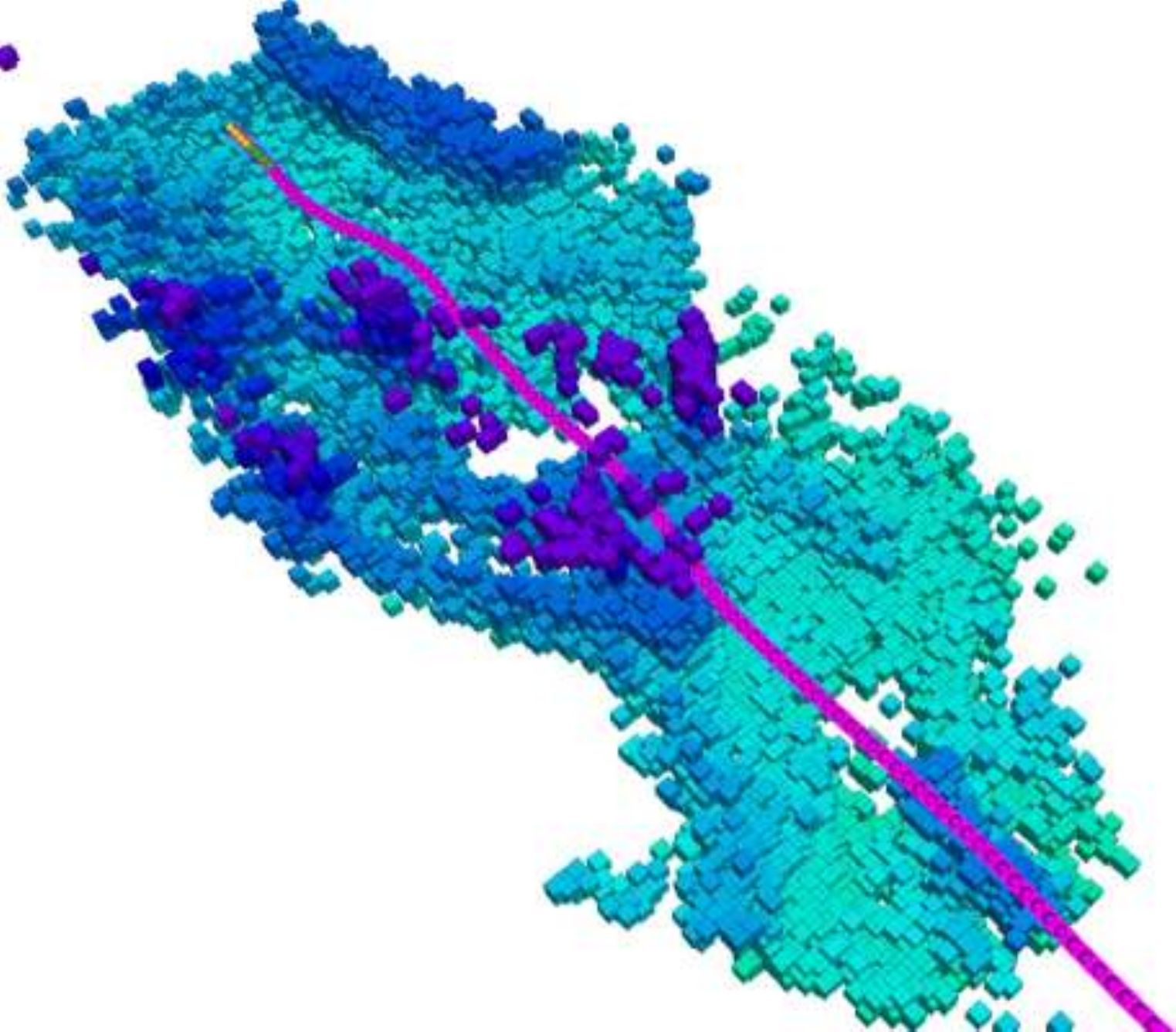}}
		\subfigure[\label{fig:outdoor_trees}]{\includegraphics[trim={0cm 0cm 0cm 0cm},clip,width=0.2\textwidth]{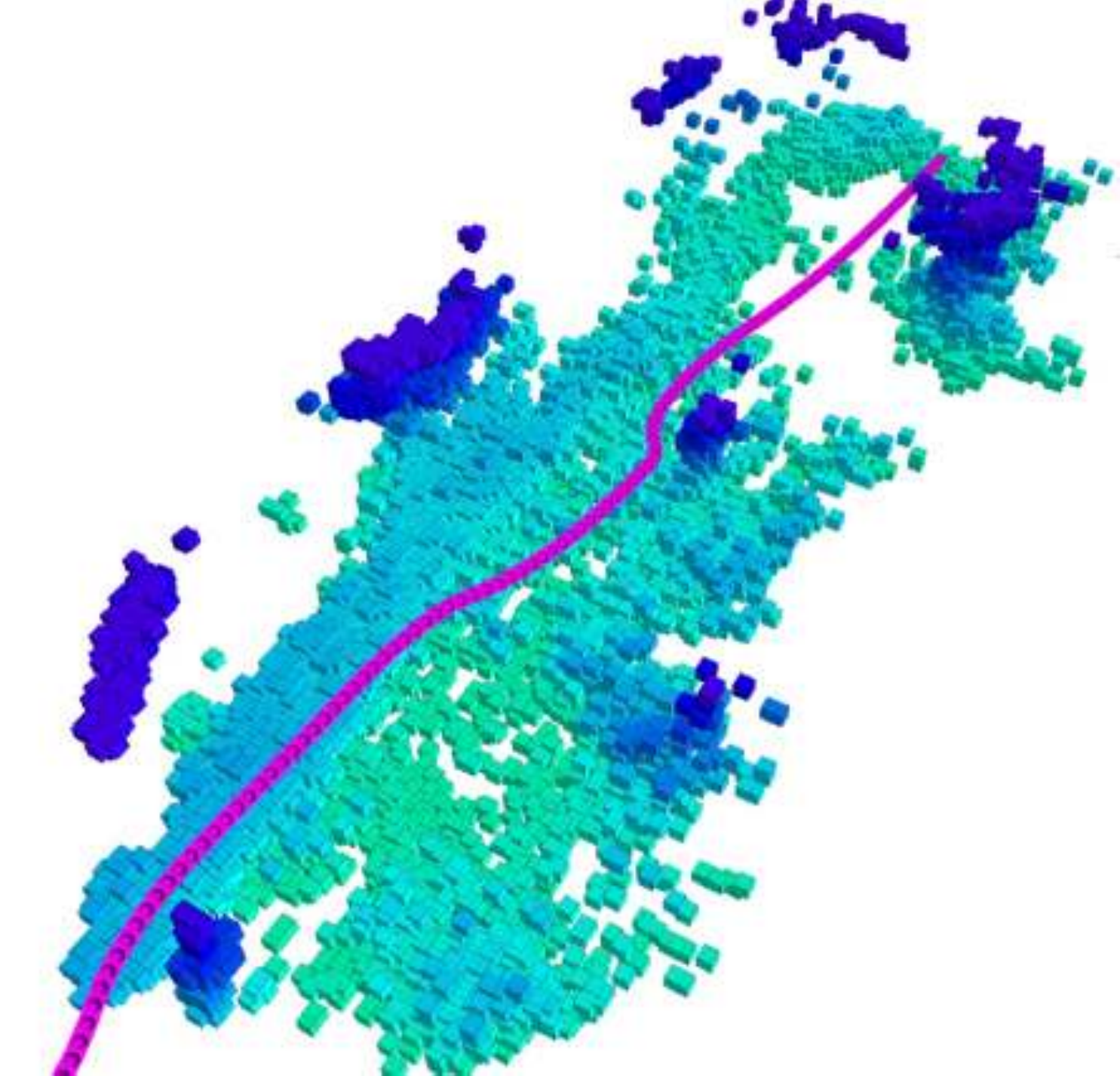}}
	\end{center}
	\caption{Illustration of the outdoor experiments using the monocular vision: A flight through a pavilion is shown in (a) (part of the map on the top is cut for visualization purposes); a flight avoiding trees is shown in (b).\label{fig:onboard_plan_outdoor}}
\end{figure}

\section{Conclusion and Future Work}\label{sec:conclusion}
In this paper, we present an efficient kinodynamic replanning framework by exploiting the advantageous properties of the B-spline.
The proposed EBK search algorithm is flexible and provides a user-specified parameter which can be used to control the algorithm efficiency and solution quality. The problem of the B-spline-based kinodynamic search on a spatial grid is characterized in detail, and the theoretical performance of the EBK search is analyzed. To compensate for the discretization, we propose an elastic optimization process. We combine the two components into a receding horizon framework.
Detailed analysis and comprehensive experiments are carried out to validate the performance. Systematic comparisons against the state-of-the-art are provided to verify the claims. The replanning framework is efficient and complete, and can be used in various kinds of exploration tasks and different kinds of quadrotor testbeds. The current limitation of the framework lies in the fact that for EBK search we are using the uniform B-spline on the spatial grid, which will result in limited B-spline patterns. In the future, we expect to explore NURBS for kinodynamic search, which will allow for various motion patterns in the kinodynamic search.

\appendix
\subsection{Proof of Prop.~\ref{prop:linear_derivative_bound}}\label{sec:appendix_prop_bound}
\noindent The correctness of the proposition follows from the fact that the derivative of the B-spline of degree $k$ is another B-spline of degree $k-1$, which enjoys the convex hull property. Specifically, for the $l$-th derivative, let $\mathbf{C}_l$ map the basis $\mathbf{b}$ to the derivatives; i.e., $\frac{d\mathbf{b}}{d^l u} = \mathbf{C}_l \mathbf{b}$. It follows that
\begin{equation}
	\frac{d \mathbf{s}_j(u)}{d^{l}u} = \frac{1}{ {(\Delta_t)}^l} \frac{d \mathbf{b}^{\intercal}}{d^l u} \mathbf{M}_k \mathbf{P}_j = \frac{1}{ {(\Delta t)}^l} \mathbf{b}^{\intercal} \mathbf{C}^{\intercal}_l \mathbf{M}_k \mathbf{P}_j.
\end{equation}

Plug in $\mathbf{S}_l = \mathbf{M}_k^{-1} \mathbf{C}_l \mathbf{M}_k /(\Delta_t) ^{l}$. It follows that
\begin{equation}
	\frac{d \mathbf{s}_j(u)}{d^{l}u} = \mathbf{b}^{\intercal} \mathbf{M}_k (\mathbf{S}_j \mathbf{P}_j) ,
\end{equation}
where $\mathbf{S}_j \mathbf{P}_j$ is the control point span of the derivative spline. By applying the convex hull property, we complete the proof.

\subsection{Relationship between $\mathcal{G}$ and $\mathcal{G}_H$}\label{sec:appendix_graph_relation}
\begin{lemma}\label{lemma:lift_uniqueness1}
Given the initial state $\pi_s$ and goal state $\pi_g$, let $\pi=(v_0, v_1, \ldots, v_T)$ be an admissible control point placement of Problem~\ref{prob:optimal_search}. The extended sequence $\tilde{\pi}$ uniquely corresponds to an admissible path $\Phi = (\hat{v}_0, \hat{v}_1, \ldots, \hat{v}_Q)$ on the graph $\mathcal{G}_H$, with $\hat{v}_0 = \pi_s$, $\hat{v}_Q = \pi_g$ and $Q=K+T+2$.
\end{lemma}

\begin{lemma}\label{lemma:lift_uniqueness2}
Any admissible path $\Phi = (\hat{v}_0, \hat{v}_1, \ldots, \hat{v}_Q)$ on the graph $\mathcal{G}_H$ which satisfies $\hat{v}_0=\pi_s$, $\hat{v}_Q=\pi_g$ and $(\hat{v}_{j-1},\hat{v}_j)\in E_H$ for $j=1,\ldots,Q$ uniquely corresponds to an admissible control point placement $\pi$ of Problem~\ref{prob:optimal_search}.
\end{lemma}

\begin{prop}\label{prop:graph_transformation}
Given a strictly positive cost function $f_{k,\Delta_t}:{[\tilde{\pi}]}^k \to \mathbb{R}_{+}$, Problem~\ref{prob:optimal_search} is equivalent to the shortest path problem on the graph $\mathcal{G}_H$, where the cost is defined on the vertices according to the function $f_{k,\Delta_t}(\cdot)$.
\end{prop}

\subsection{Characterization of the inflation for the B-spline-based kinodynamic search}\label{sec:appendix_prop_collifree}
We denote by $\mathcal{C}^{\text{BK}}$ the configuration space in which we conduct the B-spline-based kinodynamic search. The configuration space $\mathcal{C}^{\text{BK}}$ is generated by inflating all the obstacles by $\delta^{\text{BK}}$ in the workspace. We take a 26-connected 3-D grid with fixed cell size $d_\text{x} \times d_\text{y} \times d_\text{z}$ and a fifth-degree B-spline as an example. There is a finite number of possible span patterns (${27}^5$ in total). The minimum clearance $c_{\min}^{\text{BK}} $ of the configuration space $\mathcal{C}^{\text{BK}}$ can be expressed by the cell size and obstacle inflation $\delta_{\text{BK}} $ according to $c_{\min}^{\text{BK}} = \min \left( d_\text{x}/2 + \delta^{\text{BK}}, d_\text{y}/2 + \delta^{\text{BK}}, d_\text{z}/2 + \delta^{\text{BK}} \right) $. The problem of finding the sufficient condition such that the overall trajectory is collision free is equivalent to finding the minimum inflation $\delta^{\text{BK}}$ such that the trajectories for all the B-spline patterns are completely bounded inside the inflated cells. Since the total number of patterns is finite, $\delta^{\text{BK}}$ can be found by enumerating all the possible span patterns and picking out the one with the largest deviation. The script is available.\footnote{The corresponding script can be found at \url{https://github.com/WenchaoDing/kinodynamic_replanning.git}.} Note that the process of finding the sufficient inflation is one-time work prior to the planning process. Therefore, it does not affect the EBK search efficiency. Typically, for a 26-connected 3-D grid with fixed cell size $0.16\,m\times 0.16\,m \times 0.16\,m$, the inflation needed is less than $0.03\,m$, which is easy to satisfy in practice.

\subsection{Performance analysis of the EBK search}\label{sec:appendix_resolution_complete}
To analyze the performance of the EBK search, we show than the modified \textproc{Index} function in Algo.~\ref{algo:index_efficient} induces another search graph. The characterization of the search graph unveils the complexity of the EBK search. In the following, we give a formal definition of the search graph given by the modified \textproc{Index} function.

We begin with the definition of the nodes which are called \textit{virtual nodes}. Specifically, given the encoding level $d$ and encoding index $e$, we denote by $\mathcal{H}^d_e \coloneqq  \{\hat{v}\in V_H | \Call{Index}{\hat{v},d} = e\}$ the virtual node aggregating all the vertex tuples which share the same encoding, i.e., the same last $d$ coordinates.
It follows that each vertex tuple $\hat{v}\in V_H$ belongs to exactly one virtual node due to the uniqueness induced by the function \textproc{UniqueEncode}$(\cdot)$.
For each vertex tuple $\hat{v}_i \in \mathcal{H}^d_{e_i}$, we can obtain the set of neighboring virtual nodes $\{\mathcal{H}^d_{e_j} | e_j =\Call{Index}{\hat{v}_j, d}, \forall(\hat{v}_i, \hat{v}_j)\in E_H\}$.
The interesting part is that every vertex tuple of $\mathcal{H}^d_{e_i}$ has exactly the same set of neighboring virtual nodes.
We denote by $\mathcal{E}$ the set of edges between virtual nodes, and we denote by $\mathcal{G}_{D}=(\mathcal{H}^d, \mathcal{E})$ the \textit{virtual graph} formed by the virtual nodes, where $\mathcal{H}^d$ denotes the set of all virtual nodes.

There are several unique transformations between~$\mathcal{G}_{D}$ and~$\mathcal{G}_{H}$. Given the initial state $\pi_s\in\mathcal{H}^d_{e_s}$ and the goal state $\pi_g \in \mathcal{H}^d_{e_g}$, an \textit{admissible} path $\Theta=(\mathcal{H}^d_{e_0}, \mathcal{H}^d_{e_1}, \ldots, \mathcal{H}^d_{e_Q})$ on the virtual graph $\mathcal{G}_{D}$ should satisfy $\mathcal{H}^d_{e_0} = \mathcal{H}^d_{e_s}$, $\mathcal{H}^d_{e_Q} = \mathcal{H}^d_{e_g}$ and $(\mathcal{H}^d_{e_{j-1}},\mathcal{H}^d_{e_j}) \in \mathcal{E}$ for $j=1,\ldots,Q$. By the construction, any admissible path $\Phi$ on $\mathcal{G}_H$ uniquely corresponds to an admissible path $\Theta$ on $\mathcal{G}_{D}$ due to the uniqueness of encoding. Given a known initial state $\pi_s$, any admissible path $\Theta$ on $\mathcal{G}_{D}$ also uniquely corresponds to an admissible path $\Phi$ on $\mathcal{G}_H$.\footnote{Without $\pi_s$, the uniqueness no longer holds.} The reason is that the encoding is based on the last $d > 1$ coordinates, and with the known initial virtual node, the original vertex tuple can be reconstructed.

We define the cost to the virtual node to be $c[\mathcal{H}^d_e] = \min \{c[\pi_s, \hat{v}] | \forall \hat{v} \in \mathcal{H}^d_e \}$, where $c[\pi_s, \hat{v}]$ is the minimum cost from the start vertex tuple $\pi_s$ to $\hat{v}$ according to Eq.~\ref{eq:node_cost}. The EBK search cannot reach the exact goal state $\pi_g$, and instead it can only reach the virtual node $\mathcal{H}^d_{e_g}$. In other words, the EBK search can only reach the  relaxed goal state. Given the start and goal virtual nodes, the EBK search is complete (finds the optimal admissible path if one exists) with respect to the aggregated graph $\mathcal{G}_D$, as stated below:
\begin{theorem}\label{thm:resolution_complete}
	Given the graph $\mathcal{G}_D = (\mathcal{H}^d, \mathcal{E})$, the start virtual node $\mathcal{H}^d_{e_s}\in \mathcal{H}^d$ and the goal virtual node $\mathcal{H}^d_{e_g}\in \mathcal{H}^d$, the EBK search finds the optimal admissible path $\Theta=(\mathcal{H}^d_{e_s}, \mathcal{H}^d_{e_1}, \ldots, \mathcal{H}^d_{e_g})$ on $\mathcal{G}_D$ if one exists.
\end{theorem}
\begin{proof}
We prove by induction and contradiction. The proof follows a similar reasoning process to proving the correctness of Dijkstra's algorithm~\cite{dijkstra1959note}. And the difference is that for one virtual node, there is one vertex tuple picked out to associate with the virtual node, and the association will be updated during the search process. For the expanded virtual node, the association will be fixed and is supposed to yield the minimum cost to the virtual node among all the aggregated vertex tuples.

Suppose the following hypothesis holds: for each expanded virtual node $\mathcal{H}^d_{v}$, $c[\mathcal{H}^d_{v}]$ is the lowest cost from the source virtual node to $\mathcal{H}^d_{v}$; and for unexpanded virtual node $\mathcal{H}^d_{u}$, $c[\mathcal{H}^d_{u}]$ is the lowest cost from the source virtual node to $\mathcal{H}^d_{u}$ via expanded virtual nodes only. The base case is that there is just the initial virtual node, and the hypothesis holds obviously.

Assume there are $n-1$ expanded virtual nodes, and the hypothesis holds, in which case, the lowest costs to the $n-1$ virtual nodes are known, and given the source virtual node, the vertex tuple associations for the $n-1$ virtual nodes are fixed accordingly. We choose $\mathcal{H}^d_{v}$ from the $n-1$ nodes such that it has the least $c[\mathcal{H}^d_{u}] = c[\mathcal{H}^d_{v}] + f[\mathcal{H}^d_{u}]$ while satisfying $(\mathcal{H}^d_{v},\mathcal{H}^d_{u})\in \mathcal{E}_d$, where $f[\mathcal{H}^d_{u}]$ is the cost of the virtual node.

First, $c[\mathcal{H}^d_{u}]$ should be the shortest path from the source node to $\mathcal{H}^d_{u}$. Since if there is a shorter path reaching $\mathcal{H}^d_{u}$ via the nodes other than the $n-1$ expanded nodes, and $\mathcal{H}^d_{w}$ is the first unexpanded node on that path, it follows that $c[\mathcal{H}^d_{w}] > c[\mathcal{H}^d_{u}]$, which yields a contradiction.

Second, for any of the remaining unvisited nodes $\mathcal{H}^d_{w}$, $c[\mathcal{H}^d_{w}]$ should still be the shortest path via the expanded nodes. Since if a lower cost of $c[\mathcal{H}^d_{w}]$ is found by adding $\mathcal{H}^d_{u}$ to the expanded nodes, Line 17 of Algo.~\ref{algo:optimal_search} will have updated it. Note that the update of $c[\mathcal{H}^d_{w}]$ will update the association with $\mathcal{H}^d_{w}$, so the cost and association are consistent.
\end{proof}

\subsection{Proof of Theorem~\ref{thm:finite_safety}}\label{sec:appendix_thm_safety}
The correctness of the theorem follows from the convex hull property of the B-spline. For brevity, we only consider one control point span which consists of $k+1$ control points, but can be generalized to a long control point sequence without any difficulty. The original tube of one control point span consists of $k+1$ balls, and the connectivity is already guaranteed by the two-level inflation scheme. As such, there are $k$ intersection areas for the sequence of $k+1$ control points. Note that here we only consider the intersection between the two balls associated with two neighboring control points. In the extreme case, we add $k$ control points to each of the intersection areas, and the overall control point sequence consists of $k+1+k^2$ control points, i.e., $k^2+1$ control point spans. By applying the convex hull property to each of the spans, the trajectory for each control point span is bounded inside one of the balls. As such, the iterative process can succeed in $k^2$ iterations if no violation of the dynamical feasibility is reported.

\bibliography{paper}

\begin{thebibliography}{10}
\providecommand{\url}[1]{#1}
\csname url@rmstyle\endcsname
\providecommand{\newblock}{\relax}
\providecommand{\bibinfo}[2]{#2}
\providecommand\BIBentrySTDinterwordspacing{\spaceskip=0pt\relax}
\providecommand\BIBentryALTinterwordstretchfactor{4}
\providecommand\BIBentryALTinterwordspacing{\spaceskip=\fontdimen2\font plus
\BIBentryALTinterwordstretchfactor\fontdimen3\font minus
  \fontdimen4\font\relax}
\providecommand\BIBforeignlanguage[2]{{%
\expandafter\ifx\csname l@#1\endcsname\relax
\typeout{** WARNING: IEEEtran.bst: No hyphenation pattern has been}%
\typeout{** loaded for the language `#1'. Using the pattern for}%
\typeout{** the default language instead.}%
\else
\language=\csname l@#1\endcsname
\fi
#2}}

\bibitem{mellinger2011minsnap}
D.~Mellinger and V.~Kumar, ``Minimum snap trajectory generation and control for
  quadrotors,'' in \emph{Proc. of the {IEEE} Intl. Conf. on Robot. and Autom.},
  2011, pp. 2520--2525.

\bibitem{gao2016online}
F.~Gao and S.~Shen, ``Online quadrotor trajectory generation and autonomous
  navigation on point clouds,'' in \emph{IEEE International Symposium on
  Safety, Security, and Rescue Robotics (SSRR)}, 2016, pp. 139--146.

\bibitem{richter2016polyunqp}
C.~Richter, A.~Bry, and N.~Roy, ``Polynomial trajectory planning for aggressive
  quadrotor flight in dense indoor environments,'' in \emph{Intl. J. Robot.
  Research}.\hskip 1em plus 0.5em minus 0.4em\relax Springer, 2016, pp.
  649--666.

\bibitem{liu2017sfc}
S.~Liu, M.~Watterson, K.~Mohta, K.~Sun, S.~Bhattacharya, C.~J. Taylor, and
  V.~Kumar, ``Planning dynamically feasible trajectories for quadrotors using
  safe flight corridors in 3-{D} complex environments,'' \emph{IEEE Robotics
  and Automation Letters}, vol.~2, 2017.

\bibitem{chen2016online}
J.~Chen, T.~Liu, and S.~Shen, ``Online generation of collision-free
  trajectories for quadrotor flight in unknown cluttered environments,'' in
  \emph{Proc. of the {IEEE} Intl. Conf. on Robot. and Autom.}\hskip 1em plus
  0.5em minus 0.4em\relax IEEE, 2016, pp. 1476--1483.

\bibitem{lavalle2001randomized}
S.~M. LaValle and J.~J. Kuffner~Jr, ``Randomized kinodynamic planning,''
  \emph{Intl. J. Robot. Research}, vol.~20, no.~5, pp. 378--400, 2001.

\bibitem{webb2013kinodynamic}
D.~J. Webb and J.~van~den Berg, ``Kinodynamic {RRT}*: Asymptotically optimal
  motion planning for robots with linear dynamics,'' in \emph{Proc. of the
  {IEEE} Intl. Conf. on Robot. and Autom.}, 2013, pp. 5054--5061.

\bibitem{gammell2015bit}
J.~D. Gammell, S.~S. Srinivasa, and T.~D. Barfoot, ``Batch informed trees
  (bit*): Sampling-based optimal planning via the heuristically guided search
  of implicit random geometric graphs,'' in \emph{Proc. of the {IEEE} Intl.
  Conf. on Robot. and Autom.}, 2015, pp. 3067--3074.

\bibitem{karaman2011optimal}
S.~Karaman and E.~Frazzoli, ``Sampling-based algorithms for optimal motion
  planning,'' \emph{Intl. J. Robot. Research}, pp. 846--894, 2011.

\bibitem{janson2015fast}
L.~Janson, E.~Schmerling, A.~Clark, and M.~Pavone, ``Fast marching tree: A fast
  marching sampling-based method for optimal motion planning in many
  dimensions,'' \emph{Intl. J. Robot. Research}, pp. 883--921, 2015.

\bibitem{kuwata2008motion}
Y.~Kuwata, J.~Teo, S.~Karaman, G.~Fiore, E.~Frazzoli, and J.~How, ``Motion
  planning in complex environments using closed-loop prediction,'' in
  \emph{AIAA Guidance, Navigation and Control Conference and Exhibit}, 2008, p.
  7166.

\bibitem{xie2015kinobit}
C.~Xie, J.~van~den Berg, S.~Patil, and P.~Abbeel, ``Toward asymptotically
  optimal motion planning for kinodynamic systems using a two-point boundary
  value problem solver,'' in \emph{Proc. of the {IEEE} Intl. Conf. on Robot.
  and Autom.}, 2015.

\bibitem{li2016sst}
Y.~Li, Z.~Littlefield, and K.~E. Bekris, ``Asymptotically optimal
  sampling-based kinodynamic planning,'' \emph{Intl. J. Robot. Research},
  vol.~35, no.~5, pp. 528--564, 2016.

\bibitem{liu2017smp}
S.~Liu, N.~Atanasov, K.~Mohta, and V.~Kumar, ``Search-based motion planning for
  quadrotors using linear quadratic minimum time control,'' in \emph{Proc. of
  the {IEEE/RSJ} Intl. Conf. on Intell. Robots and Syst.}, 2017.

\bibitem{ding18replanning}
W.~Ding, W.~Gao, K.~Wang, and S.~Shen, ``Trajectory replanning for quadrotors
  using kinodynamic search and elastic optimization,'' in \emph{Proc. of the
  {IEEE} Intl. Conf. on Robot. and Autom.}\hskip 1em plus 0.5em minus
  0.4em\relax IEEE, 2018, pp. 7595--7602.

\bibitem{van2012lqg}
J.~Van Den~Berg, D.~Wilkie, S.~J. Guy, M.~Niethammer, and D.~Manocha,
  ``{LQG}-obstacles: Feedback control with collision avoidance for mobile
  robots with motion and sensing uncertainty,'' in \emph{Proc. of the {IEEE}
  Intl. Conf. on Robot. and Autom.}\hskip 1em plus 0.5em minus 0.4em\relax
  IEEE, 2012, pp. 346--353.

\bibitem{zhou2014vector}
D.~Zhou and M.~Schwager, ``Vector field following for quadrotors using
  differential flatness,'' in \emph{Proc. of the {IEEE} Intl. Conf. on Robot.
  and Autom.}\hskip 1em plus 0.5em minus 0.4em\relax IEEE, 2014, pp.
  6567--6572.

\bibitem{bareiss2015stochastic}
D.~Bareiss, J.~Van Den~Berg, and K.~K. Leang, ``Stochastic automatic collision
  avoidance for tele-operated unmanned aerial vehicles,'' in \emph{Proc. of the
  {IEEE/RSJ} Intl. Conf. on Intell. Robots and Syst.}\hskip 1em plus 0.5em
  minus 0.4em\relax IEEE, 2015, pp. 4818--4825.

\bibitem{liu2017search}
S.~Liu, K.~Mohta, N.~Atanasov, and V.~Kumar, ``Search-based motion planning for
  aggressive flight in {SE} (3),'' \emph{arXiv preprint arXiv:1710.02748},
  2017.

\bibitem{likhachev2009dstar}
M.~Likhachev and D.~Ferguson, ``Planning long dynamically feasible maneuvers
  for autonomous vehicles,'' \emph{Intl. J. Robot. Research}, vol.~28, 2009.

\bibitem{aine2016multi}
S.~Aine, S.~Swaminathan, V.~Narayanan, V.~Hwang, and M.~Likhachev,
  ``Multi-heuristic {A}*,'' \emph{Intl. J. Robot. Research}, pp. 224--243,
  2016.

\bibitem{karaman2010incremental}
S.~Karaman and E.~Frazzoli, ``Incremental sampling-based algorithms for optimal
  motion planning,'' \emph{Proc. of Robot.: Sci. and Syst.}, vol. 104, p.~2,
  2010.

\bibitem{allen2016real}
R.~E. Allen and M.~Pavone, ``A real-time framework for kinodynamic planning
  with application to quadrotor obstacle avoidance,'' Ph.D. dissertation,
  Stanford University, 2016.

\bibitem{pivtoraiko2013sampleprimitive}
M.~Pivtoraiko, D.~Mellinger, and V.~Kumar, ``Incremental micro-{UAV} motion
  replanning for exploring unknown environments,'' in \emph{Proc. of the {IEEE}
  Intl. Conf. on Robot. and Autom.}\hskip 1em plus 0.5em minus 0.4em\relax
  IEEE, 2013, pp. 2452--2458.

\bibitem{oleynikova2016ct}
H.~Oleynikova, M.~Burri, Z.~Taylor, J.~Nieto, R.~Siegwart, and E.~Galceran,
  ``Continuous-time trajectory optimization for online {UAV} replanning,'' in
  \emph{Proc. of the {IEEE/RSJ} Intl. Conf. on Intell. Robots and Syst.}, 2016,
  pp. 5332--5339.

\bibitem{deits2015iris}
R.~Deits and R.~Tedrake, ``Efficient mixed-integer planning for {UAV}s in
  cluttered environments,'' in \emph{Proc. of the {IEEE} Intl. Conf. on Robot.
  and Autom.}, 2015, pp. 42--49.

\bibitem{kannan2013close}
S.~K. Kannan, W.~M. Sisson, D.~A. Ginsberg, J.~C. Derenick, X.~C. Ding, T.~A.
  Frewen, and H.~Sane, ``Close proximity obstacle avoidance using
  sampling-based planners,'' in \emph{AHS Specialists' Meeting on Unmanned
  Rotorcraft and Network-Centric Operations}, 2013.

\bibitem{chen2017improving}
J.~Chen and S.~Shen, ``Improving octree-based occupancy maps using environment
  sparsity with application to aerial robot navigation,'' in \emph{Proc. of the
  {IEEE} Intl. Conf. on Robot. and Autom.}\hskip 1em plus 0.5em minus
  0.4em\relax IEEE, 2017, pp. 3656--3663.

\bibitem{qin2000bsplineMatrix}
K.~Qin, ``General matrix representations for b-splines,'' \emph{The Visual
  Computer}, vol.~16, no.~3, pp. 177--186, 2000.

\bibitem{yang2010analytical}
K.~Yang and S.~Sukkarieh, ``An analytical continuous-curvature path-smoothing
  algorithm,'' \emph{IEEE Transactions on Robotics}, vol.~26, no.~3, pp.
  561--568, 2010.

\bibitem{yang2015generation}
L.~Yang, D.~Song, J.~Xiao, J.~Han, L.~Yang, and Y.~Cao, ``Generation of
  dynamically feasible and collision free trajectory by applying six-order
  {B}ezier curve and local optimal reshaping,'' in \emph{Proc. of the
  {IEEE/RSJ} Intl. Conf. on Intell. Robots and Syst.}\hskip 1em plus 0.5em
  minus 0.4em\relax IEEE, 2015, pp. 643--648.

\bibitem{ding2019safe}
W.~Ding, L.~Zhang, J.~Chen, and S.~Shen, ``Safe trajectory generation for
  complex urban environments using spatio-temporal semantic corridor,''
  \emph{{IEEE} Robot. and Auto. Letters}, 2019.

\bibitem{feigao2017hg}
F.~Gao, Y.~Lin, and S.~Shen, ``Gradient-based online quadrotor safe trajectory
  planning in 3d complex environments,'' in \emph{Proc. of the {IEEE/RSJ} Intl.
  Conf. on Intell. Robots and Syst.}, 2017.

\bibitem{usenko2017bspgradient}
V.~Usenko, L.~von Stumberg, A.~Pangercic, and D.~Cremers, ``Real-time
  trajectory replanning for mavs using uniform b-splines and 3d circular
  buffer,'' \emph{arXiv preprint arXiv:1703.01416}, 2017.

\bibitem{verriest1991linear}
E.~Verriest and F.~Lewis, ``On the linear quadratic minimum-time problem,''
  \emph{IEEE Transactions on Automatic Control}, vol.~36, no.~7, pp. 859--863,
  1991.

\bibitem{dijkstra1959note}
E.~W. Dijkstra, ``A note on two problems in connexion with graphs,''
  \emph{Numerische mathematik}, vol.~1, no.~1, pp. 269--271, 1959.

\bibitem{hart1968astar}
P.~E. Hart, N.~J. Nilsson, and B.~Raphael, ``A formal basis for the heuristic
  determination of minimum cost paths,'' \emph{IEEE transactions on Systems
  Science and Cybernetics}, vol.~4, no.~2, pp. 100--107, 1968.

\bibitem{russell2016artificial}
S.~J. Russell and P.~Norvig, \emph{Artificial Intelligence: A Modern
  Approach}.\hskip 1em plus 0.5em minus 0.4em\relax Malaysia; Pearson Education
  Limited,, 2016.

\bibitem{kleinbort2016collision}
M.~Kleinbort, O.~Salzman, and D.~Halperin, ``Collision detection or
  nearest-neighbor search? on the computational bottleneck in sampling-based
  motion planning,'' \emph{arXiv preprint arXiv:1607.04800}, 2016.

\bibitem{bertsekas1995dynamic}
D.~P. Bertsekas, D.~P. Bertsekas, D.~P. Bertsekas, and D.~P. Bertsekas,
  \emph{Dynamic Programming and Optimal Control}.\hskip 1em plus 0.5em minus
  0.4em\relax Athena Scientific Belmont, MA, 1995, vol.~1, no.~2.

\bibitem{cormen2009introduction}
T.~H. Cormen, \emph{Introduction to Algorithms}.\hskip 1em plus 0.5em minus
  0.4em\relax MIT press, 2009.

\bibitem{quinlan1993elastic}
S.~Quinlan and O.~Khatib, ``Elastic bands: Connecting path planning and
  control,'' in \emph{Proc. of the {IEEE} Intl. Conf. on Robot. and
  Autom.}\hskip 1em plus 0.5em minus 0.4em\relax IEEE, 1993, pp. 802--807.

\bibitem{zhu2015ces}
Z.~Zhu, E.~Schmerling, and M.~Pavone, ``A convex optimization approach to
  smooth trajectories for motion planning with car-like robots,'' in
  \emph{Proc. of the {IEEE} Control and Decision Conf.}, 2015, pp. 835--842.

\bibitem{qin2017vins}
T.~Qin, P.~Li, and S.~Shen, ``{VINS}-mono: A robust and versatile monocular
  visual-inertial state estimator,'' \emph{arXiv preprint arXiv:1708.03852},
  2017.

\bibitem{wang18quadtree}
K.~Wang, W.~Ding, and S.~Shen, ``Quadtree-accelerated real-time monocular dense
  mapping,'' in \emph{Proc. of the {IEEE/RSJ} Intl. Conf. on Intell. Robots and
  Syst.}\hskip 1em plus 0.5em minus 0.4em\relax IEEE, 2018.

\bibitem{gao2017dual}
W.~Gao and S.~Shen, ``Dual-fisheye omnidirectional stereo,'' in \emph{Proc. of
  the {IEEE/RSJ} Intl. Conf. on Intell. Robots and Syst.}\hskip 1em plus 0.5em
  minus 0.4em\relax IEEE, 2017, pp. 6715--6722.

\bibitem{gao18timeall}
F.~Gao, W.~Wu, J.~Pan, B.~Zhou, and S.~Shen, ``Optimal time allocation for
  quadrotor trajectory generation,'' in \emph{Proc. of the {IEEE/RSJ} Intl.
  Conf. on Intell. Robots and Syst.}\hskip 1em plus 0.5em minus 0.4em\relax
  IEEE, 2018.

\bibitem{Johnson2011nlopt}
\BIBentryALTinterwordspacing
S.~G. Johnson, \emph{The NLopt nonlinear-optimization package}, 2011. [Online].
  Available: \url{http://ab-initio.mit.edu/nlopt}
\BIBentrySTDinterwordspacing

\end{thebibliography}
\vspace{-1.0cm}

\begin{IEEEbiography}[{\includegraphics[width=1in,height=1.25in,clip,keepaspectratio]{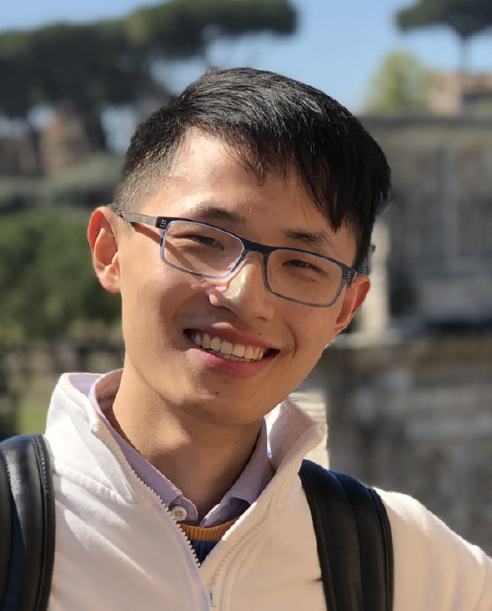}}]{Wenchao Ding}
	received his B.Eng. degree in Electronic and Information Engineering from Huazhong University of Science and Technology, China, in 2015. He is currently pursuing his PhD degree in the Hong Kong University of Science and Technology under the supervision of Prof. Shaojie Shen.

	His research interests include decision making, prediction, motion planning and autonomous navigation for aerial robots and autonomous vehicles.
\end{IEEEbiography}

\vskip -1.8\baselineskip plus -1fil

\begin{IEEEbiography}[{\includegraphics[width=1in,height=1.25in,clip,keepaspectratio]{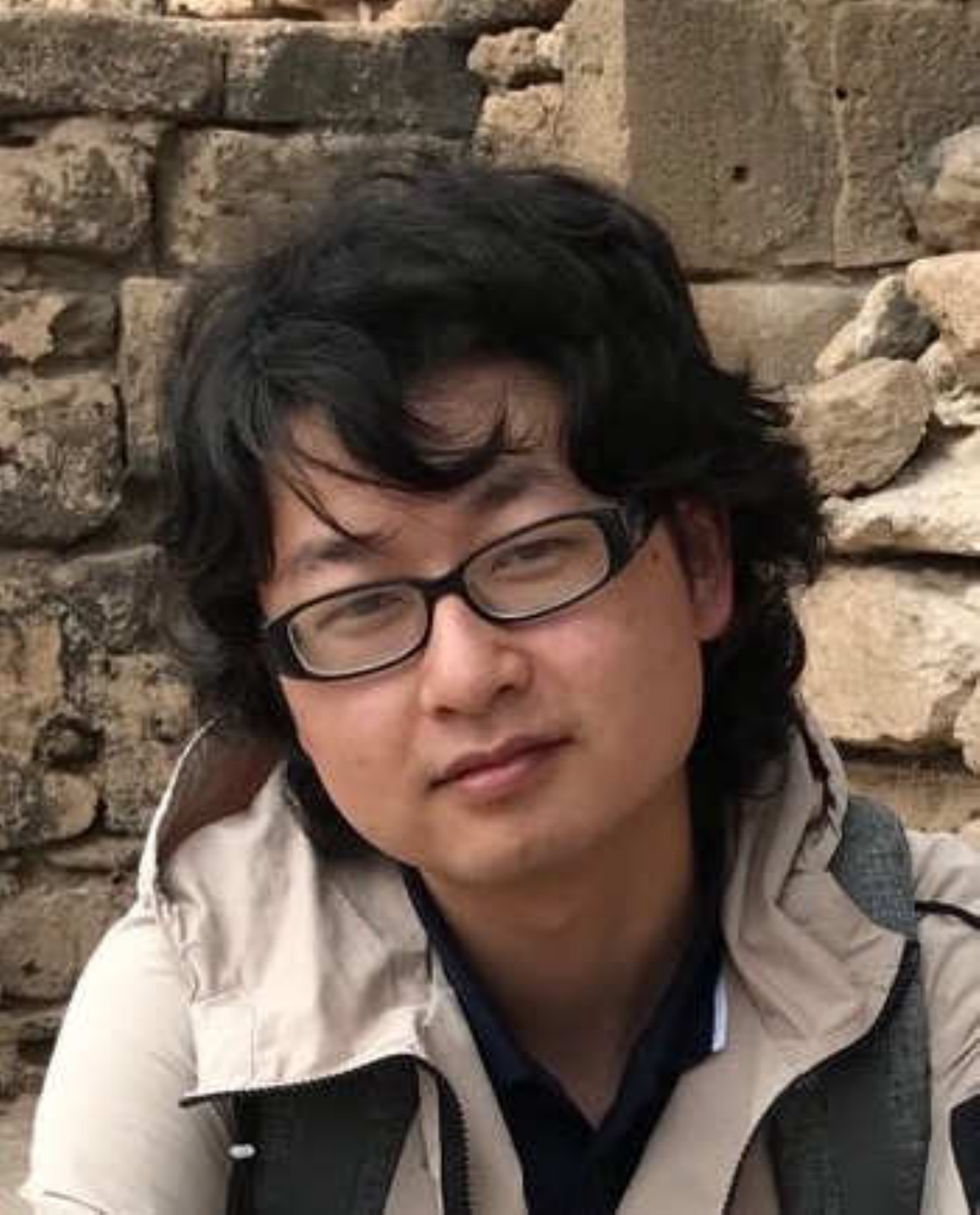}}]{Wenliang Gao}
	received his B.Eng. degree in optical engineering from Beijing Institute of Technology, Beijing, China, in 2016. He received his M.Phil. degree in robotics from Hong Kong University of Science and Technology, Hong Kong, in 2018. He is currently working as an algorithm engineer in DJI.

	His research interests include state estimation, navigation, and mapping with the visual-inertial system of autonomous robots.
\end{IEEEbiography}

\vskip -1.8\baselineskip plus -1fil

\begin{IEEEbiography}[{\includegraphics[width=1in,height=1.25in,clip,keepaspectratio]{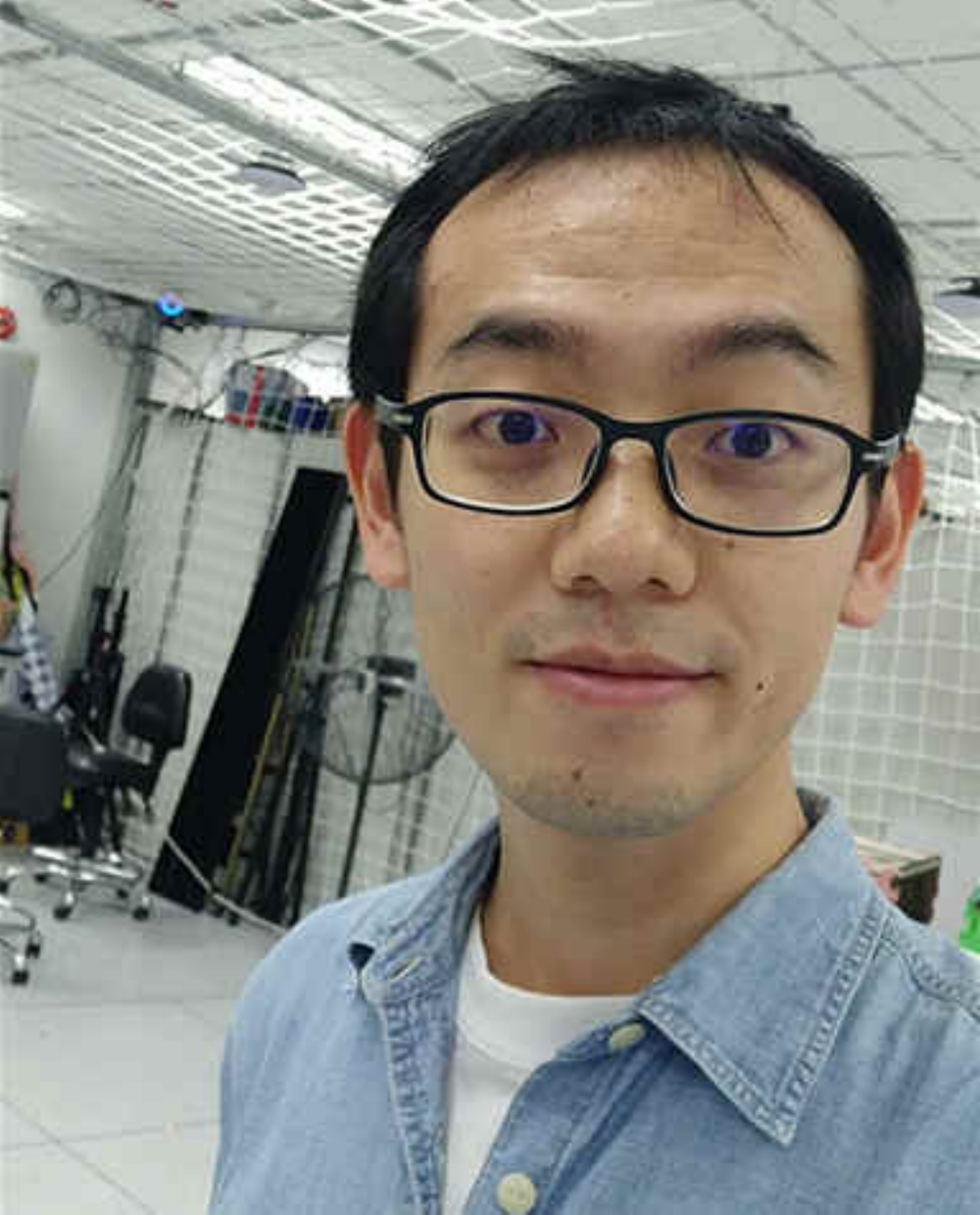}}]{Kaixuan Wang}
	received his B.Eng. degree in automation from Southeast University, Nanjing, China, in 2016. He then joined HKUST Aerial Robotics Group at the Hong Kong University of Science and Technology, under the supervision of Prof. Shaojie Shen.

	His research interests cover real-time 3D perception and mapping for robotics navigation, especially for UAV autonomous flight and autonomous vehicles.
\end{IEEEbiography}

\vskip -1.8\baselineskip plus -1fil

\begin{IEEEbiography}[{\includegraphics[width=1in,height=1.25in,clip,keepaspectratio]{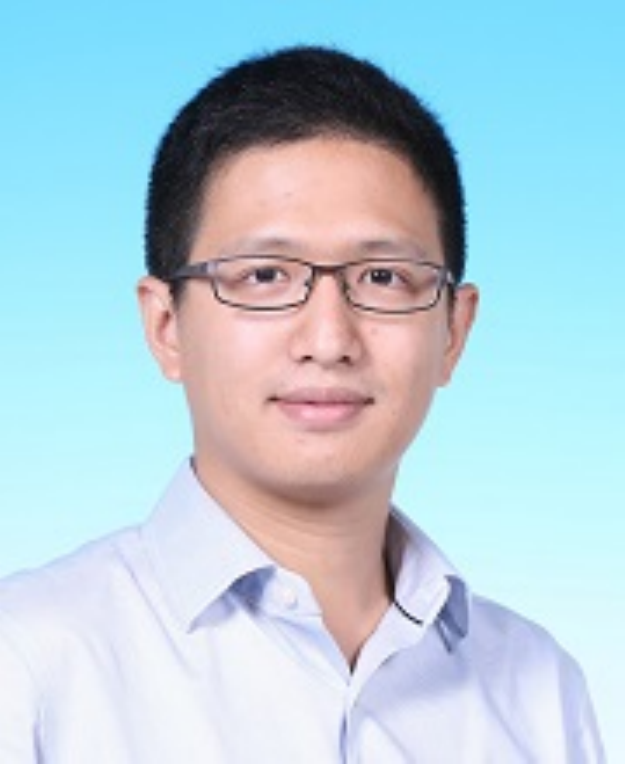}}]{Shaojie Shen}
	received his B.Eng. degree in Electronic Engineering from the Hong Kong University of Science and Technology (HKUST) in 2009. He received his M.S. in Robotics and Ph.D. in Electrical and Systems Engineering in 2011 and 2014, respectively, all from the University of Pennsylvania. He joined the Department of Electronic and Computer Engineering at the HKUST in September 2014 as an Assistant Professor.

	His research interests are in the areas of robotics and unmanned aerial vehicles, with focus on state estimation, sensor fusion, localization and mapping, and autonomous navigation in complex environments. He is currently serving as associate editors for T-RO and AURO.
\end{IEEEbiography}
\end{document}